\documentclass{article}

\usepackage{arxiv}

\usepackage[utf8]{inputenc} 
\usepackage[T1]{fontenc}    
\usepackage{url}            
\usepackage{booktabs}       
\usepackage{stmaryrd}
\usepackage{nicefrac}       
\usepackage{microtype}      
\usepackage{lipsum}		
\usepackage{graphicx}
\PassOptionsToPackage{numbers,sort&compress}{natbib}
\usepackage{natbib}
\usepackage{doi}
\usepackage{tocloft}

\usepackage{graphicx}

\usepackage{amsfonts,amssymb,amsthm,amsmath,amscd}
\usepackage{enumerate}
\usepackage{tikz}
\usepackage{subcaption} 
\usepackage[utf8]{inputenc}

\usepackage{overpic} 
\usepackage{url} 
\usepackage{comment} 

\usepackage[dvipsnames]{xcolor}

\usepackage{pifont}      
\usepackage{makecell}    
\usepackage{booktabs}   

\usetikzlibrary{backgrounds,fit}

\usepackage{circuitikz}

\usepackage{float}

\usepackage{array}

\usetikzlibrary{positioning}
\usepackage{tikz-cd}
\usepackage{mathrsfs}
\usetikzlibrary{decorations.pathmorphing}

\numberwithin{equation}{section}

\newtheorem{theorem}{Theorem}
\newtheorem{lemma}{Lemma}

\newtheorem{definition}{Definition}
\theoremstyle{definition}
\newtheorem{remark}{Remark}
\newtheorem{example}{Example}

\usetikzlibrary{tikzmark}

\numberwithin{theorem}{section}
\numberwithin{lemma}{section}
\numberwithin{corollary}{section}
\numberwithin{proposition}{section}
\numberwithin{definition}{section}
\numberwithin{remark}{section}

\usetikzlibrary{arrows.meta}
\newcommand{\cmark}{\textcolor{green!60!black}{\ding{51}}}  
\newcommand{\xmark}{\textcolor{red}{\ding{55}}}

\definecolor{tancolor}{RGB}{204,102,0}
\definecolor{edgegreen}{RGB}{0,153,0}

\colorlet{edgegreen}{edgegreen}
\colorlet{purplepink}{purple!50!pink}

\tikzset{
  Curved/.style={
    rounded corners,
    to path={
      -- ([xshift=2ex]\tikztostart.east)
      |- (#1) [near end]\tikztonodes
      -| ([xshift=-2ex]\tikztotarget.west)
      -- (\tikztotarget)
    }
  }
}

\usepackage{xcolor}

\usepackage[capitalise,noabbrev]{cleveref}
\usepackage{mathtools}

\usepackage{xcolor}

\definecolor{darkgreen}{HTML}{006400}   
\definecolor{darkorange}{HTML}{FF8C00}  
\definecolor{deeppink}{HTML}{FF69B4}    

\title{On the Sheafification of Higher-Order Message Passing}

\author{  Jacob Hume\\
	University of Cambridge\\
	\texttt{jmh274@cam.ac.uk} \\
	\And
	Pietro Liò \\
	University of Cambridge\\\texttt{pl219@cam.ac.uk} \\
}

\date{}

\renewcommand{\headeright}{Part III: Essay}
\renewcommand{\undertitle}{Part III of the Mathematical Tripos: Essay}
\renewcommand{\shorttitle}{On the Sheafification of Higher-Order Message Passing}

\hypersetup{
pdftitle={A template for the arxiv style},
pdfsubject={q-bio.NC, q-bio.QM},
pdfauthor={David S.~Hippocampus, Elias D.~Striatum},
pdfkeywords={First keyword, Second keyword, More},
}

\begin{document}
\maketitle

\vspace{-6mm}

\begin{abstract}
Recent work in Topological Deep Learning (TDL) seeks to generalize graph learning's preeminent \textit{message passing} paradigm to more complex relational structures: simplicial complexes, cell complexes, hypergraphs, and combinations thereof. Many approaches to such \textit{higher-order message passing} (HOMP) admit formulation in terms of nonlinear diffusion with the Hodge (combinatorial) Laplacian, a graded operator which carries an inductive bias that dimension-$k$ data features correlate with dimension-$k$ topological features encoded in the (singular) cohomology of the underlying domain. For $k=0$ this recovers the graph Laplacian and its well-studied homophily bias. In higher gradings, however, the Hodge Laplacian's bias is more opaque and potentially even degenerate. In this essay, we position sheaf theory as a natural and principled formalism for modifying the Hodge Laplacian's diffusion-mediated interface between local and global descriptors toward more expressive message passing. The sheaf Laplacian's inductive bias correlates dimension-$k$ data features with dimension-$k$ \textit{sheaf} cohomology, a data-aware generalization of singular cohomology. We will contextualize and novelly extend prior theory on sheaf diffusion in graph learning ($k=0$) in such a light — and explore how it fails to generalize to $k>0$ — before developing novel theory and practice for the higher-order setting. Our exposition is accompanied by a self-contained introduction shepherding sheaves from the abstract to the applied.

\end{abstract}

\vspace{4mm}

\small{\tableofcontents}

\section{Introduction}
\label{sec:introduction}





A sheafification of data science is underway. As foundational objects in modern algebraic geometry and topology, sheaves model the notion of a locally consistent attachment of data to a space. Although sheaves on arbitrary spaces are famously slippery, the spaces that support real-world relational data — (hyper)graphs, cell complexes, and other posets — are not arbitrary. Sheaves on a given poset admit a tremendously tractable characterization, for, as we shall see, they turn out to be precisely the \textit{diagrams} on that poset (Figure~\ref{fig:diagram-example}). Diagrams on posets latently pervade signal processing and artificial intelligence, where they often serve to encode relational structure indexed by the underlying domain~\cite{hansen2020laplacians}. The lofty term `sheaf' is used to maintain suggestivity:  behind the innocuous, routine notion of a diagram on data lies a storied geometric enterprise. Can we wield it? 

\begin{figure}[htbp]
    \centering
    \includegraphics[width=0.5\linewidth]{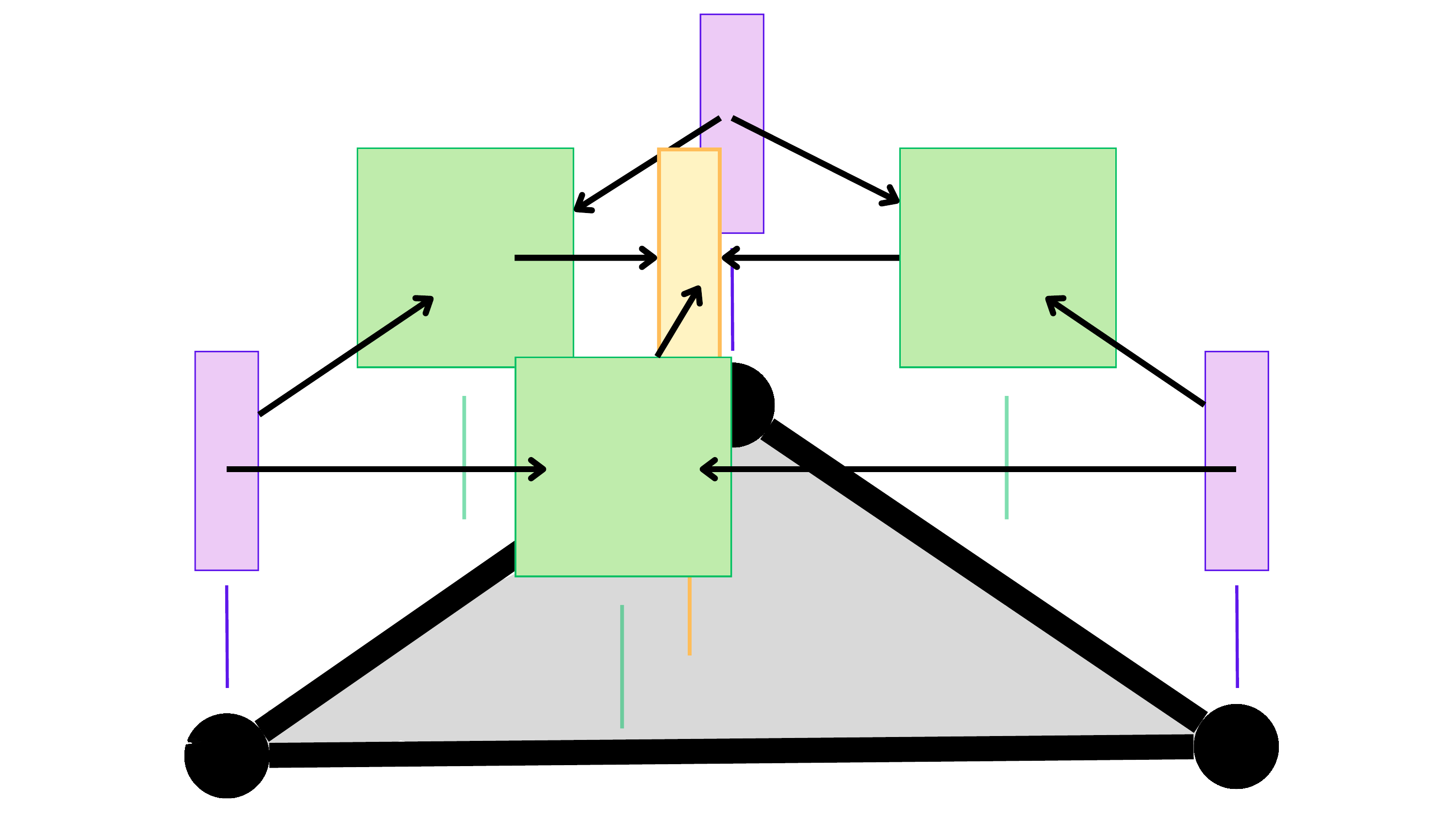} 
    \caption{A diagram (sheaf) on the face poset of a simplicial complex.}
    \label{fig:diagram-example}
\end{figure}

A flourish of recent inquiry gestures to the affirmative. Applied sheaf theory has lent novel insights to opinion dynamics~\cite{hansen2021opinion}, temporal data processing~\cite{bumpus_towards_2025}, causal knowledge representation~\cite{dacunto_relativity_2025}, network science~\cite{wolf2023topological} persistent homology and topological data analysis~\cite{wei2021persistent, curry2014sheaves, arya2025sheaf, russold2022persistent}, dimensionality reduction~~\cite{hansen2020laplacians, singer2012vector, peach2023implicit, gao2021diffusion}, consensus and distributed optimization problems~\cite{hansen2019distributed, hanks2025distributed}, knowledge graphs and symbolic reasoning~\cite{gebhart2023knowledge}, relational database theory~\cite{abramsky2015contextuality}, network coding~\cite{panteleev2024maximally, ghrist2011applications},  structural engineering~\cite{cooperband2023cosheaf,cooperband2024cellular, cooperband2024equivariant}, and much more. Deep learning is no exception. A plethora of recent work in graph representation learning employs \textit{sheaf-based message passing}~\cite{hansen2020sheaf, bodnar2022neural, duta2023sheaf, bamberger2024bundle, barbero2022sheafconnectionlaplacian, braithwaite2024heterogeneous, gillespie2024bayesian, battiloro2023tangent}, with successful standalone applications ranging from natural language processing~\cite{atri2023promoting} and recommendation systems~\cite{purificato2023sheaf4rec} to federated learning~\cite{nguyen2024sheaf, liang2024fedsheafhn}. Central to sheaf-based GNNs is the use of the \textit{sheaf graph Laplacian} as the shift operator with respect to which message passing is performed. This sheaf Laplacian generalizes the graph Laplacian, substituting its low-pass message passing dynamics for a more controllable type of diffusion process. Sheaf neural networks excel, therefore, at modifying the homophily bias implicit in prior message passing approaches toward mitigating oversmoothing and accommodating heterophily~\cite{bodnar_neural_2023} and heterogeneity~\cite{braithwaite2024heterogeneous}. 








Despite the generality of the applied sheaf theory formalism, extant work within and without deep learning focuses near-exclusively on order $k=0$, namely, on (hyper)graphs. This comes in spite of the fact that learning on higher-order data has burgeoned into a frontier of network science and relational deep learning~\cite{papamarkou2024position}, becoming thus a core tenet of \textit{topological deep learning} (TDL,~\cite{su2025topological, zia2024topological, hajij2022topological, papamarkou2024position}). Most TDL architectures operate with respect to the \textit{higher-order message passing} paradigm, commonly instantiated via augmented diffusion with the \textit{Hodge (combinatorial) Laplacian}, a graded extension of the graph Laplacian to orders $k\geq 0$. Diffusion with the $k$th Hodge Laplacian converts local descriptors into global ones; namely, the diffusion limit projects initial features onto the $k$th (singular) cohomology of the underlying domain. This encodes into such architectures the inductive bias that dimension-$k$ data features roughly correlate with dimension-$k$ topological ones. For graphs, $k=0$ and dimension-$k$ topological features coincide with components, whence this \textit{Hodge bias} amounts to the rough assumption that intra-cluster nodes are likely to have similar features, i.e., that homophily is present. When $k>0$, however, the Hodge bias appears to model a notion that is not as widespread for higher-order data as homophily is for graph data, and (as we shall see) can in fact be quite harmful. As natural tools for tweaking biases to align with data, sheaves thus have the potential to impact expressive TDL even more than they have graph learning. 

This essay aims to initiate a theory and practice for higher-order message passing with sheaves. Many new challenges emerge when $k>0$, alongside many compelling applications. Our contributions may be summarized as follows:

\begin{itemize}
    \item We offer a thorough exposition of the role sheaves play in graph representation learning. We focus in particular depth on the expressivity of linear and neural sheaf diffusion on (hyper)graphs and the relation to phenomena such as oversmoothing. This includes an explication of extant arguments in addition to the development of novel ones. While the excellent work~\cite{ayzenberg2025sheaf} reviews sheaves in deep learning at a macroscopic level, this is to our knowledge the first review to discuss graph sheaf diffusion in deep learning at this level of detail. Key to both our conceptual motivation and technical arguments is the identification of sheaf cohomology with \textit{global sections} in grading zero. Importantly, a similarly nice identification is absent in higher gradings, representing another nontrivial step in the generalization of sheaf-based message passing to higher orders. 

    \item We provide a streamlined account shepherding sheaves from their abstract origins to their applications in deep learning. In contrast to a majority of applied work, we focus on sheaves supported on general posets rather than immediately specializing to cell complexes or hypergraphs. Proceeding at this level of abstraction has a few benefits. For one, the two main settings where applied sheaf theory and deep learning meet, cell complexes and hypergraphs, may be treated with the same formalism. Dually, the formalism applies without modification to other domains of interest with which sheaves have yet to interact: posets arise naturally in realms ranging from complex systems to formal concept analysis and knowledge representation~\cite{sowa1984conceptual, ganter2005formal}. Indeed, there is a sense in which posets are the `natural habitat' supporting combinatorial sheaf theory (this will be the content of Theorem~\ref{thm:sheaves-diagrams-category-equivalence}); consequently, working over them can serve to clarify — not complicate — phenomena taking place at lower levels of abstraction.

    \item  We analyze the nature of higher-order message passing based on diffusion with the Hodge Laplacian, highlighting its potential and pitfalls. Along the way, we address some misconceptions in recent literature regarding the nature of `higher-order oversmoothing'. Using sheaf cohomology and combinatorial Hodge theory, we prove that sheaf diffusion has no such expressivity limitations and is capable of asymptotically solving any rank-$k$ classification task on a poset. Our proof proceeds in distinct cases $k=0$ and $k>0$, where the $k=0$ proof arises as a (nontrivial) generalization of previous work using a subclass of sheaves called \textit{discrete vector bundles}. We show that \textit{discrete vector bundles are not powerful enough when $k>0$}, contrasting existing intuition in the literature and requiring new machinery which we develop and which may be of independent interest. 
    
    \item We introduce a suite of `sheafified' architectures for TDL, discussing both the handcrafting and the learning of higher-order sheaves and providing empirical validation.

\end{itemize}

The remainder of this document is organized as follows. Section~\ref{sec:rapid-preliminaries} rapidly reviews the message passing paradigm in graph representation learning before providing motivation for how sheaves can help mitigate its traditional shortcomings. Section~\ref{sec:sheaves} introduces sheaves on topological spaces
and on posets, and derives the canonical setting where they become equivalent. Section~\ref{sec:motvation-higher-order} discuss combinatorial Hodge theory and heat diffusion in the abstract before specializing to the setting of simplicial cochains, leading to a discussion of the \textit{Hodge bias} underlying many higher-order message passing algorithms. Higher-order sheaves are then motivated as an approach to mitigating Hodge diffusion's shortcomings. Section~\ref{sec:sheaf-cohomology-diffusion} substantiates the theory of higher-order sheaves on posets, establishing a dictionary between sheaf cohomology, sheaf Laplacians, heat diffusion, and (in grading zero) global sections. In Section~\ref{sec:linear-separation-power-sheaf-diffusion} this dictionary is employed toward proving various results concerning the expressive power of \textit{linear} sheaf diffusion. Finally, Section~\ref{sec:learning-with-sheaves} proposes and validates a collection of higher-order \textit{neural} sheaf diffusion architectures.


\textbf{Notation and Convention.} All (co)homology is done with $\mathbb{R}$-coefficients. The inner product on $\mathbb{R}^n$ is the standard one unless otherwise stated. By `simplicial complex' we mean `\textit{abstract} simplicial complex'. By `cell complex' we mean `\textit{regular} cell complex'. (These additional assumptions are what allow us to treat the topological domains in question merely as posets without worrying e.g. about Euclidean embeddings or attaching maps.) Throughout this document, we assume acquaintance with algebraic topology and differential geometry, including a basic familiarity with category theory (the respective Part III courses are more than sufficient). We do not assume prior knowledge of sheaf theory, though it is engaging to compare our discussions with those of Part III Algebraic Geometry. Only a few proofs will really invoke the full machinery of these Part III courses, and these proofs were written by the first author — meaning it is always possible that more elementary arguments were missed.

\begin{table}[h]
\label{table:notation-table}
  \centering
  \begin{tabular}{@{}ll@{}}
    \toprule
    \textbf{Notation} & \textbf{Description} \\ \midrule
    $\mathcal{F}$, $\mathcal{F}_{x}$         & Sheaf on a topological space $X$, stalk at $x\in X$ \\
    $F$                                      & Diagram/sheaf on a poset $S$ \\
    $\Delta^{k}=d^{*}d + d\,d^{*}$           & $k$th-order Laplacian\\
    $\Delta^{0},\,L$                         & $0$th-order Laplacian\\
    ${H}^{k}(C^{\bullet}, d^{\bullet})$& Cohomology of a cochain complex $(C^\bullet, d^\bullet)$\\
    $Q_{k}= \langle \Delta^{k} -,\, - \rangle$  & $k$th Dirichlet energy\\
    $L'$                                     & Normalized vanilla Laplacian \\
    $E_{F}$                                  & Sheaf Dirichlet energy \\
    $G=(V,E)$                                & (Possibly weighted) graph with $n = \lvert V\rvert$ nodes, by default connected \\
    $S$                                      & Finite (pre)ordered set / (pre)poset with relation $\le$ \\
    $X_{S}$                                  & Alexandrov topological space associated to $S$ (as sets, $S = X_{S}$) \\
    $\mathsf{D}$                             & “Data” category in which a sheaf is valued; $\mathsf{Set}$-like, Abelian, or $\mathbb{R}\mathsf{Vect}$ as needed\\
    \bottomrule
  \end{tabular}
  \caption{Notation.}
\end{table}

\section{Motivation $\mathrm{I}$: Sheaves and Graph Representation Learning}
\label{sec:rapid-preliminaries}

We begin by motivating the use of sheaf diffusion in message passing graph neural networks. Following this, we will motivate the use of sheaf diffusion for higher-order message passing, observing which principles discussed here do and do not generalize to the higher-order setting.

\subsection{Message Passing Graph Neural Networks}
\label{sec:MPNNs}

Contemporary graph representation learning largely operates within the \textbf{message passing paradigm}, an architectural framework wherein the representations of graph nodes are progressively updated based on information from their neighbors. 

\begin{definition}Let $G=(V,E)$ be a graph, and consider a node $u \in V$. Let $N_{u}$ be the (one-hop) neighborhood of $u$. Let $\boldsymbol x_{u}$ be the features of $u \in V$, and $\boldsymbol e_{uv}$ be the features of edge $(u,v) \in E$. A \textbf{message passing update}, or \textbf{message passing layer}, can be expressed node-wise as follows:
\begin{equation}
\boldsymbol  h_{u}=\phi\left( \boldsymbol  x_{u}, \biguplus_{v \in N_{u}} \psi(\boldsymbol  x_{u}, \boldsymbol  x_{v}, \boldsymbol  e_{uv}) \right),
\end{equation}
where $\phi$ and $\psi$ are differentiable functions and $\biguplus$ is a permutation-invariant \textbf{aggregation operator}\footnote{Commonly the notation $\bigoplus$ is used instead of $\biguplus$ to denote the aggregation operator. The symbol $\bigoplus$ already plays the role of various direct sum constructions in this document, however, and is hence avoided here.} accepting indefinitely many arguments (e.g., \texttt{mean} or \texttt{max}). $\phi$ is called the  \textbf{update function}. $\psi$ is called the \textbf{message function}. The output of a message passing layer is a representation $\boldsymbol h_{u}$ for each node $u \in V$.
\end{definition}

\noindent Graph neural networks are generally built as compositions of several message passing layers, potentially in tandem with other standard techniques in the deep learning toolkit~\cite{hamilton2020graph}. The following instantiation of the message passing framework is central to our exposition. 



\begin{example}[Kipf and Welling, 2017 \cite{kipf2016semi}]
\label{ex:kipf-welling-GCN}
    Let $\boldsymbol W$ be a $|V| \times |V|$ `weight matrix', and let $d_u$ denote the degree of a node $u \in V$. Define a message passing layer as follows: put \begin{itemize}
        \item (Aggregation operator) $\biguplus:=\sum$ (summation);
        \item  (Message function) $\psi(\boldsymbol x_{u}, \boldsymbol x_{v}, \boldsymbol e_{uv}):=\frac{1}{\sqrt{ (d_{u}+1) (d _{v}+1) }}\boldsymbol W^{\top}\boldsymbol x_{v}$;
        \item (Update function) $\phi(\boldsymbol x_{u}, \biguplus):=\sigma\left( \frac{1}{\sqrt{ (d_{u}+1) ^{2}}} \boldsymbol W^{\top} \boldsymbol x_{u} + \biguplus\right)$ for element-wise nonlinearity $\sigma$  .
    \end{itemize}
This yields, for a given node $u$, the update 
\begin{align}
\boldsymbol  h_{u}&=\sigma \left( \frac{1}{\sqrt{ (d_{u}+1)^{2} }} \boldsymbol  W^{\top} \boldsymbol  x_{u}+ \sum_{v \in N_{u}} \frac{1}{\sqrt{ (d_{u}+1)(d_{v}+1) }}\boldsymbol  W^{\top} \boldsymbol  x_{v} \right) \\
&= \sigma\left( \sum_{v \in \tilde{N}_{u}} \frac{1}{\sqrt{ (d_{u}+1)(d_{v}+1) }} \boldsymbol  W^{\top} \boldsymbol  x_{v} \right),
\end{align}
where $\tilde{N}_{u}=N_{u}\cup \{ u \}$. With $\tilde{\boldsymbol A}=\boldsymbol I+ \boldsymbol A$ denoting the adjacency matrix of $G$ augmented with self-loops, and $\tilde{\boldsymbol D}$ the degree matrix of $\tilde{\boldsymbol A}$, this may be written \begin{equation}
\label{eqn:original-gcn-update}
    \boldsymbol  h_{u}= \sigma \left(\sum_{v \in V} \tilde{\boldsymbol  D}_{u u}^{-1/2} \tilde{\boldsymbol  A}_{uv} \tilde{\boldsymbol  D}^{-1/2}_{vv} \boldsymbol  W^{\top} \boldsymbol  x_{v}\right),\end{equation}
which is the $u$th row of $\sigma( \tilde{\boldsymbol D}^{-1/2} \tilde{\boldsymbol A} \tilde{\boldsymbol D}^{-1/2} \boldsymbol X \boldsymbol W)$. This is precisely the update for the well-studied graph convolutional network (GCN) architecture of Kipf and Welling~\cite{kipf2016semi}.  

\end{example}



\subsection{Intuition: Sheaves and Oversmoothing in Graph Representation Learning}
\label{sec:intuition-oversmoothing-heterophily} 
This section aims to motivate the introduction of sheaves and sheaf diffusion for more expressive deep learning on graphs and their generalizations. 
Perhaps the most concrete way to exhibit sheaves' utility in this context is to demonstrate the role they play in mitigating the so-called \textit{oversmoothing phenomenon} in message passing graph neural networks. Thus, the goal of this section is twofold: to explore sheaves in relation to oversmoothing, and then to explore oversmoothing in terms of sheaves. (We will return to this topic much later in Section~\ref{sec:learning-with-sheaves}, proving explicit results on the relationship between sheaves and oversmoothing and finishing the story begun here.)



\subsubsection{Oversmoothing, Heterophily, and Heat Diffusion}
We lead off with the following fact, which, though elementary, is central to the relationship between oversmoothing in MPNNs and heat diffusion on graphs. Later (Sections~\ref{sec:motvation-higher-order}-\ref{sec:sheaf-cohomology-diffusion}), we will generalize it to the sheaf-theoretic setting, where it will again be central.

\begin{theorem}[GCNs and Heat Diffusion]
\label{thm:GCNs-and-heat-diffusion}
    Omitting nonlinearity, weights, normalization, and self-loops from the GCN update (\ref{eqn:original-gcn-update}) recovers an Euler discretization of the graph heat diffusion equation \begin{equation}
    \label{eqn:graph-heat-diffusion}
        \boldsymbol{\dot{X}}(t)=-\alpha\boldsymbol{L} \boldsymbol{X}(t),  \ \ \alpha>0,
    \end{equation}
    where $\boldsymbol{L}=\boldsymbol{D}-\boldsymbol{A}$ is the (unnormalized) graph Laplacian of $G$. Solutions $\boldsymbol X(t)$ converge exponentially as $t \to \infty$ to the orthogonal projection of $\boldsymbol X(0)$ onto the \textit{harmonic space} $\ker \boldsymbol L$.
\end{theorem}

\noindent In order to prove Theorem~\ref{thm:GCNs-and-heat-diffusion}, we recall the following standard fact from the theory of ordinary differential equations. 

\begin{theorem}
\label{thm:ODE-fact}
    Let $\boldsymbol{M} \in \mathbb{R}^{n \times n}$, giving rise to a linear homogeneous system of ODEs with constant coefficients \begin{equation}
    \label{eqn:LHSODECC}
\dot{\boldsymbol{X}}(t)=\boldsymbol{M} \boldsymbol{X}(t).        
    \end{equation}
    Solutions to (\ref{eqn:LHSODECC}) are given by \begin{equation}
    \boldsymbol{X}(t)=e^{\boldsymbol{M}t} \boldsymbol{X}(0) \end{equation} for $\boldsymbol{X}(0)$ a choice of initial condition.
\end{theorem}
\noindent Note that, when $\boldsymbol{M}$ is diagonalizable, in a corresponding eigenbasis the solution operator takes the form $\operatorname{diag}(e^{\lambda_1 t}, \dots, e^{\lambda_n t})$ if $(\lambda_i)_{i=1}^n$ is the spectrum of $\boldsymbol{M}$. If $\boldsymbol{M}$ is in fact negative-semidefinite, so that $\lambda_i(\boldsymbol{M}) \leq 0$ for all $i \in [n]$, then as $t \to \infty$ this solution operator converges to the orthogonal projection onto $\ker \boldsymbol{M}$. 

\begin{proof}
    Deriving from (\ref{eqn:original-gcn-update}) the Euler discretization of Equation \ref{eqn:graph-heat-diffusion} is a matrix computation. Showing convergence to $\ker \boldsymbol{L}$ is a straightforward exercise in differential equations. Indeed, let $\boldsymbol{L}=\boldsymbol{D}-\boldsymbol{A}$ denote the Laplacian of $G$, and $\boldsymbol L'=\boldsymbol{D}^{-1/2} \boldsymbol{L} \boldsymbol{D}^{-1/2}$ its normalization. In matrix form, the update in (\ref{eqn:original-gcn-update}) may be written (recall that self-loops have been omitted): \begin{align}
\boldsymbol  H&=\sigma( {\boldsymbol D}^{-1/2} {\boldsymbol A} {\boldsymbol D}^{-1/2} \boldsymbol X \boldsymbol W) \\
&= \sigma(\boldsymbol  D^{-1/2}(\boldsymbol  D - \boldsymbol  L)\boldsymbol  D^{-1/2} \boldsymbol  X \boldsymbol  W) \notag \\
&= \sigma\big((\boldsymbol  I - \boldsymbol  D^{-1/2} \boldsymbol  L \boldsymbol  D^{-1/2}) \boldsymbol  X \boldsymbol  W \big) \notag \\
& = \sigma\big( (\boldsymbol  I - \boldsymbol  L' ) \boldsymbol  X \boldsymbol  W\big) \label{eqn:gcn-update-Laplacian}
\end{align}
Omitting nonlinearity $\sigma$, weights $\boldsymbol{W}$, and normalization from (\ref{eqn:gcn-update-Laplacian}) gives the update \begin{equation}
\label{eqn:discrete-graph-linear-heat-diffusion}
    \boldsymbol{X_{\operatorname{new}}} = (\boldsymbol{I} - \boldsymbol{L})\boldsymbol{X} = \boldsymbol{X} - \boldsymbol{L} \boldsymbol{X}.
\end{equation}
This is precisely the Euler update for the differential equation (\ref{eqn:graph-heat-diffusion}), as claimed.
Now, Theorem~\ref{thm:ODE-fact} says that solutions $\boldsymbol{X}(t)$ to Equation~\ref{eqn:graph-heat-diffusion} are of the form $e^{-\alpha \boldsymbol{L} t} \boldsymbol{X}(0)$.  $\boldsymbol{L}$ is orthogonally diagonalizable (it is symmetric); in an orthonormal eigenbasis for $\boldsymbol{L}$ the solution operator is represented by the matrix $\operatorname{diag}(e^{-\alpha \lambda_1 t},\dots, e^{-\alpha \lambda_n t})$, where $(\lambda_i)_{i=1}^n$ is the spectrum of $\boldsymbol{L}$. Because $\boldsymbol{L}$ is positive semidefinite, $\lambda_i \geq 0$. As $t \to \infty$, the diagonal entries corresponding to $\lambda >0$ exponentially decay to $0$, while those corresponding to $\lambda=0$ — that is, those corresponding to the harmonic space $\ker \boldsymbol{L}$ — remain constant $1$. This specifies exactly the orthogonal projection onto $\ker \boldsymbol{L}$. 
\end{proof}

\remark{Proving the convergence claim in Theorem~\ref{thm:GCNs-and-heat-diffusion} relied only on positive-semidefiniteness of $\boldsymbol{L}$ (or, rather, negative-semidefiniteness of $-\boldsymbol{L}$). Since normalizing $\boldsymbol{L}$ preserves positive-semidefiniteness, the analogous result holds also for the normalization $\boldsymbol{L'}$ of $\boldsymbol{L}$. 
}\\

\noindent In the motivating examples that follow, we will for simplicity restrict ourselves to scalar features/signals $\boldsymbol{x} \in \mathbb{R}^{|V|}$ on the nodes of $G$. 
The following provisional definition quantifies the `smoothness' of a node signal $\boldsymbol{x}$ on $G$. It is greatly generalized in Section~\ref{sec:motvation-higher-order} (Definition~\ref{def:dirichlet-energy-general-def}).

\begin{definition}
\label{def:graph-dirichlet-energy}
The \textbf{Dirichlet energy form} on $G$ is the symmetric positive semidefinite quadratic form $E=\langle -, \boldsymbol{L} - \rangle.$
The \textbf{Dirichlet energy of a scalar node signal $\boldsymbol{x}$} on $G$ is $E(\boldsymbol{x})=\langle \boldsymbol{x}, \boldsymbol{L} \boldsymbol x \rangle$.
\end{definition}
\noindent As a sanity check, recall that the harmonic space $\ker \boldsymbol{L}$ consists of locally constant signals on $G$. Thus, the Dirichlet energy of a constant (i.e., smoothest possible) signal on a connected graph is zero. In general,  the Dirichlet energy specifies how far away a signal $\boldsymbol{x}$ is from being (locally) constant. Indeed, recalling that $\boldsymbol{L}=\boldsymbol{B}^\top \boldsymbol{B}$ for $\boldsymbol{B}$ the incidence matrix of $G$, and recalling the general linear-algebraic fact that $\ker \boldsymbol{B}^\top \boldsymbol{B}=\ker \boldsymbol{B}$, one has \begin{equation}
     E(\boldsymbol{x}) = \boldsymbol{x}^\top \boldsymbol{B}^\top \boldsymbol{B} \boldsymbol{x}=\| \boldsymbol{B} \boldsymbol{x} \|^2_2=\text{distance}(\boldsymbol{x}, \ker \boldsymbol{L}).
\end{equation}

As with much of the content in this section, the Dirichlet energy's generalization to the sheaf-theoretic setting and beyond will be instrumental in future sections.

\begin{example}
\label{ex:mlmi}
    Let $G$ be a graph and $\boldsymbol{x}=\boldsymbol{x}(0) \in \mathbb{R}^n$ a scalar signal on $G$. What are the consequences of Theorem~\ref{thm:GCNs-and-heat-diffusion} for the limiting behavior $\boldsymbol{x}_{\infty}$ of node representations if we diffuse $\boldsymbol{x}$ over $G$? \\
    If $G$ is connected, then we know $\ker \boldsymbol{L}$ consists of constant functions: $\ker \boldsymbol{L}=\langle \boldsymbol 1 \rangle$. Thus, $E(\boldsymbol{x}_\infty)=0$: in the limit, the diffusion process has fully `smoothened' $\boldsymbol x$ out. 

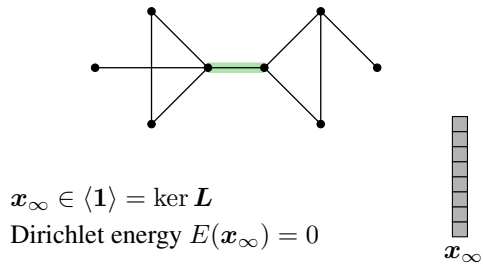
\begin{figure}[H]
    \centering
    \begin{tikzpicture}
 
    \begin{scope}[shift={(0,-1)}]
        \draw [line width=0.5pt](5,14.75) to[short] (5,14.75);
        \draw [line width=0.5pt](5,14.5) to[short] (5.75,13.75);
        \draw [line width=0.5pt](5.75,13.75) to[short] (5,13);
        \draw [line width=0.5pt](5,14.5) to[short] (5,13);
        \draw [line width=0.5pt](5.75,13.75) to[short] (4.25,13.75);
        \draw [line width=0.5pt](7.25,14.5) to[short] (6.5,13.75);
        \draw [line width=0.5pt](6.5,13.75) to[short] (7.25,13);
        \draw [line width=0.5pt](7.25,14.5) to[short] (7.25,13);
        \draw [line width=0.5pt](7.25,14.5) to[short] (8,13.75);
        \draw [line width=4pt, edgegreen, opacity=0.3] (5.75,13.75) to[short] (6.5,13.75);
        \draw [line width=0.5pt](5.75,13.75) to[short] (6.5,13.75);

        \node (n1) at (5.75,13.75) [circle, draw, fill=black, inner sep=1pt] {};
        \node (n2) at (4.25,13.75) [circle, draw, fill=black, inner sep=1pt] {};
        \node (n3) at (5,14.5) [circle, draw, fill=black, inner sep=1pt] {};
        \node (n4) at (5,13) [circle, draw, fill=black, inner sep=1pt] {};
        \node (n5) at (8,13.75) [circle, draw, fill=black, inner sep=1pt] {};
        \node (n6) at (6.5,13.75) [circle, draw, fill=black, inner sep=1pt] {};
        \node (n7) at (7.25,14.5) [circle, draw, fill=black, inner sep=1pt] {};
        \node (n8) at (7.25,13) [circle, draw, fill=black, inner sep=1pt] {};
    \end{scope}

    \begin{scope}[shift={(0,-3)}]
        \node[anchor=west] at (3,14) {\textbf{$\boldsymbol{x}_\infty \in \langle \boldsymbol{1} \rangle=\ker \boldsymbol{L}$}};
        \node[anchor=west] at (3,13.5) {Dirichlet energy $E(\boldsymbol{x}_\infty)=0$};
        
        \begin{scope}[shift={(9,13.5)}]
            \foreach \y in {0,1,...,7} {
                \fill[gray!60] (0,\y*0.2) rectangle +(0.2,0.2);
                \draw (0,\y*0.2) rectangle +(0.2,0.2);
            }
            \node[below] at (.15,0) {$\boldsymbol{x}_\infty$};
        \end{scope}
    \end{scope}
    \end{tikzpicture}
    \caption{Graph $G$ with a single connected component. Node representations become identical to one another in the heat diffusion limit.}
    \label{fig:1-oversmoothing-motivation}
\end{figure}

\noindent Now, we might next sever $G$ along the green bottleneck depicted in Figure~\ref{fig:1-oversmoothing-motivation}, so that it now consists of components $C_1$ and $C_2$. The indicator signals $\boldsymbol{1}_{C_1}, \boldsymbol{1}_{C_2}$ for these components form then a basis for $\ker \boldsymbol{L}$. Consequently, each node's limiting representation takes on one of two values, depending on the component of $G$ to which it belongs. 

\begin{figure}[H]
    \centering
    \begin{tikzpicture}

    \begin{scope}[shift={(0,-1)}]
        \draw [line width=0.5pt](5,14.75) to[short] (5,14.75);
        \draw [line width=0.5pt](5,14.5) to[short] (5.75,13.75);
        \draw [line width=0.5pt](5.75,13.75) to[short] (5,13);
        \draw [line width=0.5pt](5,14.5) to[short] (5,13);
        \draw [line width=0.5pt](5.75,13.75) to[short] (4.25,13.75);
        \draw [line width=0.5pt](7.25,14.5) to[short] (6.5,13.75);
        \draw [line width=0.5pt](6.5,13.75) to[short] (7.25,13);
        \draw [line width=0.5pt](7.25,14.5) to[short] (7.25,13);
        \draw [line width=0.5pt](7.25,14.5) to[short] (8,13.75);

        \node (n1) at (5.75,13.75) [circle, draw, fill=black, inner sep=1pt] {};
        \node (n2) at (4.25,13.75) [circle, draw, fill=black, inner sep=1pt] {};
        \node (n3) at (5,14.5) [circle, draw, fill=black, inner sep=1pt] {};
        \node (n4) at (5,13) [circle, draw, fill=black, inner sep=1pt] {};
        \node (n5) at (8,13.75) [circle, draw, fill=black, inner sep=1pt] {};
        \node (n6) at (6.5,13.75) [circle, draw, fill=black, inner sep=1pt] {};
        \node (n7) at (7.25,14.5) [circle, draw, fill=black, inner sep=1pt] {};
        \node (n8) at (7.25,13) [circle, draw, fill=black, inner sep=1pt] {};
        
        \begin{scope}[on background layer]
            \node[fill=gray!30,draw=gray!60,rounded corners,inner sep=2pt,fit=(n1) (n2) (n3) (n4)] {};
            \node[fill=gray!20,draw=gray!50,rounded corners,inner sep=2pt,fit=(n5) (n6) (n7) (n8)] {};
            \node at (4.3,13.45) {$C_1$};
            \node at (7.95,13.45) {$C_2$};
        \end{scope}
    \end{scope}

    \begin{scope}[shift={(0,-3)}]
        \node[anchor=west] at (3,14) {$\boldsymbol{x}_\infty \in \langle \boldsymbol{1}_{C_1}, \boldsymbol{1}_{C_2} \rangle = \operatorname{ker }\boldsymbol{L}$};
        \node[anchor=west] at (3,13.5) {Get \textit{smoothing `in the components'}};
        ;
        
        \begin{scope}[shift={(9,13.5)}]
            \foreach \y in {0,1,...,4} {
                \fill[gray!30] (0,\y*0.2) rectangle +(0.2,0.2);
                \draw (0,\y*0.2) rectangle +(0.2,0.2);
            }
            \foreach \y in {4, 5,...,7} {
                \fill[gray!60] (0,\y*0.2) rectangle +(0.2,0.2);
                \draw (0,\y*0.2) rectangle +(0.2,0.2);
            }
            \node[below] at (.15,0) {$\boldsymbol{x}_\infty$};
        \end{scope}
    \end{scope}
    \end{tikzpicture}
    \caption{Graph $G$ with components $C_1$ and $C_2$. Node representations converge under heat diffusion to local constancy.}
    \label{fig:2-oversmoothing-motivation-disconnected-components}
\end{figure}
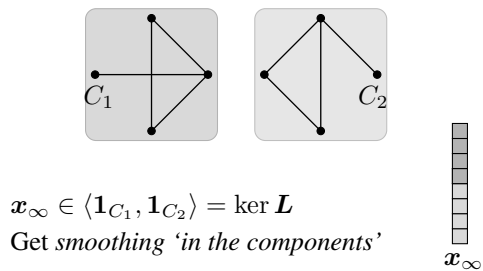

\noindent If $G$ is homophilic, this may be exactly what one wants! For instance, let us assume that the components $C_1$ and $C_2$ of $G$ coincide with two node classes $\textcolor{purple}{A_1}$ and $\textcolor{pink}{A_2}$ respectively. In this case, heat diffusion serves to separate the node representations, making classification a very easy task for all but the most degenerate initial conditions.

\begin{figure}[H]

    \centering
    \begin{tikzpicture}

    \begin{scope}[shift={(0,-1)}]
        \draw [line width=0.5pt](5,14.75) to[short] (5,14.75);
        \draw [line width=0.5pt](5,14.5) to[short] (5.75,13.75);
        \draw [line width=0.5pt](5.75,13.75) to[short] (5,13);
        \draw [line width=0.5pt](5,14.5) to[short] (5,13);
        \draw [line width=0.5pt](5.75,13.75) to[short] (4.25,13.75);
        \draw [line width=0.5pt](7.25,14.5) to[short] (6.5,13.75);
        \draw [line width=0.5pt](6.5,13.75) to[short] (7.25,13);
        \draw [line width=0.5pt](7.25,14.5) to[short] (7.25,13);
        \draw [line width=0.5pt](7.25,14.5) to[short] (8,13.75);
        
        \node (n1) at (5.75,13.75) [draw, circle, fill=purple, inner sep=1pt] {};
        \node (n2) at (4.25,13.75) [draw, circle, fill=purple, inner sep=1pt] {};
        \node (n3) at (5,14.5) [draw, circle, fill=purple, inner sep=1pt] {};
        \node (n4) at (5,13) [draw, circle, fill=purple, inner sep=1pt] {};
        \node (n5) at (8,13.75) [draw, circle, fill=pink, inner sep=1pt] {};
        \node (n6) at (6.5,13.75) [draw, circle, fill=pink, inner sep=1pt] {};
        \node (n7) at (7.25,14.5) [draw, circle, fill=pink, inner sep=1pt] {};
        \node (n8) at (7.25,13) [draw, circle, fill=pink, inner sep=1pt] {};
        
        \begin{scope}[on background layer]
            \node[fill=gray!30,draw=gray!60,rounded corners,inner sep=2pt,fit=(n1) (n2) (n3) (n4)] {};
            \node[fill=gray!20,draw=gray!50,rounded corners,inner sep=2pt,fit=(n5) (n6) (n7) (n8)] {};
            \node at (4.3,13.45) {$C_1$};
            \node at (7.95,13.45) {$C_2$};
        \end{scope}
    \end{scope}

    \begin{scope}[shift={(0,-3)}]
        \node[anchor=west] at (3,14) {$\boldsymbol{x}_\infty \in \langle \boldsymbol{1}_{C_1}, \boldsymbol{1}_{C_2} \rangle = \operatorname{ker }\boldsymbol{L}$};
        \node[anchor=west] at (3,13.5) {Get \textit{smoothing `in the components'.}};
        \node[anchor=west] at (3,13) {\textbf{If $G$ is homophilic, this may be exactly desirable!}};
        
        \begin{scope}[shift={(9,13.5)}]
            \foreach \y in {0,1,...,4} {
                \fill[pink] (0,\y*0.2) rectangle +(0.2,0.2);
                \draw (0,\y*0.2) rectangle +(0.2,0.2);
            }
            \foreach \y in {4, 5,...,7} {
                \fill[purple] (0,\y*0.2) rectangle +(0.2,0.2);
                \draw (0,\y*0.2) rectangle +(0.2,0.2);
            }
            \node[below] at (.15,0) {$\boldsymbol{x}_\infty$};
            \node[right] at (0.25,0.78) {\footnotesize{$\}$} \scriptsize{separable!}};
        \end{scope}
    \end{scope}
    \end{tikzpicture}
    \caption{The components $C_1$ and $C_2$ of $G$ align with node classes $\textcolor{purple}{A_1}$ and $\textcolor{pink}{A_2}$. This allows heat diffusion to achieve perfect separability in the diffusion limit.}
    \label{fig:3-oversmoothing-motivation-components-with-classes}
\end{figure}
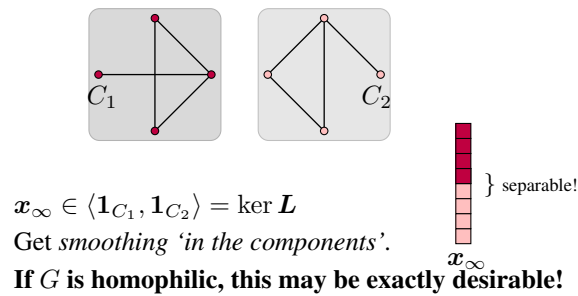

Of course, this situation is as fragile as it is fortuitous. Keeping the node classes $\textcolor{purple}{A_1}$ and $\textcolor{pink}{A_2}$ as is, and connecting $G$ once again — introducing thus a tiny amount of heterophily into the network — yields the same diffusion process as Figure~\ref{fig:1-oversmoothing-motivation}. In the diffusion limit, separating the classes based on node representations becomes impossible.

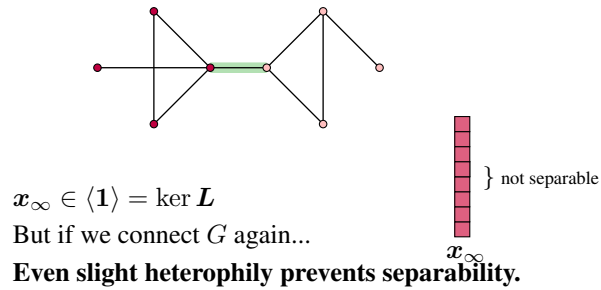
\begin{figure}[H]
    \centering
    \begin{tikzpicture}

    \begin{scope}[shift={(0,-1)}]
        \draw [line width=0.5pt](5,14.75) to[short] (5,14.75);
        \draw [line width=0.5pt](5,14.5) to[short] (5.75,13.75);
        \draw [line width=0.5pt](5.75,13.75) to[short] (5,13);
        \draw [line width=0.5pt](5,14.5) to[short] (5,13);
        \draw [line width=0.5pt](5.75,13.75) to[short] (4.25,13.75);
        \draw [line width=0.5pt](7.25,14.5) to[short] (6.5,13.75);
        \draw [line width=0.5pt](6.5,13.75) to[short] (7.25,13);
        \draw [line width=0.5pt](7.25,14.5) to[short] (7.25,13);
        \draw [line width=0.5pt](7.25,14.5) to[short] (8,13.75);
        \draw [line width=4pt, edgegreen, opacity=0.3] (5.75,13.75) to[short] (6.5,13.75);
        \draw [line width=0.5pt](5.75,13.75) to[short] (6.5,13.75);

        \node (n1) at (5.75,13.75) [draw, circle, fill=purple, inner sep=1pt] {};
        \node (n2) at (4.25,13.75) [draw, circle, fill=purple, inner sep=1pt] {};
        \node (n3) at (5,14.5) [draw, circle, fill=purple, inner sep=1pt] {};
        \node (n4) at (5,13) [draw, circle, fill=purple, inner sep=1pt] {};
        \node (n5) at (8,13.75) [draw, circle, fill=pink, inner sep=1pt] {};
        \node (n6) at (6.5,13.75) [draw, circle, fill=pink, inner sep=1pt] {};
        \node (n7) at (7.25,14.5) [draw, circle, fill=pink, inner sep=1pt] {};
        \node (n8) at (7.25,13) [draw, circle, fill=pink, inner sep=1pt] {};
    \end{scope}

    \begin{scope}[shift={(0,-3)}]
        \node[anchor=west] at (3,14) {\textbf{$\boldsymbol{x}_\infty \in \langle \boldsymbol{1} \rangle=\ker \boldsymbol{L}$}};
        \node[anchor=west] at (3,13.5) {But if we connect $G$ again...};
        \node[anchor=west] at (3,13) {\textbf{Even slight heterophily prevents separability.}};

        \begin{scope}[shift={(9,13.5)}]
            \foreach \y in {0,1,...,4} {
                \fill[purplepink] (0,\y*0.2) rectangle +(0.2,0.2);
                \draw (0,\y*0.2) rectangle +(0.2,0.2);
            }
            \foreach \y in {4, 5,...,7} {
                \fill[purplepink] (0,\y*0.2) rectangle +(0.2,0.2);
                \draw (0,\y*0.2) rectangle +(0.2,0.2);
            }
            \node[below] at (.15,0) {$\boldsymbol{x}_\infty$};
            \node[right] at (0.25,0.78) {\footnotesize{$\}$} \scriptsize{not separable}};
        \end{scope}
    \end{scope}
    \end{tikzpicture}
    \caption{When $G$ is reconnected, heterophily causes the node representations to mix, preventing class separability in the diffusion limit.}
    \label{fig:4-oversmoothing-motivation-connected-with-classes}
\end{figure}
\end{example}

\noindent The takeaway from Example~\ref{ex:mlmi}, in view of Theorem~\ref{thm:GCNs-and-heat-diffusion}, is that if we strip away weights, nonlinearity, and Laplacian normalization from a GCN, oversmoothing becomes exponentially inevitable. Importantly, this behavior fundamentally topological\footnote{Here and throughout this document, `topological' is to be construed in the coarse, point-set sense of the term, which does not always conceptually coincide with the `graph topology' sense of the term. One could conceptually say that modifying the latter really can include modifications both to `geometry' and `topology' (as is made somewhat precise e.g. in discrete curvature-based rewiring techniques for oversquashing~\cite{topping2021understanding}). Here, we mean `topological' in the specific sense that modifying the number of connected components of a graph (a topological invariant) will change the asymptotic oversmoothing behavior described previously, while merely perturbing it generally will not.}, and does not depend, for example, on the specific node features. One asks: \begin{quote}
\centering
\emph{Does oversmoothing remain inevitable even when nonlinearity, weights, and normalization are (re)introduced?}
\end{quote}
In time (Section~\ref{sec:learning-with-sheaves}), we shall see the answer is very-nearly \textit{yes}, and that a \textit{trivial sheaf} is to blame. So what is a sheaf?

\subsubsection{Network Sheaves: A First Look in Context}
The sheaf formalism will be properly developed in Section~\ref{sec:sheaves}. In this section, we motivate the introduction of sheaves in graph representation learning (\textit{network sheaves}) by illustrating how they naturally counteract the oversmoothing phenomenon informally introduced in the previous section.\\

So let us reconsider our graph from Example~\ref{ex:mlmi} (any connected graph will do). Zoom in, say, on the green bottleneck edge:

\begin{figure}[H]
\centering
\begin{tikzpicture}[baseline]

    \path[use as bounding box] (-5.5,-1.5) rectangle (5.5,3.5);

    \node [circle, draw=tancolor, fill=tancolor!20, minimum size=0.2cm] (u) at (-3,0) {};
    \node [circle, draw=tancolor, fill=tancolor!20, minimum size=0.2cm] (v) at (3,0) {};

    \draw [black, thick] (u) -- (v);
    \draw [edgegreen, opacity=0.3, line width=3pt] (u) -- (v);

    \draw [opacity=0.2] (-3,0) -- (-5,1);
    \draw [opacity=0.2] (-3,0) -- (-5,-1);

    \draw [opacity=0.2] (3,0) -- (5,1);
    \draw [opacity=0.2] (3,0) -- (5,-1);

    \node [tancolor] at (-3,-0.5) {$u$};
    \node [tancolor] at (3,-0.5) {$v$};
    \node [edgegreen] at (0,-0.3) {$e$};
\end{tikzpicture}
\label{fig:1-network-sheaves-motivation}
\end{figure}

\noindent We can visualize the scalar signal $\boldsymbol{x}$ over $G$ by attaching a copy of $\mathbb{R}$ to each node $u$ and depicting $x_u$ inside:

\begin{figure}[H]
\centering
\begin{tikzpicture}[baseline, every node/.style={font=\small}]
    \path[use as bounding box] (-5.5,-1.5) rectangle (5.5,3.5);
    
    \draw [opacity=0.2] (-3,0) -- (-5,1);
    \draw [opacity=0.2] (-3,0) -- (-5,-1);
    
    \node [circle, draw=tancolor, fill=tancolor!20, minimum size=0.2cm] (u) at (-3,0) {};
    \node [tancolor] at (-3,-0.5) {$u$};
    
    \draw [opacity=0.2] (3,0) -- (5,1);
    \draw [opacity=0.2] (3,0) -- (5,-1);
    
    \node [circle, draw=tancolor, fill=tancolor!20, minimum size=0.2cm] (v) at (3,0) {};
    \node [tancolor] at (3,-0.5) {$v$};

    \draw [black, thick] (u) -- (v);
    \draw [edgegreen, opacity=0.3, line width=3pt] (u) -- (v);
    \node [edgegreen] at (0,-0.3) {$e$};
    
    \node [draw=tancolor, minimum width=0.2cm, minimum height=1.2cm] (stalku) at (-3,2.5) {};
    \draw [tancolor, opacity=0.2] ($(u)+(0,0.2)$) -- ($(stalku.south)+(0,-0.1)$);

    \node [left=0.1cm of stalku] {\textcolor{tancolor}{$\mathbb{R}$}};

    \draw [dotted] ($(stalku.north)+(0,-0.1)$) -- ($(stalku.south)+(0,0.1)$);

    \node at (-3,2.7) {\textcolor{blue}{$\bullet$}};
    \node at (-2.7,2.7) {\textcolor{blue}{$x_u$}};

    \node [draw=tancolor, minimum width=0.2cm, minimum height=1.2cm] (stalkv) at (3,2.5) {};
    \draw [tancolor, opacity=0.2] ($(v)+(0,0.2)$) -- ($(stalkv.south)+(0,-0.1)$);

    \node [right=0.1cm of stalkv] {\textcolor{tancolor}{$\mathbb{R}$}};

    \draw [dotted] ($(stalkv.north)+(0,-0.1)$) -- ($(stalkv.south)+(0,0.1)$);

    \node at (3,2.9) {\textcolor{pink}{$\bullet$}};
    \node at (3.29,2.9) {\textcolor{pink}{$x_v$}};
    
\end{tikzpicture}
\label{fig:2-network-sheaves-motivation}
\end{figure}

\noindent Under the discrete heat diffusion (\ref{eqn:discrete-graph-linear-heat-diffusion}), the representation of the node $v$ is continually updated as $x_v^{\operatorname{new}}=x_v - (\boldsymbol{L} \boldsymbol{x})_v$, where \begin{equation}
\label{eqn:graph-Laplacian-nodewise}
(\boldsymbol{L}\boldsymbol{x})_v = \sum_{u \in N_v} w_{uv}(x_u - x_v) = \sum_{u \in N_v} \sqrt{w_{uv}}(\sqrt{w_{uv}}x_v - \sqrt{w_{uv}}x_u).
\end{equation} 
The sheaf-theoretic perspective interprets this in the following manner. Attached to the edges $e$, too, are copies of the vector space $\mathbb{R}$. The vector spaces over nodes and edges are called \textbf{stalks}: 

\begin{figure}[H]
\centering
\begin{tikzpicture}[baseline, every node/.style={font=\small}]

    \path[use as bounding box] (-5.5,-1.5) rectangle (5.5,3.5);

    \draw [opacity=0.2] (-3,0) -- (-5,1);
    \draw [opacity=0.2] (-3,0) -- (-5,-1);

    \node [circle, draw=tancolor, fill=tancolor!20, minimum size=0.2cm] (u) at (-3,0) {};
    \node [tancolor] at (-3,-0.5) {$u$};

    \draw [opacity=0.2] (3,0) -- (5,1);
    \draw [opacity=0.2] (3,0) -- (5,-1);

    \node [circle, draw=tancolor, fill=tancolor!20, minimum size=0.2cm] (v) at (3,0) {};
    \node [tancolor] at (3,-0.5) {$v$};

    \draw [black, thick] (u) -- (v);
    \draw [edgegreen, opacity=0.3, line width=3pt] (u) -- (v);
    \node [edgegreen] at (0,-0.3) {$e$};

    \node [draw=tancolor, minimum width=0.2cm, minimum height=1.2cm] (stalku) at (-3,2.5) {};
    \draw [tancolor, opacity=0.2] ($(u)+(0,0.2)$) -- ($(stalku.south)+(0,-0.1)$);

    \node [left=0.1cm of stalku] {\textcolor{tancolor}{$\mathbb{R}$}};
    
    \draw [dotted] ($(stalku.north)+(0,-0.1)$) -- ($(stalku.south)+(0,0.1)$);
    \node at (-3,2.7) {\textcolor{blue}{$\bullet$}};
    \node at (-2.7,2.7) {\textcolor{blue}{${x}_u$}};

    \node [draw=tancolor, minimum width=0.2cm, minimum height=1.2cm] (stalkv) at (3,2.5) {};
    \draw [tancolor, opacity=0.2] ($(v)+(0,0.2)$) -- ($(stalkv.south)+(0,-0.1)$);

 \node [right=0.1cm of stalkv] {\textcolor{tancolor}{$\mathbb{R}$}};
    
    \draw [dotted] ($(stalkv.north)+(0,-0.1)$) -- ($(stalkv.south)+(0,0.1)$);

    \node at (3,2.9) {\textcolor{pink}{$\bullet$}};
    \node at (3.29,2.9) {\textcolor{pink}{${x}_v$}};

    \node [draw=edgegreen, minimum width=0.2cm, minimum height=1.2cm] (stalke) at (0,1.5) {};
    \draw [edgegreen, opacity=0.2] ($(0,0)+(0,0.2)$) -- ($(stalke.south)+(0,-0.1)$);

    \node at ($(stalke.north)+(0,0.3)$) {\textcolor{edgegreen}{$\mathbb{R}$}};
    \draw [dotted] ($(stalke.north)+(0,-0.1)$) -- ($(stalke.south)+(0,0.1)$);

    \node at (0,1.5) {$-$};

\end{tikzpicture}
\label{fig:graph_sheaf}
\end{figure}

The $u$th term in the sum (\ref{eqn:graph-Laplacian-nodewise}) is then viewed as follows. First, the node representations $x_u, x_v$ `send' themselves into the vector space over $e$ via the linear maps $\mathbb{R} \to \mathbb{R}$ given by multiplication by the matrices ${F}({u \leq e})= \begin{bmatrix}\sqrt{w_{uv}}\end{bmatrix}$. These modulations are called \textbf{restriction maps}, for reasons that will be made clear in Section~\ref{sec:sheaves}. The difference of images $\sqrt{w_{uv}}x_v-\sqrt{w_{uv}}x_u$ is taken in the stalk over $e$, then sent back into the stalk over $v$ via the transpose of the restriction map (here, this is again given by $\sqrt{w_{uv}}$):

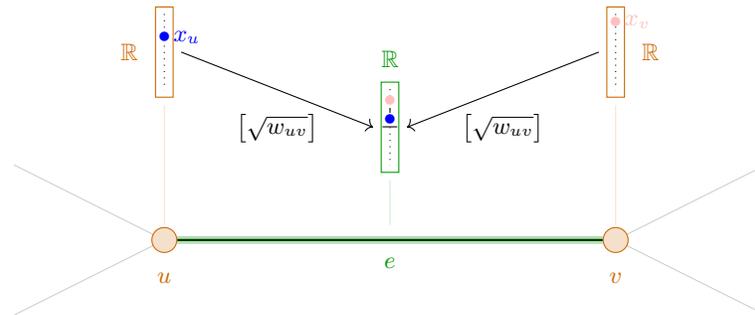
\begin{figure}[H]
\centering
\begin{tikzpicture}[baseline, every node/.style={font=\small}]

    \path[use as bounding box] (-5.5,-1.5) rectangle (5.5,3.5);

    \draw [opacity=0.2] (-3,0) -- (-5,1);
    \draw [opacity=0.2] (-3,0) -- (-5,-1);

    \node [circle, draw=tancolor, fill=tancolor!20, minimum size=0.2cm] (u) at (-3,0) {};
    \node [tancolor] at (-3,-0.5) {$u$};

    \draw [opacity=0.2] (3,0) -- (5,1);
    \draw [opacity=0.2] (3,0) -- (5,-1);

    \node [circle, draw=tancolor, fill=tancolor!20, minimum size=0.2cm] (v) at (3,0) {};
    \node [tancolor] at (3,-0.5) {$v$};

    \draw [black, thick] (u) -- (v);
    \draw [edgegreen, opacity=0.3, line width=3pt] (u) -- (v);
    \node [edgegreen] at (0,-0.3) {$e$};

    \node [draw=tancolor, minimum width=0.2cm, minimum height=1.2cm] (stalku) at (-3,2.5) {};
    \draw [tancolor, opacity=0.2] ($(u)+(0,0.2)$) -- ($(stalku.south)+(0,-0.1)$);

    \node [left=0.1cm of stalku] {\textcolor{tancolor}{$\mathbb{R}$}};
    
    \draw [dotted] ($(stalku.north)+(0,-0.1)$) -- ($(stalku.south)+(0,0.1)$);

    \node at (-3,2.7) {\textcolor{blue}{$\bullet$}};
    \node at (-2.7,2.7) {\textcolor{blue}{${x}_u$}};

    \node [draw=tancolor, minimum width=0.2cm, minimum height=1.2cm] (stalkv) at (3,2.5) {};
    \draw [tancolor, opacity=0.2] ($(v)+(0,0.2)$) -- ($(stalkv.south)+(0,-0.1)$);

 \node [right=0.1cm of stalkv] {\textcolor{tancolor}{$\mathbb{R}$}};
    
    \draw [dotted] ($(stalkv.north)+(0,-0.1)$) -- ($(stalkv.south)+(0,0.1)$);

    \node at (3,2.9) {\textcolor{pink}{$\bullet$}};
    \node at (3.29,2.9) {\textcolor{pink}{${x}_v$}};

    \node [draw=edgegreen, minimum width=0.2cm, minimum height=1.2cm] (stalke) at (0,1.5) {};
    \draw [edgegreen, opacity=0.2] ($(0,0)+(0,0.2)$) -- ($(stalke.south)+(0,-0.1)$);

    \node at ($(stalke.north)+(0,0.3)$) {\textcolor{edgegreen}{$\mathbb{R}$}};
    \draw [dotted] ($(stalke.north)+(0,-0.1)$) -- ($(stalke.south)+(0,0.1)$);

    \node at (0,1.5) {$-$};

    \draw [->] 
        ($(stalku.east)+(0.1,0)$) -- node[below=0.2cm, midway] {$\begin{bmatrix}\sqrt{w_{uv}}\end{bmatrix}$} 
        ($(stalke.west)+(-0.1,0)$);
        
    \draw [->] 
        ($(stalkv.west)+(-0.1,0)$) -- node[below=0.2cm, midway] {$\begin{bmatrix}\sqrt{w_{uv}}\end{bmatrix}$} 
        ($(stalke.east)+(0.1,0)$);

    \node at (0,1.6) {\textcolor{blue}{$\bullet$}};
    \node at (0,1.85) {\textcolor{pink}{$\bullet$}};
    \draw (0,1.69) -- (0,1.75);
    
\end{tikzpicture}
\caption{The graph Laplacian propagates node representations into edge-stalks, where their disagreement may be quantified. }
\label{fig:3-network-sheaves-motivation}
\end{figure}

Although each is a copy of $\mathbb{R}$, we view node representations $x_u$, $x_v$ as living in distinct spaces ${F}(u)$, ${F}(v)$. In this schematic, directly comparing node representations is illegal. The way to compare the representations of adjacent nodes is to map them to the stalk ${F}(e)$ of the edge between them and perform the operation there. 

Theorem~\ref{thm:GCNs-and-heat-diffusion} says that, in the limit, this procedure will `smoothen out' the node representations by way of successive comparisons to one another in the edge stalks. Eventually, the images in the edge stalks of node representations all become equal, $\sqrt{w_{uv}}x_u=\sqrt{w_{uv}}x_v$ for all nodes $u,v$. Consequently, the node representations themselves become equal.

\begin{figure}[H]
\centering
\begin{tikzpicture}[baseline]

    \path[use as bounding box] (-5.5,-1.5) rectangle (5.5,3.5);

    \node [circle, draw=tancolor, fill=tancolor!20, minimum size=0.2cm] (u) at (-3,0) {};
    \node [circle, draw=tancolor, fill=tancolor!20, minimum size=0.2cm] (v) at (3,0) {};

    \draw [black, thick] (u) -- (v);
    \draw [edgegreen, opacity=0.3, line width=3pt] (u) -- (v);

    \draw [opacity=0.2] (-3,0) -- (-5,1);
    \draw [opacity=0.2] (-3,0) -- (-5,0);
    \draw [opacity=0.2] (-3,0) -- (-5,-1);

    \draw [opacity=0.2] (3,0) -- (5,1);
    \draw [opacity=0.2] (3,0) -- (5,0);
    \draw [opacity=0.2] (3,0) -- (5,-1);

    \node [tancolor] at (-3,-0.5) {$u$};
    \node [tancolor] at (3,-0.5) {$v$};
    \node [edgegreen] at (0,-0.3) {$e$};

    \node [draw=tancolor, minimum width=0.2cm, minimum height=1.2cm] (squareu) at (-3,2.5) {};
    \draw [tancolor, opacity=0.2] ($(u)+(0,0.2)$) -- ($(squareu.south)+(0,-0.1)$);
    \node [left=0.1cm of squareu] {\textcolor{tancolor}{$\mathbb{R}$}};
    \draw [dotted] ($(squareu.north)+(0,-0.1)$) -- ($(squareu.south)+(0,0.1)$);
    \node at (-3,2.5) {$-$};

    \node [draw=edgegreen, minimum width=0.2cm, minimum height=1.2cm] (squaree) at (0,1.5) {};

    \draw [edgegreen, opacity=0.2] ($(0,0)+(0,0.2)$) -- ($(squaree.south)+(0,-0.1)$);
    \node [above=0.1cm of squaree] {\textcolor{edgegreen}{$\mathbb{R}$}};
    \draw [dotted] ($(squaree.north)+(0,-0.1)$) -- ($(squaree.south)+(0,0.1)$);
    \node at (0,1.5) {$-$}; 

    \node [draw=tancolor, minimum width=0.2cm, minimum height=1.2cm] (squarev) at (3,2.5) {};
    \draw [tancolor, opacity=0.2] ($(v)+(0,0.2)$) -- ($(squarev.south)+(0,-0.1)$);
    \node [right=0.1cm of squarev] {\textcolor{tancolor}{$\mathbb{R}$}};
    \draw [dotted] ($(squarev.north)+(0,-0.1)$) -- ($(squarev.south)+(0,0.1)$);
    \node at (3,2.5) {$-$};

    \draw [->]
        ($(squareu.east)+(0.1,0)$) --
        node[below=0.2cm, midway] {$\begin{bmatrix}
            \sqrt{w_{uv}}
        \end{bmatrix}$}
        ($(squaree.west)+(-0.1,0)$);

    \draw [->]
        ($(squarev.west)+(-0.1,0)$) --
        node[below=0.2cm, midway] {$\begin{bmatrix}
            \sqrt{w_{uv}}
        \end{bmatrix}$}
        ($(squaree.east)+(0.1,0)$);

    \node at (-3,2.8) {\textcolor{blue}{$\bullet$}}; 
    \node at (3,2.8) {\textcolor{pink}{$\bullet$}}; 
    \node at (-0.05,1.6) {\textcolor{blue}{$\bullet$}}; 
    \node at (0.05,1.6) {\textcolor{pink}{$\bullet$}}; 

\end{tikzpicture}
\end{figure}

It is worth reiterating the cause-and-effect here: equality in the \textit{edge} stalks leads to equality in the \textit{node} stalks, not the other way around (this will be made more precise in future sections). That is, in the diffusion limit we have ${F}({u \leq e})x_u={F}({u \leq e})x_v$, and \textit{this} is what entails $x_u=x_v$. What this means is that, at least in this special case, we might prevent oversmoothing simply by sufficiently breaking the symmetry between the two restriction maps incident to an edge.
For example, we might negate one of the restriction maps, in which case the picture in the diffusion limit might look something like this:

\begin{figure}[H]
\centering
\begin{tikzpicture}[baseline]

    \path[use as bounding box] (-5.5,-1.5) rectangle (5.5,3.5);

    \node [circle, draw=tancolor, fill=tancolor!20, minimum size=0.2cm] (u) at (-3,0) {};
    \node [circle, draw=tancolor, fill=tancolor!20, minimum size=0.2cm] (v) at (3,0) {};

    \draw [black, thick] (u) -- (v);
    \draw [edgegreen, opacity=0.3, line width=3pt] (u) -- (v);

    \draw [opacity=0.2] (-3,0) -- (-5,1);
    \draw [opacity=0.2] (-3,0) -- (-5,0);
    \draw [opacity=0.2] (-3,0) -- (-5,-1);

    \draw [opacity=0.2] (3,0) -- (5,1);
    \draw [opacity=0.2] (3,0) -- (5,0);
    \draw [opacity=0.2] (3,0) -- (5,-1);

    \node [tancolor] at (-3,-0.5) {$u$};
    \node [tancolor] at (3,-0.5) {$v$};
    \node [edgegreen] at (0,-0.3) {$e$};

    \node [draw=tancolor, minimum width=0.2cm, minimum height=1.2cm] (squareu) at (-3,2.5) {};
    \draw [tancolor, opacity=0.2] ($(u)+(0,0.2)$) -- ($(squareu.south)+(0,-0.1)$);
    \node [left=0.1cm of squareu] {\textcolor{tancolor}{$\mathbb{R}$}};
    \draw [dotted] ($(squareu.north)+(0,-0.1)$) -- ($(squareu.south)+(0,0.1)$);
    \node at (-3,2.5) {$-$};

    \node [draw=edgegreen, minimum width=0.2cm, minimum height=1.2cm] (squaree) at (0,1.5) {};
    \draw [edgegreen, opacity=0.2] ($(0,0)+(0,0.2)$) -- ($(squaree.south)+(0,-0.1)$);
    \node [above=0.1cm of squaree] {\textcolor{edgegreen}{$\mathbb{R}$}};
    \draw [dotted] ($(squaree.north)+(0,-0.1)$) -- ($(squaree.south)+(0,0.1)$);
    \node at (0,1.5) {$-$};

    \node [draw=tancolor, minimum width=0.2cm, minimum height=1.2cm] (squarev) at (3,2.5) {};
    \draw [tancolor, opacity=0.2] ($(v)+(0,0.2)$) -- ($(squarev.south)+(0,-0.1)$);
    \node [right=0.1cm of squarev] {\textcolor{tancolor}{$\mathbb{R}$}};
    \draw [dotted] ($(squarev.north)+(0,-0.1)$) -- ($(squarev.south)+(0,0.1)$);
    \node at (3,2.5) {$-$}; 

    \draw [->]
        ($(squareu.east)+(0.1,0)$) --
        node[below=0.2cm, midway] {$\begin{bmatrix}
            -\sqrt{w_{uv}}
        \end{bmatrix}$}
        ($(squaree.west)+(-0.1,0)$);

    \draw [->]
        ($(squarev.west)+(-0.1,0)$) --
        node[below=0.2cm, midway] {$\begin{bmatrix}
            \sqrt{w_{uv}}
        \end{bmatrix}$}
        ($(squaree.east)+(0.1,0)$);

    \node at (3,2.8) {\textcolor{pink}{$\bullet$}}; 
    \node at (-0.05,1.6) {\textcolor{blue}{$\bullet$}}; 
    \node at (0.05,1.6) {\textcolor{pink}{$\bullet$}}; 
    \node at (-3,2.2) {\textcolor{blue}{$\bullet$}};  

\end{tikzpicture}
\end{figure}

In particular, note that although the \textit{images} of the node representations in the edge stalks `smoothen out', the node representations themselves in this picture have remained distinct.\footnote{This type of `twisted' or `lying' sheaf, despite its simplicity, turns out to be quite powerful, as Theorems~\ref{thm:2-way-linear-separability} and~\ref{thm:k-way-linear-separability} will (much later) show.}
More generally, we could allow the stalks ${F}(u)$, ${F}(v)$, ${F}(e)$ and restriction maps ${F}({u \leq e})$, ${F}({v \leq e})$ to be arbitrary, and study diffusion with the resulting `Laplacian' \begin{equation}
\label{eqn:motivation-network-sheaf-Laplacian}
    (\boldsymbol{L}_{F} \boldsymbol{x})_v= \sum_{v,u \leq e} {F}({v \leq e})^\top ({F}({v \leq e}) \boldsymbol{x}_v - {F}({u \leq e}) \boldsymbol{x}_u).
\end{equation}
Roughly speaking, such an assignment ${F}$ of of objects to nodes/edges of $G$ and morphisms to their incidences is a \textbf{sheaf} on the graph $G$; the operator $\boldsymbol{L}_{F}$ its \textbf{sheaf Laplacian}.



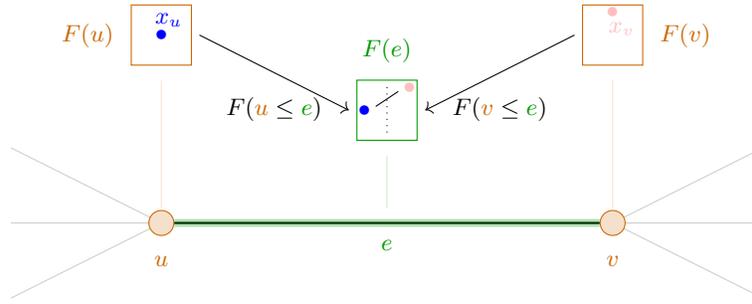
\begin{figure}[H]
\centering
\begin{tikzpicture}[baseline, every node/.style={font=\small}]
    \path[use as bounding box] (-5.5,-1.5) rectangle (5.5,3.5);

    \draw [opacity=0.2] (-3,0) -- (-5,1);
    \draw [opacity=0.2] (-3,0) -- (-5,0);
    \draw [opacity=0.2] (-3,0) -- (-5,-1);

    \node [circle, draw=tancolor, fill=tancolor!20, minimum size=0.2cm] (u) at (-3,0) {};
    \node [tancolor] at (-3,-0.5) {$u$};

    \draw [opacity=0.2] (3,0) -- (5,1);
    \draw [opacity=0.2] (3,0) -- (5,0);
    \draw [opacity=0.2] (3,0) -- (5,-1);

    \node [circle, draw=tancolor, fill=tancolor!20, minimum size=0.2cm] (v) at (3,0) {};
    \node [tancolor] at (3,-0.5) {$v$};

    \draw [black, thick] (u) -- (v);
    \draw [edgegreen, opacity=0.3, line width=3pt] (u) -- (v);
    \node [edgegreen] at (0,-0.3) {$e$};

    \node [draw=tancolor, minimum size=0.8cm] (squareu) at (-3,2.5) {};
    \draw [tancolor, opacity=0.2] ($(u)+(0,0.2)$) -- ($(squareu.south)+(0,-0.2)$);

    \node [left=0.1cm of squareu] {\textcolor{tancolor}{${F}(u)$}};

    \node at (-3,2.5) {\textcolor{blue}{$\bullet$}};
    \node at (-2.9,2.7) {\textcolor{blue}{$x_u$}};

    \node [draw=edgegreen, minimum size=0.8cm] (squaree) at (0,1.5) {};
    \draw [edgegreen, opacity=0.2] ($(0,0)+(0,0.2)$) -- ($(squaree.south)+(0,-0.2)$);

    \node [above=0.1cm of squaree] {\textcolor{edgegreen}{${F}(e)$}};
    \draw [dotted] ($(squaree.north)+(0,-0.1)$) -- ($(squaree.south)+(0,0.1)$);

    \node at (-0.3,1.5) {\textcolor{blue}{$\bullet$}};
    \node at (0.3,1.8) {\textcolor{pink}{$\bullet$}};
    \draw (-0.15,1.55) -- (0.15,1.75);
    
    \node [draw=tancolor, minimum size=0.8cm] (squarev) at (3,2.5) {};
    \draw [tancolor, opacity=0.2] ($(v)+(0,0.2)$) -- ($(squarev.south)+(0,-0.2)$);

    \node [right=0.1cm of squarev] {\textcolor{tancolor}{${F}(v)$}};

    \node at (3,2.8) {\textcolor{pink}{$\bullet$}};
    \node at (3.12,2.58) {\textcolor{pink}{$x_v$}};

    \draw [->] 
        ($(squareu.east)+(0.1,0)$) -- 
        node[below=0.2cm, midway] {${F}({\textcolor{tancolor}{u}\leq \textcolor{edgegreen}{e}})$} 
        ($(squaree.west)+(-0.1,0)$);
        
    \draw [->] 
        ($(squarev.west)+(-0.1,0)$) -- 
        node[below=0.2cm, midway] {${F}({\textcolor{tancolor}{v}\leq \textcolor{edgegreen}{e}})$} 
        ($(squaree.east)+(0.1,0)$);
        
\end{tikzpicture}
\caption{A general network sheaf $F$. A diffusion step with the sheaf Laplacian (Equation~\ref{eqn:motivation-network-sheaf-Laplacian}) $\boldsymbol{L}_F$ adjusts the representation $x_v$ of node $v$ by aggregating its `public disagreement' with each neighbor $u$. }
\label{fig:sheaf_square_final}
\end{figure}
We end this section with two foreshadowing remarks. The first remark is that the graph Laplacians considered in this motivating setup are \textit{unnormalized}. We will see in Section~\ref{sec:linear-separation-power-sheaf-diffusion} that simply normalizing the graph Laplacian before performing diffusion increases separation power with respect to the present discussion. Nevertheless, it cannot guarantee separation power: for that, we will need to embrace sheaf Laplacians. The second remark is that, because graphs are one-dimensional objects, the network sheaves defined here are not illustrative of the entire picture: \textit{compositionality} of restriction maps will be a crucial feature of general sheaves, the subject to which we now turn.



\section{Sheaves: Definitions and Examples}
\label{sec:sheaves}
This section rapidly introduces sheaves, formally and generally. In deference to the needs of downstream applications, the presentation's focus necessarily skews somewhat nontraditional after Section~\ref{sec:sheaves-supported-on-top-spaces}. 


On the other hand, Section~\ref{sec:sheaves-supported-on-top-spaces} treats entirely standard material which may be found e.g. in Part III Algebraic Geometry or in the books of Bredon~\cite{bredon2012sheaf}, Hartshorne~\cite{hartshorne2013algebraic}, or Vakil~\cite{vakil2025rising}. We consequently take the liberty to omit many proofs in this section, freeing up space to discuss less standard material later on.

We will briefly review sheaves supported on arbitrary topological spaces before specializing to sheaves supported on finite posets. The setting of posets will be where the majority of exposition takes place: it is specific enough to admit friendly theories of cohomology, heat diffusion, etc. while being general enough to subsume many interesting data structures as special cases. Eventually, the setting of \textit{network sheaves}, i.e., sheaves supported on graphs, will come into focus as our eye turns toward extant applications in deep learning and making the motivating ideas introduced in the previous section precise.



Throughout, $\mathsf{D}$ will denote a `nice enough' category in which data shall live. Initially $\mathsf{D}$ will be arbitrary; before long we will require it to be abelian. We will refer liberally to `elements' of objects in $\mathsf{D}$. It is not harmful to imagine $\mathsf{D}=\mathbb{R}\mathsf{Vect}$ throughout. In fact, we will eventually have need for ordering and completeness, fixing such a choice.

\subsection{Sheaves Supported on Topological Spaces}
\label{sec:sheaves-supported-on-top-spaces}
Let $X$ be a topological space. The collection $\tau_X$ of open sets in $X$ forms a filtered poset with respect to inclusion. The ability to `locally restrict data to smaller open subsets' gives rise to the concept of a \textit{pre}sheaf on $X$.

\begin{definition}
    A \textbf{presheaf $\mathcal{F}$ on $X$} is an inverse system indexed by $\tau_X$, i.e.,  a contravariant functor from $\tau_{X}$ to $\mathsf{D}$. That is to say, 
\begin{enumerate}
\item For each open $U \subset X$, there is an object $\mathcal{F}(U)$ of $\mathsf{D}$. This object is also denoted $\Gamma(U,\mathcal{F})$. Its elements are called \textbf{sections of $\mathcal{F}$}. Elements of $\Gamma(X, \mathcal{F})$ are called \textbf{global sections of $\mathcal{F}$}. 
\item For each inclusion of open sets $V \subset U$, there is a morphism $\mathcal{F}(U) \xrightarrow{} \mathcal{F}(V)$, denoted $\mathcal{F}(U \supset V)$ or $\mathcal{F}_{U \supset V}$ , called a \textbf{restriction morphism}. Given a section $s \in \mathcal{F}(V)=\Gamma(V,\mathcal{F})$, the notation $s |_{V}$ is sometimes used to denote $\mathcal{F}_{U \supset V}(s) \in \mathcal{F}(V)$. 
\item The restriction morphisms are (contravariantly) functorial, meaning that $\mathcal{F}_{U \subset U}=1_{\mathcal{F}(U)}$ and given three opens $W \subset V \subset U$ one has  $\mathcal{F}_{V\supset W} \circ \mathcal{F}_{U\supset V}=\mathcal{F}_{U\supset W}$.
\end{enumerate}

\noindent Presheaves on $X$ and natural transformations between them form a (functor) category, which we shall denote $\mathsf{pShv}_{\mathsf{D}}(X):=[\tau_X^{\operatorname{op}}, \mathsf{D}]$. Each morphism of presheaves in particular carries with it a morphism of global sections; the functor to which this gives rise is called the \textbf{global sections functor} $\Gamma(X, -):\mathsf{pShv}_\mathsf{D}(X) \to \mathsf{D}$ on $X$.
\end{definition}

\begin{example}
\label{ex:constant-presheaf}
Any object $D$ of $\mathsf{D}$ gives rise to a \textbf{constant presheaf} $\underline{D}$, defined as $\underline{D}(U)=D$ for all $U \in \tau_X$ with each restriction map the identity.
\end{example}

\begin{example}
\label{ex:presheaf-cont-ftns}
    The \textbf{presheaf of continuous\footnote{Or smooth, holomorphic, real analytic, regular, differential forms, etc. as relevant.} functions} on a topological space $X$ is defined as the assignment $U \mapsto \{\text{continuous maps }U \to \mathbb{R}\}$ to each $U \in \tau_X$, with restriction maps given by function restriction. 
\end{example}

\noindent A presheaf $\mathcal{F}$ is defined in terms of open sets, but the underlying topological space $X$ consists of points. In order to isolate the behavior of $\mathcal{F}$ around a specific point $p \in X$, we would like to consider `smaller and smaller neighborhoods' about $p$. This leads to the following (co)limiting construction.

\begin{definition}
 The \textbf{stalk of $\mathcal{F}$ at $x \in X$} is the object of $\mathsf{D}$ given by the colimit
\begin{equation}
\mathcal{F}_x := \varinjlim_{U \ni x} \mathcal{F}(U),
\end{equation}
taken over the filtered poset of open neighborhoods \(U\) of \(x\). \\

Explicitly, this colimit consists of equivalence classes of pairs \((U, s)\) where \(U\) is a neighborhood of \(x\) and \(s \in \mathcal{F}(U)\), with the equivalence relation
\begin{equation}
(U, s) \sim (V, t) \iff  \text{there exists }W \subset U \cap V \text{ such that } s|_W = t|_W.
\end{equation}

The class $[U,s] \in \mathcal{F}_{p}$ is written as $s_{p}$ and called the \textbf{germ} of the section $s$ at $p$. Note that a presheaf morphism $f:\mathcal{F} \to \mathcal{G}$ induces morphisms of stalks $f_{p}:\mathcal{F}_{p} \to \mathcal{G}_{p}$ by defining
\begin{equation}
f_{p}\big( [U,s] \big):=[U, f_{U}(s)],
\end{equation} giving rise to a \textbf{stalk functor} $X \to \mathsf{D}$. 




\end{definition}



Importantly, the locally compatible assignment of data to $X$ captured by a presheaf need not entail global compatibility, and the restrictions of a section may not determine it. The notion of a sheaf adds extra conditions which allow for more seamless travel between local and global phenomena. 

\begin{definition}
\label{def:sheaf-on-top-space}
A \textbf{sheaf} on $X$ is a presheaf $\mathcal{F}$ on $X$ satisfying the following. 

\textbf{1. (Locality)} Suppose $U \subset X$ is open, and let $\{ U_{i} \}_{i \in I}$ be an open cover of $U$ by subsets $U_{i} \subset U$. If $s$, $t \in \mathcal{F}(U)$ satisfy $s |_{U_{i}}=t |_{U_{i}}$ for all $i \in I$, then in fact $s=t$. 

\textbf{2. (Gluing)} Suppose $U \subset X$ is open, and let $\{ U_{i} \}_{i \in I}$ be an open cover of $U$ by subsets $U_{i} \subset U$. If a family of sections $\{ s_{i} \in \mathcal{F}(U_{i}) \}_{i}$ has pairwise agreement on all overlap of their domains\footnote{That is, if $U_{i} \cap U_{i'} \neq \emptyset$ then $s_{i} |_{U_{i} \cap U_{i'}}=s_{i'} |_{U_{i} \cap U_{i'}}$ for all indices $i,i'$. }, then they `patch together': there is a section $s \in \mathcal{F}(U)$ such that $s |_{U_{i}}=s_{i}$ for all $i \in I$.  By \textbf{(1)}, such a section is unique.\\

$\mathcal{F}(\emptyset)$ is always the final object of $\mathsf{D}$. $\mathsf{D}$-valued sheaves on $X$ form a full subcategory of $\mathsf{pShv}_\mathsf{D}(X)$, denoted $\mathsf{Shv}_\mathsf{D}(X)$.
\end{definition}




\begin{example}
    The presheaf of continuous functions on a space $X$ (Example~\ref{ex:presheaf-cont-ftns}) is a sheaf, by the pasting lemma. The presheaf of \textit{bounded} continuous functions, say, on $\mathbb{R}$, is \textit{not} a sheaf: gluing fails.
\end{example}

\begin{example}
    The constant presheaf is \textit{not} generally a sheaf, e.g. because every sheaf needs to assign to $\emptyset$ the final object of $\mathsf{D}$. 
\end{example}


\noindent There is a universal way to turn any presheaf $\mathcal{F}$ into a sheaf.

\begin{definition}[Sheafification] 
Let $X$ be a topological space, $\mathcal{F}$ a presheaf on $X$. The \textbf{sheafification of $\mathcal{F}$} is a new sheaf $\mathcal{F}^{+}$, together with a morphism of (pre)sheaves $\theta:\mathcal{F} \to \mathcal{F}^{+}$, satisfying the universal property that any morphism $\mathcal{F} \xrightarrow{\varphi}\mathcal{G}$, $\mathcal{G}$ a sheaf, factors uniquely through this new sheaf $\mathcal{F}^{+}$:

\begin{center}
\begin{tikzcd}
\mathcal{F} \arrow[r, "\varphi"] \arrow[d, "\theta"']             & \mathcal{G} \\
\mathcal{F}^+ \arrow[ru, "\exists ! \overline{\varphi}"', dashed] &            
\end{tikzcd}
\end{center}

this defines $\mathcal{F}^{+}$ up to isomorphism, if it exists. Indeed, exist it does: $\mathcal{F}^{+}$ may be obtained as
\begin{align}
\mathcal{F}^{+}(U)=\left\{ s:U \to \coprod_{p \in U} \mathcal{F}_{p} : \begin{aligned}
&\textcolor{Thistle}{(1) \ s(p) \in \mathcal{F}_{p} \text{ for all }p}  \\
&\textcolor{SkyBlue}{(2) \ \forall p \in U, \exists \text{neighborhood } p \in V \subset U \text{ and }} \\
&\quad \quad \textcolor{SkyBlue}{t \in \mathcal{F}(V) \text{ such that } s(q)=[V,t] \ \forall q \in V}
\end{aligned} \right\}.
\end{align}

with the natural transformation $\theta:\mathcal{F} \to \mathcal{F}^{+}$ specified via $U$-component
\begin{align}
\theta_{U}:\mathcal{F}(U) &\to \mathcal{F}^{+}(U) \\
s &\mapsto (p \mapsto [U,s]).
\end{align}
\end{definition}

The most important property of sheafification is that it preserves stalks: $(\mathcal{F}^{+})_{p} \cong_{\theta_{p}} \mathcal{F}_p$ for all $p \in X$.

\begin{example}
\label{ex:constant-sheaf}
      Let $\mathsf{D}$ be a $\mathsf{Set}$-like category, $D$ an object of $\mathsf{D}$. The \textbf{constant sheaf} $\underline{D}$ on $X$ is defined to be the sheafification of the constant presheaf. Explicitly, $\underline{D}$ is given by the rule \begin{equation}
        U \mapsto \{ \text{locally constant maps } U \to  S\}
      \end{equation}
      with restriction maps given by function restriction.
\end{example}

Generally speaking, the situation with the constant sheaf is the rule, not the exception: oftentimes the `obvious' way to build a new sheaf from old merely gives a presheaf. By constructing an optimally parsimonious sheaf out of a given presheaf, the sheafification functor offers a canonical resolution to this obstruction. Here are two more definitions in this flavor: 

\begin{definition}
The \textbf{direct sum} $\bigoplus_{i \in I} \mathcal{F}_{i}$ of a family of sheaves $\{ \mathcal{F}_{i} \}_{i \in I}$ on a topological space $X$ is defined as the sheafification of the presheaf
\begin{align}
U \mapsto \bigoplus_{i \in I}\mathcal{F}_{i}(U).
\end{align}
\end{definition}

\begin{definition}
    Let $X$ be a topological space. Let $f:\mathcal{F} \to \mathcal{G}$ be a morphism of (pre)sheaves on $X$. The \textbf{sheaf image} $\operatorname{im}f$ is defined as the sheafification\footnote{It follows from the universal property of sheafification and Theorem~\ref{thm:testing-on-stalks} below that $\operatorname{im}f$ can be naturally identified with a subsheaf of $\mathcal{G}$.} of the presheaf \begin{equation}
        U \mapsto \operatorname{im}f_U.
    \end{equation}

\end{definition}

\noindent There is also a self-evident definition of $\ker f$. Notably, $\ker f$ is a sheaf whenever $\mathcal{F}$ and $\mathcal{G}$ are.

\begin{definition}
    Let $f:\mathcal{F} \to \mathcal{G}$ be a morphism of presheaves valued in a category where kernels make sense. The \textbf{presheaf kernel} of $f$, $\ker f$, is the presheaf specified by $(\ker f)(U):=\ker \big(f_{U}:\mathcal{F}(U) \to \mathcal{G}(U) \big)$. If $f:\mathcal{F} \to \mathcal{G}$ is a morphism of sheaves, then $\ker \mathcal{F}$ is in fact a sheaf — no sheafification required. 
\end{definition}

\begin{definition}
    A sheaf morphism $f: \mathcal{F} \to \mathcal{G}$ is \textbf{injective} if its kernel if trivial. It is \textbf{surjective} if its image equals $\mathcal{G}$.
\end{definition}

\begin{definition}
    \textbf{Cochain complexes} and \textbf{exactness} are defined for sheaves in terms of images and kernels in the usual way. An important property of the global sections functor $\Gamma(X, -)$ is that it is left-exact. In a sense, its failure to be right-exact gives rise to sheaf cohomology (Section~\ref{sec:sheaf-cohomology-diffusion}).
\end{definition}   

In practice, sheafification does not require one to think `as hard' as it might first seem. Indeed, sheafification preserves stalks, and many properties can be `tested at the level of stalks'. For instance:

\begin{theorem}[Testing at the level of stalks]
\label{thm:testing-on-stalks}
Let $f:\mathcal{F} \to \mathcal{G}$ be a morphism of sheaves. For each assertion that follows, $\mathsf{D}$ is understood to be a category where that assertion is defined.
\begin{enumerate}
    \item $(\operatorname{im} f)_{p}=\operatorname{im}(f_{p}:\mathcal{F}_{p} \to \mathcal{G}_{p})$ for all $p \in X$.
    \item $(\ker f)_{p}=\ker (f_{p}: \mathcal{F}_{p} \to \mathcal{G}_{p})$ for all $p \in X$.
    \item  $\mathcal{F}$ is the zero sheaf if and only if $\mathcal{F}_{p}=(0)$ for all $p \in X$
    \item If $\mathcal{F} \subset \mathcal{G}$ is an inclusion of sheaves, one has
$\mathcal{F}=\mathcal{G} \iff \mathcal{F}_{p}=\mathcal{G}_{p} \ \forall p \in X$.
\item $f$ is injective if and only if $f_p:\mathcal{F}_p \to \mathcal{G}_p$ is injective for all $p \in X$
\item $f$ surjective if and only if $f_p:\mathcal{F}_p \to \mathcal{G}_p$ is surjective for all $p \in X$ 
\item A cochain complex of sheaves over a topological space $X$
\begin{align}
\cdots \to \mathcal{F}^{i-1} \to \mathcal{F}^{i} \xrightarrow{d} \mathcal{F}^{i+1} \to \cdots
\end{align}
is exact at $i$ if and only if for every $p \in X$, the cochain complex 
\begin{align}
\cdots \to \mathcal{F}^{i-1}_{p} \to \mathcal{F}^{i}_{p} \xrightarrow{d_{p}} \mathcal{F}^{i+1}_{p} \to \cdots
\end{align}
is exact at $i$.
\end{enumerate}

\end{theorem}

In the following section, we will discover that sheaves on posets — i.e., most applied sheaves — are, in a sense, `defined stalkwise'. Theorem~\ref{thm:testing-on-stalks} is great news, in this case: it tells us that images, kernels, injectivity, surjectivity, exactness, etc. all behave in applied contexts just as one might initially anticipate. 

\subsection{Sheaves Supported on Posets}
\label{sec:sheaves-supported-on-finite-posets}


So far, it is not clear why our provisional notion of a sheaf on a graph $G$ (Section~\ref{sec:intuition-oversmoothing-heterophily}) — an assignment of objects to nodes/edges and morphisms to their incidences — is an earnest sheaf in the manner just defined. Indeed, there are not even any open sets to assign data to, since $G$ was never given a topology. There was, however, \textit{some} sense of functoriality involved, as objects and morphisms in $\mathsf{D}$ were `indexed' by $G$. A more apt description of the construction in Section~\ref{sec:intuition-oversmoothing-heterophily} is that it captures the notion of a diagram in $\mathsf{D}$ supported on $G$ (appropriately viewed as a poset).\\

In this section, we establish an imperfect dictionary between topology and order theory witnessing conditions under which the two ideas are (categorically) equivalent. While the topological sheaf definition in Section~\ref{sec:sheaves-supported-on-top-spaces} is certainly a relevant perspective and will be invoked later, the diagrammatic formulation of sheaves discussed herein will become a true cornerstone of this document. 

\begin{definition}

Let $\mathsf{I}$ be an `indexing category'.  A \textbf{diagram in $\mathsf{D}$ supported on\footnote{Or \textbf{indexed by} $\mathsf{I}$, or \textbf{of shape} $\mathsf{I}$.} $\mathsf{I}$ }is a (covariant) functor $F:\mathsf{I} \to \mathsf{D}$. Often the object of $\mathsf{D}$ to which $i$ in $\mathsf{I}$ gets assigned under the functor will be denoted $F_{i}$; likewise for morphisms.\\

\noindent We denote category of diagrams on $\mathsf{I}$ and natural transformations between them by $\mathsf{Diag}_{\mathsf{D}}(\mathsf{I}):=[\mathsf{I}, \mathsf{D}]$
\end{definition}

\noindent The idea is that $\mathsf{I}$ `indexes both morphisms and objects'. The actual objects and morphisms in $\mathsf{I}$ do not matter — only their interrelationships.

\begin{example}
    If $\mathsf{I}=\square$ is the category 
\begin{tikzcd}
\bullet \arrow[r] \arrow[d] & \bullet \arrow[d, "\text{,}"] \\
\bullet \arrow[r]           & \bullet     
\end{tikzcd} then a diagram $\mathsf{\square} \to \mathsf{D}$ is precisely the data of a commutative square in $\mathsf{D}$. 
\end{example}

\begin{example}
\label{ex:presheaves-as-diagrams}
    If $\mathsf{I}$ is a filtered poset, then a diagram $\mathsf{I} \to \mathsf{D}$ is precisely a directed system in $\mathsf{D}$ indexed by $\mathsf{I}$. In this picture, an inverse system corresponds to a diagram $\mathsf{I}^{\operatorname{op}} \to \mathsf{D}$. It follows that, for a given topological space $X$, $\mathsf{pShv}_{\mathsf{D}}(X)=\mathsf{Diag}_{\mathsf{D}}(\tau_X^{\operatorname{op}})$. Sheaves on $X$ are diagrams on $\tau_X^{\operatorname{op}}$ subject to the locality and gluing axioms.
\end{example}

\begin{example}
\label{ex:diagram-on-graph}
    A graph $G$ may be endowed with a (strict) partial order by declaring $v \leq e$ when $e$ is an edge incident to $v$. A diagram ${F}:G \to \mathsf{D}$ indexed by the poset $G$ in this way consists of an assignment to each node $v$ an object ${F}(v)$ of $\mathsf{D}$, to each edge $e$ an object ${F}(e)$ of $\mathsf{D}$, and to each incidence $v \leq e$ a morphism ${F}({v \leq e}) \in \operatorname{Hom}_{\mathsf{D}}\big({F}(v), {F}(e)\big)$. 

\begin{figure}[H]
\centering
\begin{tikzpicture}[baseline, every node/.style={font=\small}]
    \path[use as bounding box] (-5.5,-1.5) rectangle (5.5,3.5);

    \draw [opacity=0.2] (-3,0) -- (-5,1);
    \draw [opacity=0.2] (-3,0) -- (-5,0);
    \draw [opacity=0.2] (-3,0) -- (-5,-1);

    \node [circle, draw=tancolor, fill=tancolor!20, minimum size=0.2cm] (u) at (-3,0) {};
    \node [tancolor] at (-3,-0.5) {$u$};

    \draw [opacity=0.2] (3,0) -- (5,1);
    \draw [opacity=0.2] (3,0) -- (5,0);
    \draw [opacity=0.2] (3,0) -- (5,-1);

    \node [circle, draw=tancolor, fill=tancolor!20, minimum size=0.2cm] (v) at (3,0) {};
    \node [tancolor] at (3,-0.5) {$v$};

    \draw [black, thick] (u) -- (v);
    \draw [edgegreen, opacity=0.3, line width=3pt] (u) -- (v);
    \node [edgegreen] at (0,-0.3) {$e$};

    \node [draw=tancolor, minimum size=0.8cm] (squareu) at (-3,2.5) {};
    \draw [tancolor, opacity=0.2] ($(u)+(0,0.2)$) -- ($(squareu.south)+(0,-0.2)$);

    \node [left=0.1cm of squareu] {\textcolor{tancolor}{${F}(u)$}};


    \node [draw=edgegreen, minimum size=0.8cm] (squaree) at (0,1.5) {};
    \draw [edgegreen, opacity=0.2] ($(0,0)+(0,0.2)$) -- ($(squaree.south)+(0,-0.2)$);

    \node [above=0.1cm of squaree] {\textcolor{edgegreen}{${F}(e)$}};

    
    \node [draw=tancolor, minimum size=0.8cm] (squarev) at (3,2.5) {};
    \draw [tancolor, opacity=0.2] ($(v)+(0,0.2)$) -- ($(squarev.south)+(0,-0.2)$);

    \node [right=0.1cm of squarev] {\textcolor{tancolor}{${F}(v)$}};


    \draw [->] 
        ($(squareu.east)+(0.1,0)$) -- 
        node[below=0.2cm, midway] {${F}({\textcolor{tancolor}{u}\leq \textcolor{edgegreen}{e}})$} 
        ($(squaree.west)+(-0.1,0)$);
        
    \draw [->] 
        ($(squarev.west)+(-0.1,0)$) -- 
        node[below=0.2cm, midway] {${F}({\textcolor{tancolor}{v}\leq \textcolor{edgegreen}{e}})$} 
        ($(squaree.east)+(0.1,0)$);
        
\end{tikzpicture}
\end{figure}
    
\end{example}

Example~\ref{ex:diagram-on-graph} concerns the object provisionally studied in Section~\ref{sec:intuition-oversmoothing-heterophily}. There, we dubbed it a sheaf $F$ on $G$. More generally:

\begin{definition}
\label{def:sheaf-on-poset}
    Let $S$ be a poset. A diagram $F:S \to \mathsf{D}$ is called a \textbf{sheaf supported on $S$}. Given $s \in S$, the object $F(s)$ of $\mathsf{D}$ is called the \textbf{stalk of $F$ over $s$}. The morphisms $F(s \leq s'): F(s) \to F(s')$ are called the \textbf{restriction maps} of $F$.
\end{definition}

For the time being, we reserve the unadorned term `sheaf' and calligraphic typesetting (e.g. $\mathcal{F}$) for sheaves on topological spaces (Definition~\ref{def:sheaf-on-top-space}). 
In view of Example~\ref{ex:presheaves-as-diagrams}, the decision to share terminology may appear to have some intuitive substance. After all, Definition~\ref{def:sheaf-on-poset} and Example~\ref{ex:presheaves-as-diagrams} both concern a diagram in $\mathsf{D}$ supported on a particular poset. Of course, the two posets involved are very different: the lattice of open sets for a general topological space can be very complicated, while the incidence structure of a graph $G$ is not even filtered unless $G$ is complete. It is not obvious how to convert between the two. Indeed, the topological category turns out to be, by nature, too general to draw a precise correspondence. At the very least, one would need a canonical embedding $X \hookrightarrow \tau_X$; some way to choose \textit{which} open set gets assigned to a particular point. Perhaps surprisingly, it turns out that having this is sufficient to establish a relationship reconciling definitions~\ref{def:sheaf-on-top-space} and~\ref{def:sheaf-on-poset}.  



\begin{theorem}
\label{thm:sheaves-diagrams-category-equivalence}
    The category $\mathsf{Diag}_{\mathsf{D}}(S)$ of diagrams $F$ on a poset $S$ is equivalent to the category $\mathsf{Shv}_{\mathsf{D}}(X_S)$  of sheaves $\mathcal{F}$ on its Alexandrov topological space $X_S$. Under this equivalence, 

    \begin{itemize}
        \item (Diagram of stalks) The stalk $\mathcal{F}_x$ over a point $x$ corresponds to the object $F(x)$ of $\mathsf{D}$;
        \item (Sheaf of sections) A section $s \in \Gamma(U, \mathcal{F})$ of $\mathcal{F}$  over $U$ corresponds to a tuple $(s_x)_{x \in U \subset S}\in \bigoplus_{x \in U}F(x)$ satisfying \begin{equation} x \leq y \implies 
        F_{x \leq y}s_{x}=s_{y}
        \end{equation}
        for all $x,y \in U$.
    \end{itemize}

\end{theorem}

\noindent In order to unpack Theorem~\ref{thm:sheaves-diagrams-category-equivalence}, we first need to understand what is meant by the term `Alexandrov topological space of a poset'. Doing so entails a very brief digression into a relationship between topology and order theory.

\subsubsection{From Topology to Order Theory, and Back}
\label{sec:digression-topology-to-order-theory}
The points in a topological space $(X, \tau)$ carry a canonical preordering.

\begin{definition}
\label{def:specialization-preorder}
The \textbf{specialization preorder} on $(X, \tau)$ is defined by declaring
\[
x \leq y \iff x \in \overline{\{ y \}},
\]
where $\overline{\{ y \}}$ denotes the closure of $\{ y \}$. We say $x$ is a \textbf{specialization} of $y$.\footnote{Some conventions (e.g. Hartshorne's~\cite{hartshorne2013algebraic}) instead say $y$ is a specialization of $x$. } Any continuous map $f:X \to Y$ is monotone with respect to the specialization preorder, giving rise to a ``specialization functor" ${W}:\mathsf{Top} \to \mathsf{PrePos}$.

If $X$ satisfies the $T_0$ Axiom\footnote{Two points $x$ and $y$ of a topological space $X$ are said to be \textbf{topologically distinguishable} if they have different (open) neighborhoods. A space $X$ in which every pair of points are topologically distinguishable is said to satisfy the \textbf{$T_{0}$ Axiom}.}, this preorder is in fact a partial order, and is called the \textbf{specialization order}.\footnote{This is indeed a preorder: certainly $x \in  \overline{\{ x \}}$ for all $x \in X$, and if $x \in \overline{\{ y \}}$ and $y \in \overline{\{ z \}}$ then $x \in \overline{\{ x \}} \subset  \overline{\{ y \}} \subset  \overline{ \overline{\{ z \}}}=\overline{\{ z \}}$. If $X$ satisfies the $T_0$ Axiom, we have both $x \leq y$ and $y \leq x$, then in fact $x=y$. Indeed, $x \leq y$ implies (cf. Remark~\ref{rmk:unwinding-specializatino-preorder}) that $y$ belongs to every neighborhood of $x$. Meanwhile, $y \leq x$ implies that $x$ belongs to every neighborhood of $y$. Two points in a $T_{0}$ space having the same neighborhoods are equal.}
\end{definition}

\begin{remark}
\label{rmk:unwinding-specializatino-preorder}
    Unwinding Definition~\ref{def:specialization-preorder}, one has $x \leq y$ if and only if $x$ belongs to every closed set containing $y$. Equivalently, $x \leq y$ if and only if $y$ belongs to every \textit{open} set containing (that is, every neighborhood of) $x$. This is the sense in which $y$ is `more general' than $x$: it is contained in more open sets.
\end{remark}

    


With sheaves in mind, we would like to draw an equivalence between posets and certain topologies. \textit{A priori}, such an objective is ill-posed: it wants to go from an order relation between two points to an inclusion of two open sets, but \textit{which} two open sets? If $X$ is Alexandrov, then there is a canonical choice.



\begin{definition}
A topology $\tau$ on a set $X$ is called an \textbf{Alexandrov topology} if the intersection of any (possibly infinite) collection of open sets is open. We call $(X, \tau)$ an \textbf{Alexandrov topological space}. The (full) subcategory of $\mathsf{Top}$ consisting of Alexandrov topological spaces is denoted $\mathsf{AlexTop}$. Crucially, any point $x$ in an Alexandrov topological space has a unique inclusion-minimal open neighborhood $U_{x}$ containing it, namely, $U_{x}:=\bigcap_{U \ni x \operatorname{ open}} U$. These open sets $U_{x}$, called \textbf{stars}, form a basis generating $\tau$.
\end{definition}

\begin{figure}
    \centering
    \includegraphics[width=0.5\linewidth]{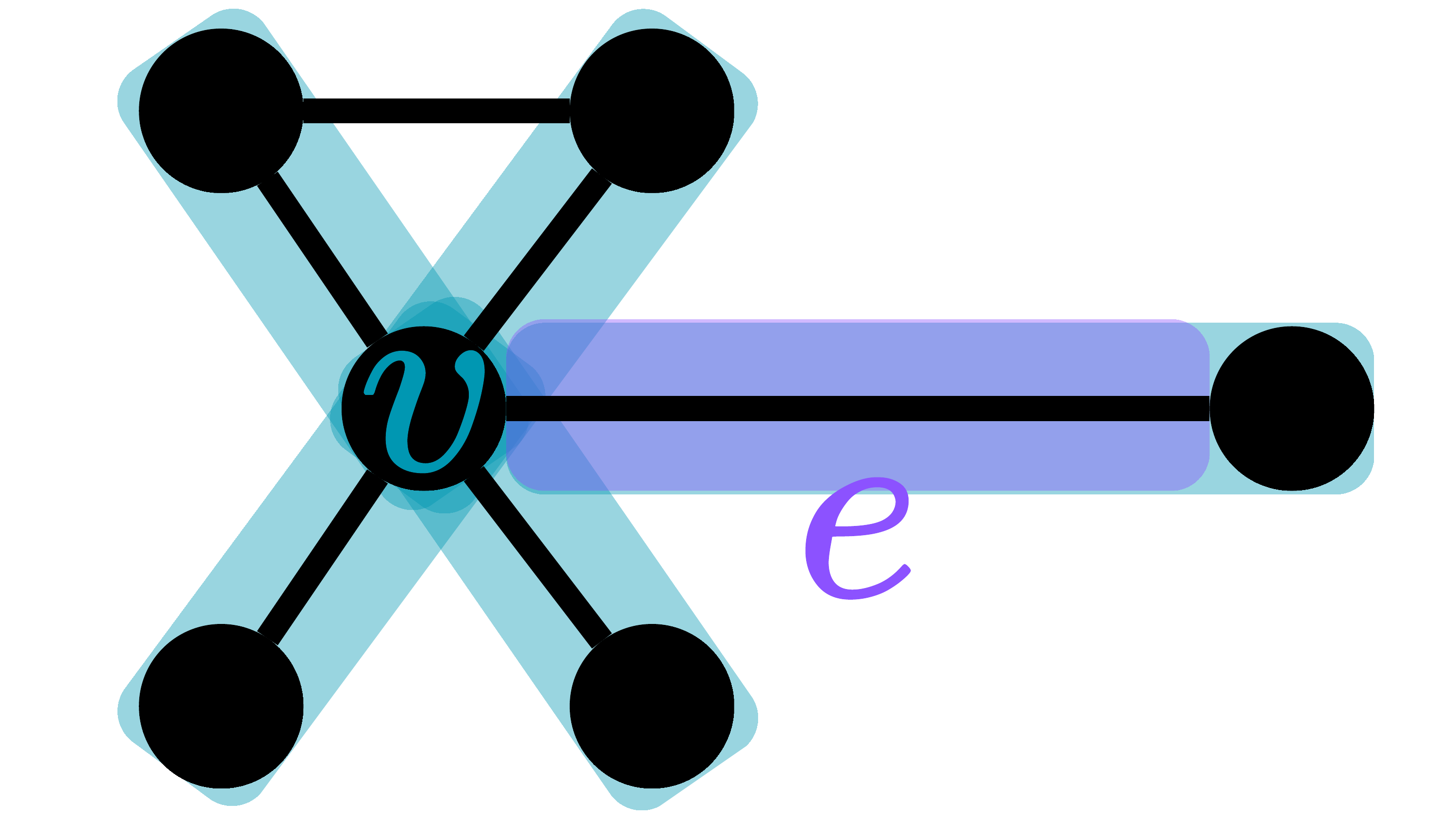}
    \caption{The stars $U_v$, $U_e$ of a node $v$ and edge $e$ in a graph $G$, where $G$ is viewed as a poset per Example~\ref{ex:diagram-on-graph}. Since $v \leq e$, $U_v \supset U_e$.}
    \label{fig:stars-figure}
\end{figure}

Note that if $X$ is Alexandrov, then for all $x, y \in X$,
\[
x \leq y \iff U_{x} \supset U_{y},
\]
where $U_{x}$ and $U_{y}$ denote the minimal open neighborhoods of $x$ and $y$ respectively.\footnote{Indeed, this is immediate from Remark~\ref{rmk:unwinding-specializatino-preorder}. If $x \leq y$, then $y$ belongs to every open set containing $x$, hence $y \in U_{x}$. Since $U_{y}$ equals the intersection of all open neighborhoods of $y$, and $U_{x}$ is one such, $U_{y} \subset U_{x}$. Conversely, if $U_{x} \supset U_{y}$, and $U$ is an open set containing $x$, then $U \supset U_{x} \supset U_{y} \ni y$. } This duality will explain why restriction maps `ascend' for sheaves on posets (Definition~\ref{def:sheaf-on-poset}) and `descend' for sheaves on topological spaces (Definition~\ref{def:sheaf-on-top-space}).



We have seen thus a way to go from a topological space $X$ to a (pre)poset $W(X)$. Of particular interest is the case when $X$ is Alexandrov, for then the topology $\tau_X$ remembers the specialization preorder on $X$. We can also go in the other direction: building an Alexandrov topology out of an order relation.

\begin{theorem}
    Given a preordered set $(X, \leq)$, there is a unique Alexandrov topology $\tau_{X}$ on $X$ whose specialization preorder is $\leq$, namely, the collection of upper sets in $(X, \leq)$. $\tau_{X}$ is called \textbf{the Alexandrov (or specialization) topology on $X$}. Any monotonic map $f:X \to Y$ is continuous with respect to the Alexandrov topologies $\tau_{X}$ and $\tau_{Y}$, giving thus a covariant functor $\tau:\mathsf{PrePos} \to \mathsf{AlexTop}$.
\end{theorem}
\begin{proof}
    Clearly the collection of upper sets in $X$ forms an Alexandrov topology on $X$.  Now, Suppose $\tau_{X}$ is an Alexandrov topology on $X$ whose specialization preorder is $\leq$. If $U \in \tau_{X}$ and $x \in U$ with $x \leq y$, then $x \in U_{y} \subset U_{x} \subset U$. In particular, $y \in U$. Hence $U$ is an upper set. Therefore, $\tau_{X}$ is necessarily the topology on $X$ consisting of upper sets. Now assume $f:X \to Y$ is a monotonic map between preordered sets $X,Y$ endowed with their canonical Alexandrov topologies.  Let $V \in \tau_{Y}$, and assume $x \in f ^{-1}(V)$. Then having $x \leq y$ for some $y \in X$ implies $f(x) \leq f(y)$, and since $V$ is open in $Y$, it follows that $f(y) \in V$. $y \in f ^{-1}(V)$, witnessing $f ^{-1}(V)$ to be open.
\end{proof}

The present digression culminates in the following result:

\begin{theorem}

The category $\mathsf{PrePos}$ of preordered sets and monotone maps is equivalent to the category $\mathsf{AlexTop}$ of Alexandrov topological spaces and continuous maps. Under this equivalence, posets correspond to Alexandrov topologies satisfying the $T_{0}$ separation axiom. Explicitly, the Alexandrov functor $\tau:\mathsf{PrePos} \to \mathsf{AlexTop}$ and specialization functor ${W}:\mathsf{AlexTop} \to \mathsf{PrePos}$ are quasi-inverses.
\end{theorem}

\subsubsection{Sheaves as Diagrams}

With the data of Section~\ref{sec:digression-topology-to-order-theory} in hand, we can make sense of the statement of Theorem~\ref{thm:sheaves-diagrams-category-equivalence}. Proving it requires the construction of two quasi-inverse functors passing between diagrams and sheaves; this is the content of the following two constructions.

\begin{definition}
Let $X_{S}$ be an Alexandrov topological space induced by a preposet $(S, \leq)$. Let $\mathcal{F}$ be a $\mathsf{D}$-valued presheaf on $X_{S}$. We define a diagram $F:S \to \mathsf{D}$ on $S$ via 
\begin{align}
F(x):=\mathcal{F}(U_{x}), \ \ F(x \leq y):=\mathcal{F}(U_{x} \supset U_{y})
\end{align}
where $U_{x}$ denotes the minimal open neighborhood of $x \in X_{S}$. In light of the fact that $\mathcal{F}_{x}= \lim\limits_{\underset{x \in U}{\xrightarrow{}}} \mathcal{F}(U)=\mathcal{F}(U_{x})$, $F$ is called the \textbf{diagram of stalks} of the presheaf $\mathcal{F}$ on $X_{S}$.
\end{definition}

\begin{definition}
    Let $F$ be a diagram on a poset $S$. $F$ induces a sheaf $\mathcal{F}$ on the Alexandrov topological space $X_{S}$ corresponding to $S$ by taking 
\begin{align}
\mathcal{F}(U):=\varprojlim\limits_{s \in U} \ {F(s)}
\end{align}
and letting the restriction maps $\mathcal{F}_{U \supset V}$ be naturally determined by the universal property of the limit. That is, $\mathcal{F}_{U \supset V}$ is the unique map making the following diagram commute for all incidences $s \leq s'$ in $S$: 

\begin{center}
\begin{tikzcd}
                                 & \mathcal{F}(U) \arrow[ldd, bend right] \arrow[rdd, bend left] \arrow[d, "\exists ! \mathcal{F}_{U \supset V}" description, dashed] &       \\
                                 & \mathcal{F}(V) \arrow[ld] \arrow[rd]                                                                                               &       \\
F(s) \arrow[rr, "F_{s\leq s'}"'] &                                                                                                                                    & F(s')
\end{tikzcd}
\end{center}

Explicitly, in our categories of interest (like $\mathbb{R}\mathsf{Vect}$), $\mathcal{F}(U)$ is given by 
\begin{align}
\mathcal{F}(U)=\left\{  (x_{s})_{s \in U} \in  \prod_{s  \in U} F(s) : F_{s \leq s'} x_{s}= x_{s'} \text{ for all } s \leq s' \right\}
\end{align}
and
\begin{align}
\mathcal{F}_{U \supset V}\big( (x_{s})_{s \in U} \big)= (x_{s})  _{s \in V}.
\end{align}

That locality and gluing are satisfied is immediate. $\mathcal{F}$ is called the \textbf{sheaf of sections} of the diagram $F$.
\end{definition}

\begin{proof}[Proof sketch\footnote{Routine category-theoretic checks are omitted in light of space and clarity considerations.} of Theorem~\ref{thm:sheaves-diagrams-category-equivalence}.]
    Let $U \subset X_{S}$ be any open set, and look at the sections $\mathcal{F}(U)=\Gamma(U, X_{S})$ of the sheaf $\mathcal{F}$ over $U$. The collection $\{ U_{x}: x \in U \}$ is a covering of $U$ by open subsets (as stars form a basis). By the sheaf axioms, then, a section $s \in \mathcal{F}(U)$ is determined by its restrictions $s_{x}=s |_{U_{x}}$, $x \in U$. Given $x,y \in U$ with $U_{y} \supset U_{x}$, compositionality ($\mathcal{F}_{U_{y} \supset U_{x}} \circ \mathcal{F}_{X \supset U_{y}}=\mathcal{F}_{X \supset U_{x}}$) enforces
\begin{align}
\mathcal{F}_{U_{y} \supset U_{x}}s_{y}=\mathcal{F}_{U_{y} \supset U_{x}} (\mathcal{F}_{X \supset U_{y}} (s))=\mathcal{F}_{X \supset U_{x}}s=s_{x}.
\end{align}
This condition is equivalent to the assertion
\begin{align}
F_{x \leq y}s_{x}=s_{y}.
\end{align}
Hence the section $s$ over $U$ determines a ``local section of the diagram''.

Conversely, suppose we have a diagram $F$ on $S$ and are given elements $s_{x} \in F(x)$, $x$ in some upper set $U$ of $S$, satisfying
\begin{align}
x \leq y \implies F_{x \leq y} s_{x}= s_{y}.
\end{align}
(A ``local section of the diagram''.) Want to show that these uniquely determine a section $s \in \mathcal{F}(U)$ of the sheaf $\mathcal{F}$ on $X_{S}$.

The $s_{x} \in F(x)$ correspond to elements $s_{x} \in \mathcal{F}_x=\mathcal{F}(U_{x})$ satisfying
\begin{align}
U_{y} \supset U_{x} \implies \mathcal{F}_{U_{y} \supset U_{x}}(s_{y})=s_{x}.
\end{align}
Now, a given overlap $U_{i} \cap U_{j}$ equals a union of stars $U_{z}$, each contained in both $U_{i}$ and $U_{j}$. It follows that, for all such $z$, $\mathcal{F}_{U_{i} \supset U_{z}}s_{i}=s_{z}=\mathcal{F}_{U_{j} \supset U_{z}}s_{j}$. Since $s_{i}$ and $s_{j}$ agree when restricted to a cover of the overlap $U_{i} \cap U_{j}$ by open subsets, the sections $s_{i} |_{U_{i} \cap U_{j}}$ and $s_{j} |_{U_{i} \cap U_{j}}$ are equal, by the locality axiom. It follows from the gluing axiom that there is a unique section $s \in \mathcal{F}(U)$ satisfying $s |_{U_{x}}=s_{x}$ for all $x$. 

\end{proof}

\begin{figure}
    \centering
    \includegraphics[width=\linewidth]{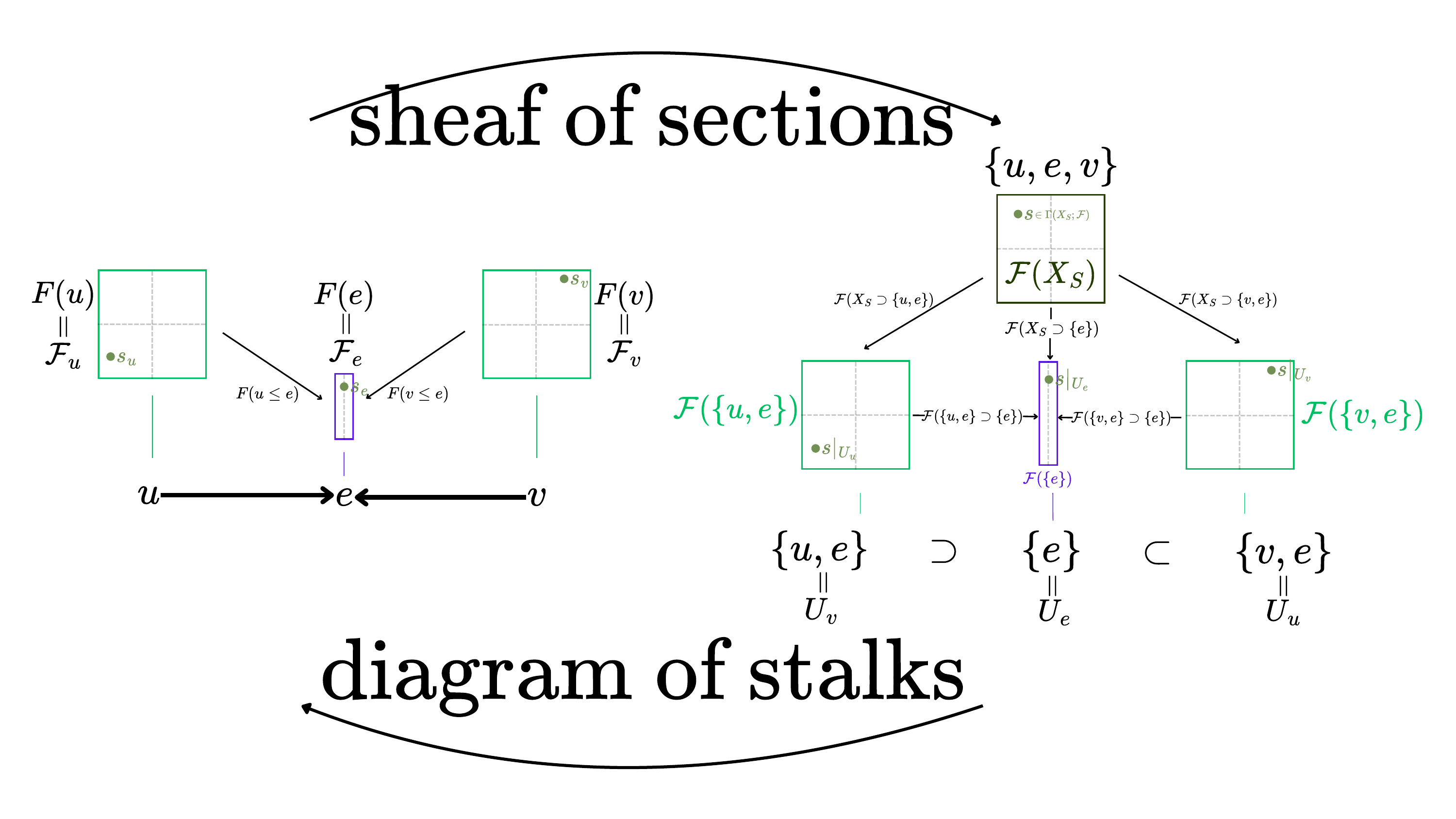}
    \caption{The correspondence between diagrams on a poset $S$ and sheaves on its Alexandrov space $X_S$.}
    \label{fig:enter-label}
\end{figure}

In deference to Theorem~\ref{thm:sheaves-diagrams-category-equivalence}, we will hereon (unambiguously) use terms such as `sheaf', `section', `stalk', `restriction map', etc. in both topological and posetal contexts.

We conclude this section by pointing out explicitly how sheafification factors through the equivalence in Theorem~\ref{thm:sheaves-diagrams-category-equivalence}. Since the sheafification $\mathcal{F}^+$ of a presheaf $\mathcal{F}$ on $X_S$ satisfies $\mathcal{F}^+_x=\mathcal{F}_x$ for all $x \in X$, the diagram of stalks $F:S \to \mathsf{D}$ corresponding to $\mathcal{F}^+$ is given by \begin{equation}
\label{eqn:diagram-sheafification-stalks}
    F(x)=\mathcal{F}^+_x=\mathcal{F}_x.
\end{equation}
A consequence of this is that, when transferring constructions from $\mathsf{Shv}_{\mathsf{D}}(X_S)$ to $\mathsf{Diag}_{\mathsf{D}}(S)$, one may merely `forget about any sheafification involved', as the following examples illustrate.

\begin{example}
    The constant sheaf $\underline{D}$ (Example~\ref{ex:constant-sheaf}) on a poset $S$ is merely the assignment $s \mapsto D$, $s \in S$, with restriction maps all the identity. 
\end{example}

\begin{theorem}
    The image of a morphism $f:F \to G$, $f=\big(f_s: F(s)\to G(s)\big )_{s \in S}$, of sheaves on a poset $S$ is merely the diagram on $G$ given by $s \mapsto \operatorname{im}f_s$, $s \in S$. 
\end{theorem}

\begin{theorem}
    The direct sum $\bigoplus_{i \in I} F_i$ of sheaves on a poset $S$ is the diagram $s \mapsto \bigoplus_{i \in I} F_i(s)$, $s \in S$, with componentwise restriction maps.
\end{theorem}

More generally, Equation~\ref{eqn:diagram-sheafification-stalks} has significance for any sheaf-theoretic notion that can be `tested on stalks' (cf. Theorem~\ref{thm:testing-on-stalks}). For instance:

\begin{theorem}
    A morphism $f:F\to G$ of sheaves on a poset $S$ is injective (resp. surjective) if and only if its components $f_s:F(s) \to G(s)$ each are.  
\end{theorem}
A similar claim holds for exactness.



\section{Motivation $\mathrm{II}$: Sheaves and Topological Deep Learning (TDL)}
\label{sec:motvation-higher-order}
\subsection{Combinatorial Hodge Theory and Heat Diffusion}
\label{sec:combinatorial-hodge-heat-diff}

Laplacian operators are ubiquitous, with avatars in Riemannian geometry, PDE, and stochastic processes manifesting myriad across discrete domains~\cite{wardetzky2007discrete}. A Laplacian's utility stems in part from the diffusion-mediated interface it establishes between local and global descriptors.  The Laplacian $\boldsymbol{L}$ on a graph $G$ (Definition~\ref{thm:GCNs-and-heat-diffusion}) offers a familiar example: 

\begin{example}
\label{ex:local-global-graph-lap-example}
\par\noindent
\begin{itemize}
    \item ($\text{Local}\to \text{Global}$) $\boldsymbol{L}$ measures ($\operatorname{grad}$), then aggregates ($\operatorname{div}$), local disagreements between node features. Iterating this process repeatedly and for different choices of features is known to eventually compute the kernel of $\boldsymbol{L}$ (Theorem~\ref{thm:ODE-fact}),  which is readily to seen to specify the connected components of $G$. 
    \item ($\text{Local}\leftarrow \text{Global}$) Conversely, we may know coarse global information about $G$, namely the number of components, and want to know how a certain node $v$ interacts with this topology, namely to which component $v$ belongs. Since $\ker \boldsymbol{L}$ is spanned by component indicator vectors, we can find out by projecting $(0, ..., \overbrace{x_v}^{\text{entry }v}, ... 0)$ onto $\ker \boldsymbol{L}$\footnote{Usually $x_v=1$ in this case, but anything nonzero would do the trick.}); equivalently, by performing diffusion with  $(0, ..., \underbrace{x_v}_{\text{entry }v}, ... 0)$ as initial condition.
\end{itemize}

\end{example}

For general Laplacians, this $\text{Local}\xleftrightarrow[]{\text{diffusion}}\text{Global}$ 
 principle follows from the \textit{Hodge Theorem}. 

\begin{definition}
 Let $(C^{\bullet}, d^{\bullet}):C^{0} \xrightarrow{d} C^{1} \xrightarrow{d} \cdots$ be a cochain complex of real Hilbert spaces.  The \textbf{Hodge Dirac operator} on $(C^{\bullet}, d^{\bullet})$ is the graded linear operator $D:=d+d^{*}$, where $d^*$ denotes the adjoint of $d$. The \textbf{Hodge Laplacian} on $(C^{\bullet}, d^{\bullet})$ is the graded linear operator \begin{equation}
     \Delta := D^2 = d^* d+ d d^*.
 \end{equation}
The operator $\Delta$ is naturally graded into components \begin{equation}
    \Delta^k:=(d^k+d^{^*k^{}})^2= \underbrace{d^{{^*} k} d^k}_{=:\Delta^k_{\text{up}}}+ \underbrace{d^{k-1} d^{^{*}k-1 }}_{=:\Delta^k_{\text{down}}},
\end{equation} 
where $\Delta^k$ is known as the \textbf{$k$th Hodge Laplacian} or the \textbf{Hodge $k$-Laplacian of $(C^{\bullet}, d^{\bullet})$}, $\Delta^k_{\text{up}}$ as its \textbf{up-Laplacian}, and $\Delta^k_{\text{down}}$ its \textbf{down-Laplacian}. The space $\ker \Delta^k$ is called the \textbf{harmonic space} of $\Delta^k$.
\end{definition}

\begin{theorem}{\cite{eckmann1944harmonische, hansen2020laplacians}}
\label{thm:discrete-Hodge-Theorem}
    Let $(C^{\bullet}, d^{\bullet})$ be a cochain complex of finite-dimensional inner product spaces, with corresponding Laplacians $\Delta^{k}:C^{k} \to C^{k}$.  Then the space $C^{k}$ has an orthogonal decomposition 
\begin{equation}
    C^{k}= \operatorname{ker }\Delta^{k} \oplus \operatorname{im }d^{k-1} \oplus \operatorname{im}(d^{k})^{*}
\end{equation}
with each summand invariant under $\Delta^k$,
and this follows from the fact that any cohomology class $a \in H^k(C^\bullet)$ has a canonical harmonic representative $\alpha \in \ker \Delta^k$, $a = [\alpha]$, yielding an isomorphism \begin{equation}
    \label{eqn:Hodge-thm-iso}
    \operatorname{ker }\Delta^{k}\cong H^{k}(C^\bullet).
\end{equation}
\end{theorem}
It then follows from positivity of $\Delta$ and standard ODE theory (namely, Theorem~\ref{thm:ODE-fact}) that solutions to the \textbf{$k$th-order heat diffusion equation} \begin{equation}
    {\dot{X}}(t) = {\Delta}^k {X}(t) 
\end{equation}
converge exponentially to the orthogonal projection of the initial condition ${X}(0)$ onto $\ker {\Delta}^k \cong H^k(C^\bullet)$, establishing a dictionary 
 \begin{equation}
\label{eqn:diff-dictionary}
   \text{harmonic space of }k\text{th} \text{ Laplacian}  \leftrightarrow k\text{th-order heat diffusion}  \leftrightarrow k\text{th cohomology},
\end{equation}
valid for any cochain complex of (in our case) finite-dimensional real inner product spaces. \\

 \begin{proof} The finite-dimensionality assumption reduces the proof of Theorem~\ref{thm:discrete-Hodge-Theorem} to linear algebra.\footnote{Some infinite-dimensional analogues hold, and are in fact central to Riemannian and complex geometry; these are trickier and require hard inputs from analysis to prove. } 
 \noindent\par
     \textit{Step 1. (Cohomology classes $\cong$ closed + coclosed forms)} Noting that $V / W \twoheadleftarrow V \hookleftarrow W^{\perp}$ is an isomorphism for any inclusion $W \subset V$ of finite-dimensional inner product spaces, it is immediate that there is an isomorphism \begin{align}
H^{k}(C^{\bullet})= \frac{\operatorname{ker }d^{}}{\operatorname{im }d^{}} & \xleftarrow{\cong} (\operatorname{im }d)^{\perp} \cap \operatorname{ker }d. \\
[\alpha] & \mapsfrom \alpha
\end{align}
By general linear algebra, $(\operatorname{im }d)^{\perp}=\operatorname{ker }d^{*}$, $\operatorname{ker }d d^{*}=\operatorname{ker }d^{*}$, and $\operatorname{ker }d=\operatorname{ker }d^{*}d$. It follows that 
$$H^{k}(C^{\bullet}) \xleftarrow{\cong} _{} \operatorname{ker }d^{*} \cap \operatorname{ker } d =\operatorname{ker } \Delta^{\text{down}} \cap \operatorname{ker }\Delta^{\text{up}}.$$

\textit{Step 2. (Closed + coclosed forms $=$ harmonic space)} At the same time, if $\alpha \in \operatorname{ker }\Delta$ then $\langle \Delta \alpha, \alpha \rangle=\|d \alpha\|^{2} + \|d^{*} \alpha\|^{2}=0$. It follows that $$\operatorname{ker }\Delta=\underbrace{ \operatorname{ker } d^{*} }_{ \text{coclosed} } \cap \underbrace{ \operatorname{ker } d }_{ \text{closed} }= \underbrace{ \operatorname{ker } \Delta^{\text{down}} }_{ \text{coclosed} } \cap \underbrace{ \operatorname{ker }\Delta^{\text{up}} }_{ \text{closed} }.$$
Thus, there is an isomorphism\begin{align}
\operatorname{ker } \Delta &\xrightarrow{\cong} H^{k}(C^{\bullet}) \\
\alpha & \mapsto [\alpha] ,
\end{align}
as claimed.

\textit{Step 3. (The decomposition)} We claim the orthogonal direct sum
\[
C^k \;=\; \operatorname{im} d^{k-1} \;\oplus\; \ker \Delta^k \;\oplus\; \operatorname{im} d^{*k}.
\]

First, $\operatorname{im} d^{k-1} \perp \operatorname{im} d^{*k}$ because for any $\alpha\in C^{k-1}$ and $\beta\in C^{k+1}$,
\[
\langle d^{k-1}\alpha,\, d^{*k}\beta\rangle \;=\; \langle d^k d^{k-1}\alpha,\, \beta\rangle \;=\; \langle 0,\, \beta\rangle \;=\; 0
\]
using $d \circ d=0$. Next, if $h\in\ker\Delta^k$ then $dh=d^*h=0$ (Step 2). Hence
\[
\langle h,\, d^{k-1}\alpha\rangle \;=\; \langle d^{*}h,\, \alpha\rangle \;=\; 0,
\qquad
\langle h,\, d^{*k}\beta\rangle \;=\; \langle dh,\, \beta\rangle \;=\; 0,
\]
so $\ker\Delta^k$ is orthogonal to both $\operatorname{im} d^{k-1}$ and $\operatorname{im} d^{*k}$. Thus the three summands are pairwise orthogonal.

It remains to show the sum equals $C^k$. Compute the orthogonal complement:
\[
\bigl(\operatorname{im} d^{k-1} \oplus \operatorname{im} d^{*k}\bigr)^{\perp}
\;=\;
\bigl(\operatorname{im} d^{k-1}\bigr)^{\perp}\cap\bigl(\operatorname{im} d^{*k}\bigr)^{\perp}
\;=\;
\ker (d^{k-1})^{*}\cap \ker d^{k}
\;=\;
\ker d^{*k}\cap \ker d^{k}
\;=\;
\ker \Delta^{k},
\]
where we used $(\operatorname{im}T)^\perp=\ker T^*$ and Step 2. Therefore
\[
C^k \;=\; \bigl(\operatorname{im} d^{k-1} \oplus \operatorname{im} d^{*k}\bigr) \;\oplus\; \bigl(\operatorname{im} d^{k-1} \oplus \operatorname{im} d^{*k}\bigr)^{\perp}
\;=\;
\operatorname{im} d^{k-1} \oplus \ker \Delta^{k} \oplus \operatorname{im} d^{*k},
\]
as claimed.

\textit{Step 4. (Invariance)}That each summand is invariant under $\Delta^k$ follows immediately from the condition $d^2=0$. Indeed,\begin{align}
  \Delta^k(\ker \Delta^k)&=\{0\} \subset \ker \Delta^k  \\
\Delta^{k}(d^{k-1} \alpha) &= d^{k-1} (d^{^{*}k-1} d^{k-1}) \alpha  \subset \operatorname{im } d^{k-1} \\
\Delta^{k}(d^{^{*}k \alpha})&= d^{^{*} k}(d^{k} d^{^{*}k}) \subset \operatorname{im }d^{^{*}k}.
\end{align}
 \end{proof}

The space $\operatorname{im} d$ is called the \textbf{gradient space}, and $\operatorname{im}d^*$ the \textbf{curl space}. Alongside the harmonic space $\ker \Delta^k$, these suggestive labels are most readily understood in the special case where $(C^\bullet, d^\bullet)$ is the simplicial cochain complex for a two-dimensional simplicial complex $S=S^2$ (imagining a triangulated surface will suffice). To this end, we first recall the following definition. 

\begin{definition}[Simplicial cochains (discrete forms)] Let $S$ be a simplicial complex. Designate an orientation for each constituent simplex $\sigma \in S$ by choosing an order for its vertices: $\sigma=[v_0, \dots, v_{\dim \sigma }]$.

When the coefficient ring of a simplicial cochain complex $(C^\bullet, d^\bullet)$ is $\mathbb{R}$, simplicial $n$-cochains are called \textbf{simplicial (discrete) differential $n$-forms}. In this case, the coboundary operator

    \begin{align} d^n:C^{n}(S) &\to C^{n+1}(S) 
 \end{align} is called the \textbf{$n$th simplicial (discrete) exterior derivative}. Explicitly, if $\{ \varphi_{\sigma}: \sigma \in C_{n}(S) \}$ is the dual basis for the vector space $C^{n}(S)=\operatorname{Hom}\big( C_{n}(S), \mathbb{R} \big)$, then $d$ is determined by incidence numbers

 \begin{align}
 [\sigma: \tau]:=(d\varphi_{\sigma})(\tau \in S^{n+1})&= \varphi_{\sigma}(d\tau)= \sum_{j=0}^{n+1} (-1)^{j} \varphi_{\sigma}(\tau |_{[v_{0},\dots,\widehat{v_{j}},\dots, v_{n}]})
 \end{align}
which manifestly equal \begin{equation}
[\sigma : \tau]=\begin{cases}
  1  & \sigma \leq \tau \text{ with matching orientation}  \\
 0 & \sigma \text{ and } \tau \text{ have no relation} \\
 -1 & \sigma \leq \tau \text{ with reversed orientation}.
 \end{cases} \end{equation}. 

The matrix $\boldsymbol{B}_n$ of the differential $d_n:C_n(S) \to C_n(S)$ with respect to the simplex bases is called the \textbf{$n$th boundary matrix} or \textbf{$n$th incidence matrix} of $S$; its transpose $\boldsymbol{B}_n^\top$ the $n$\textbf{th coboundary matrix} or \textbf{$n$th coincidence matrix}.  Note that $\boldsymbol{B}_n^\top$ is the matrix of $d$, not of $d^*$.  
 
\end{definition}

Put $V:=S^{0}$, $E:=S^{1}$, $T:=S^{2}$. 

\begin{definition}
If $f =\sum_{v \in V}x_{v} v$ (coordinates: $\boldsymbol x=(x_{v})_{v \in V}\in \mathbb{R}^{|V|}$) is a 0-form on $S$, define its \textbf{gradient} to be the 1-form given by $$\operatorname{grad }f:= df;$$in coordinates: $$\operatorname{grad } \boldsymbol  x=\boldsymbol  B_{1}^{\top} \boldsymbol  x= \left( \sum_{v \in V}[v : e] x_{v} \right)_{e \in E}=\left( x_{\text{tail}}- x_{\text{tip}} \right)$$
We see that $\operatorname{grad }\boldsymbol x$ is obtained by recording, for each edge $e=[ v_{0},v_{1} ]$ in $S$, the `potential difference' $x_{v_{1}}-x_{v_{0}}$ of the node signal $\boldsymbol x$ across $e$. $\operatorname{grad } \boldsymbol x=0$ means the node signal $\boldsymbol x$ is constant on every connected component.
\end{definition}

\begin{definition}
If $\theta =\sum_{e \in E} x_{e} e$ (coordinates: $\boldsymbol x=(x_{e})_{e \in E}$) is a 1-form on $S$, define its \textbf{curl} to be the $2$-form $$\operatorname{curl }\theta:= d\theta;$$
in coordinates: $$\text{curl } \boldsymbol  x=\boldsymbol  B_{2}^{\top}\boldsymbol  x=\left( \sum_{e \in E} [e: t]x_{e}  \right )_{t \in T}=(x_{[v_{0},v_{1}]}+x_{[v_{1},v_{2}]}-x_{[v_{0},v_{2}]})_{[v_{0},v_{1},v_{2}] \in T}$$We see that $\text{curl }\boldsymbol x$ is obtained by recording, for each triangle $t=[v_{0},v_{1},v_{2}]$ of $S$, the `circulation' $x_{[v_{0},v_{1}]}+x_{[v_{1},v_{2}]}-x_{[v_{0},v_{2}]}$ of $\boldsymbol x$ around $t$.  $\operatorname{curl } \boldsymbol x=0$ means the edge signal $\boldsymbol x$ is \textit{irrotational}.

\end{definition}

\begin{definition}
We can also define $$\operatorname{div }\theta:=d^{*} \theta;$$
in coordinates: $$\operatorname{div }\theta=\boldsymbol  B_{1}\boldsymbol  x=\left( \sum_{e \in E} [v: e] x_{e} \right)_{v \in V}=\left( \sum_{ e \text{ exiting }v} x_{e} - \sum_{e \text{ entering }v} x_{e} \right)_{v \in V}.$$
We see that $\operatorname{div }\boldsymbol  x$ is obtained by recording, for each node $v$, the 'net outward flux' through $v$. $\operatorname{div } \boldsymbol x=0$ means there are no sources or sinks. 
\end{definition}
Note that $\operatorname{curl} \operatorname{grad}=0$ because $d \circ d=0$, recovering a discrete analogue of an intuitive result in vector calculus.

\begin{definition}
    It also is useful to define the \textbf{cocurl} $d^{*}=\operatorname{curl}^{*}:C^{2} \to C^{1}$, given by $d^{*}(\omega: C_{2} \to \mathbb{R})(\theta:C_{1} \to \mathbb{R})=\omega(d \theta)$, which in coordinates is given by $$\operatorname{curl}^{*} \boldsymbol  x=\boldsymbol  B_{2}\boldsymbol  x= \left( \sum_{t \in T} [e : t]   x_{t} \right)_{e \in E}.$$
If $S$ is a triangulated surface, so that $e$ participates in at most two triangles, this simplifies to $\operatorname{curl}^{*}\boldsymbol x=(  x_{t_{L}(e)} -   x_{t_{R}(e)})_{e \in E}$, where $t_{L}(e)$ and $t_{R}(e)$ are the left and right triangles to which $e$ is interior (if such triangles exist). In this case we see that $\operatorname{curl}^{*} \boldsymbol  x$ is obtained by recording, for each edge $e$, the signed jump (difference) of the face signal $\boldsymbol x$ across $e$. 

\end{definition}
With this notation, the graph Laplacian is $$
\boldsymbol L= \Delta^{0}=\boldsymbol B_{1} \boldsymbol B_{1}^{\top}=\operatorname{div }\operatorname{grad},$$and Hodge $1$-Laplacian (sometimes called the \textbf{graph Helmholtzian}~\cite{lim2020hodge} in this setting) is $$  \Delta^{1}=    \underbrace{ \boldsymbol  B_{2}\boldsymbol   B_{2}^{\top} }_{ \text{up} } + \underbrace{ \boldsymbol  B_{1} ^{\top}\boldsymbol  B_{1} }_{ \text{down} }= \text{curl$^*$ }\text{curl} + \text{grad } \text{div} .$$

In Section~\ref{sec:rapid-preliminaries}, we saw that diffusion with the Laplacian on a graph $G$ served to minimize the Dirichlet energy functional on $G$ (Definition~\ref{def:graph-dirichlet-energy}). This holds in our present general setting, and in fact may be taken as an equivalent \textit{definition} of heat diffusion. To see this, first let $V$ be a finite dimensional inner product space over $\mathbb{R}$, and suppose $A \in \operatorname{End }V$ is any positive semidefinite operator. Since $A$ is positive semidefinite, $\operatorname{ker } A=\{ v \in V: \langle Av, v \rangle=0  \}.$ Letting $Q_{A}:V \to \mathbb{R}_{\geq 0}$ be the quadratic form $Q_{A}(x):=\langle Ax, x \rangle$, we see that
\begin{align}
x \in \operatorname{ker } A \iff Q_{A}(x)=0 \iff x \text{ minimizes }Q_{A}.
\end{align}
Thus, computing $\operatorname{ker }A$ amounts to minimizing the convex function $Q_{A}$. The gradient is $\nabla Q_{A}=2A$ (recall $A$ is self-adjoint). The continuous gradient descent flow of $Q_{A}$ with parameter $\eta>0$ takes the form
\begin{align}
\dot{x}=-2 \eta A x
\end{align}
By Theorem~\ref{thm:ODE-fact} and the surrounding discussion, the solution is $x(t)=e^{-2 \eta A t} x(0)$ for $x(0)$ an initial condition, and converges exponentially to the orthogonal projection of $x(0)$ onto $\operatorname{ker } A$. The discrete gradient descent of $Q_{A}$ is
\begin{align}
x_{k+1}=x_{k}-2 \eta A x _{k}.
\end{align}

The $k$th Laplacian $\Delta^k$ of a cochain complex $(C^\bullet, d^\bullet)$ 
is manifestly positive semidefinite, and so to it the above considerations apply. This yields the following definitions.

\begin{definition}
\label{def:dirichlet-energy-general-def}
  Let $(C^{\bullet}, d^{\bullet})$ be a cochain complex of Hilbert spaces, giving rise to Laplacians $\Delta^{k}:C^{k} \to C^{k}$. The quadratic form $Q_{k}=\langle \Delta^{k} -, - \rangle$ determined by $\Delta^{k}$,
\begin{align}
Q_{k}(x)=\langle \Delta^{k}x, x \rangle ,
\end{align}
will be called the \textbf{$k$th} \textbf{Dirichlet energy} associated to $\Delta^{k}$. We reserve the unadorned term \textbf{Dirichlet energy} for the vanilla case $k=0$.
\end{definition}
Recalling that $\ker \Delta = \ker d \cap \ker d^*$, $Q_k(x)$ quantifies how far a cochain is from being harmonic. Indeed, \begin{equation}
    Q_k(x)=\langle x, \Delta^{k} x \rangle=\langle x, d^{k-1} d^{^{*}k-1} \rangle+ \langle x, d^{^{*}k }d^{k} \rangle=\underbrace{ \|d^{^{*}k-1} x\|^{2} }_{ \text{distance from} \operatorname{ker }d^{*}}+ \underbrace{ \|d^{k} x\|^{2} }_{ \text{distance from}\operatorname{ker }d }
    \label{eqn:dirichlet-energy-as-distance}
\end{equation}

\begin{definition}
    The flow of gradient descent (continuous or discrete) of the $k$th Dirichlet energy $Q_{\Delta^k}$ is called the \textbf{$k$th order heat diffusion on the space $C^k$}. Usually $k=0$; for this scenario we reserve the unadorned term \textbf{heat diffusion}. Per the above discussion, we have that:

    \begin{enumerate}
        \item Irrespective of initial conditions, the $k$th-order heat diffusion flow converges to a harmonic $k$-cochain $x({\infty}) \in \operatorname{ker}\Delta^{k} \subset C^{k}$.
 \item The spectrum of $\Delta^{k}$ governs the convergence rate. 
    \end{enumerate}
\end{definition}

\begin{remark}
    We have chosen to formulate the present discussion in such a manner that the $k$th heat diffusion minimizes the $k$th Dirichlet energy \textit{essentially by construction}. Sometimes, e.g. in \cite{duta2023sheaf}, there is a sensible notion of Laplacian and Dirichlet energy that do not arise overtly from a cochain complex. In such cases, the existence of a relationship between the two notions is something to be \textit{proven} (this will be the case e.g. for Theorem~\ref{thm:duta-laplacian-kernel-and-global-sections}).
\end{remark}




\subsection{The Hodge Bias: What is Oversmoothing in TDL?}

Many higher-order message passing architectures arise as augmentations of discrete heat diffusion with the Hodge Laplacian $\Delta$ corresponding to a poset $S$. Of course, an immediate question should arise in light of the previous section: the Hodge Laplacian of \textit{what cochain complex}? This will be clear following the discussion of Roos cohomology in Section~\ref{sec:sheaf-cohomology-diffusion}; for our present motivational aims, it will be without much loss of generality to assume $S$ is a simplicial complex and $\Delta$ is the Hodge Laplacian on simplicial cochains. The heat diffusion underlying order-$k$ message passing is governed by the equation \begin{equation}
    \dot X(t) = \Delta^k X(t)
\end{equation}
with solutions (feature trajectories) converging in the limit to the orthogonal projection of the initial condition $X(0)$ onto the harmonic space $\ker \Delta^k$. When $k=0$, the diffusion procedure is characterized by the homogenization of node representations in each connected component of the $0$-skeleton $S^0$ and we recover the oversmoothing phenomenon from Section~\ref{sec:rapid-preliminaries}. Our goal is now to understand what happens when $k>0$, and relate this to a message passing inductive bias just like we related order-zero diffusion to the homophily bias in graphs. We will do so from a few different perspectives. 

\paragraph{The Signal Processing Perspective} 

In the parlance of topological signal processing~\cite{isufi2025topological}, the Hodge Laplacian plays the role of a \textit{shift operator} whose eigendecomposition $\Delta^{k}=\boldsymbol U \boldsymbol \Lambda \boldsymbol U^{\top}$ encodes the Fourier information of the underlying complex. Specifically, one views $\{ \{ \lambda_{i} \} \}=\text{diag}(\boldsymbol \Lambda)$ as the \textbf{frequency spectrum} corresponding to the underlying complex and $\boldsymbol U^{\top}$ as the \textbf{(simplicial) Fourier transform} which catalogs how much of each frequency $\lambda_{i}$ is present in $\boldsymbol x$ via $\hat{\boldsymbol x}=\boldsymbol U^{\top}\boldsymbol x$. The Dirichlet energy $Q_{k}(\boldsymbol u_{i})$ of each eigenvector (\textbf{Fourier mode}) $\boldsymbol u_{i}$ equals the frequency $\lambda_{i}$ it indexes, since $\langle \boldsymbol u, \Delta^{k}\boldsymbol u \rangle=\lambda_{i} \langle \boldsymbol u, \boldsymbol u \rangle=\lambda$.  The Dirichlet energy of a general $k$-cochain $\boldsymbol x$ is therefore $$Q_{k}(\boldsymbol  x)=Q_{k}(\boldsymbol  U \hat{\boldsymbol  x})=Q_{k}\left( \sum_{i}\hat{x}_{i} \boldsymbol  u_{i} \right)=\sum_{i} \lambda_{i}\hat{x}_{i}^{2} .$$
In other words, signals with higher frequency content are `less smooth', where by `smooth' we mean `small distance to being harmonic'. Given then a signal $\boldsymbol x=\boldsymbol x(0)$ on the $k$-simplices of $S$, the diffusion trajectory $\boldsymbol x(t)$ of $\boldsymbol x(0)$ looks like (Theorem~\ref{thm:ODE-fact})
$$\boldsymbol  x(t)=\boldsymbol  Ue ^{\boldsymbol  \Lambda t}\boldsymbol  U^{\top}\boldsymbol  x(0)=\boldsymbol  U\big( \text{diag}(e^{-\lambda_{i}t} \hat{x}_{i})\big )=\sum_{i} e ^{-\lambda_{i} t}\hat{x}_{i} \boldsymbol  u_{i},$$
so that every nonzero frequency component decays exponentially, and higher-frequency components decay faster. The Dirichlet energy descends as $$Q_{k}\big( \boldsymbol  x (t) \big)=\sum_{i} \lambda_{i} (e^{-\lambda _{i}t} \hat{x}_{i})^{2}=\sum_{i} \lambda_{i}e^{-2\lambda _{i} t} \hat{x}_{i};$$
as $t \to \infty$, $\boldsymbol x(t)$ converges to the orthogonal projection of $\boldsymbol x(0)$ onto the harmonic (DC) space $\operatorname{ker }\Delta^{k}$. The conclusion is that heat diffusion acts as a dynamic low-pass filter which increasingly dampens high frequencies until only the DC component of $\boldsymbol x(0)$ remains. The idea that `true' signals consist of low-frequency content corrupted by high-frequency noise is a powerful prior that reprises unremittingly throughout the signal processing canon. Message passing via augmented low-pass filtering with the Hodge Laplacian reflects such an inductive bias.

Importantly, there is no canonical notion of a shift operator on a graph (much less on a simplicial complex or more general poset), and there is therefore no canonical notion of frequency in our setting. To probe the high-frequency content which Hodge diffusion filters away (and the low-frequency content it prioritizes), we appeal to the Hodge Decomposition (Theorem~\ref{thm:discrete-Hodge-Theorem}). Every initial signal $\boldsymbol{x} \in C^k$ on the $k$-simplices of $S$ may be orthogonally decomposed as $\boldsymbol{x}=\boldsymbol{x}_h + \boldsymbol{x}_c + \boldsymbol{x}_g$ for $\boldsymbol{x}_h \in \ker \Delta^k$ the DC component (harmonic, zero frequency), $\boldsymbol{x}_c \in \operatorname{im}d^{^* k}$ the `curl component', and $\boldsymbol{x}_g \in \operatorname{im}d^{k-1}$ the `gradient component'. Hodge diffusion attenuates these latter components, reflecting the Hodge bias that they are noisy contributions to a `true' harmonic signal. 

Note that each simple non-harmonic Fourier mode $\boldsymbol{u}_i$ is contained in exactly one of the curl space or gradient space. Combining with Equation~\ref{eqn:dirichlet-energy-as-distance} and invoking the cochain condition $d^2=0$, we have

\begin{equation} \lambda_i = Q_k(\boldsymbol{u}_i)=
\tikzmarknode{term1}{\|d^{^* k-1} \boldsymbol{u}_i\|^2} + \tikzmarknode{term2}{\|d^k \boldsymbol{u}_i\|^2},
\label{eqn:eigenvalue-as-dirichlet-energy-of-eigenvector}
\end{equation}
\begin{tikzpicture}[overlay, remember picture]

    \coordinate (p1_top) at ([yshift=-0.5ex]term1.south);
    \coordinate (p2_top) at ([yshift=-0.5ex]term2.south);

    \coordinate (p1_corner) at ([yshift=-2ex]term1.south);
    \coordinate (p2_corner) at (p1_corner -| p2_top);

    \draw (p1_top) -- (p1_corner) -- (p2_corner) -- (p2_top);

    \node[anchor=north, font=\small] at ($(p1_corner)!0.5!(p2_corner)$) [below=1pt] {one of these is zero};
\end{tikzpicture}

a fact that becomes useful when attempting quantify what `high frequency' explicitly looks like for different eigenvectors (we will see an example in the next section). 







\paragraph{The Physical Perspective}
To understand more explicitly the nature of the dynamic low-pass filtering enacted by Hodge diffusion, we draw upon the intuitions and operations outlined in Section~\ref{sec:combinatorial-hodge-heat-diff}. Assume $k=1$. Then $d^0$ aligns with the gradient operation, meaning that for $\boldsymbol{x} \in C^0$ $(\operatorname{grad} \boldsymbol{x})_e=x_\text{tail}-{x}_{\text{tip}}$, recording the `potential difference' of the node signal (`scalar field') across each edge in the $1$-skeleton $S^1$. The gradient space $\operatorname{im}\operatorname{grad}$ is `spanned by stars' given by the images of the node indicator vectors $\boldsymbol{1}_v$ (Figure~\ref{fig:grad_frame}). The node signals with zero gradient are precisely those
which are locally constant. 

On the other hand, $d^{^*1}$ aligns with the cocurl operation $\operatorname{curl}^*$, meaning that (assuming for intuition that $S$ is a triangulated surface) for $\boldsymbol{x} \in C^2$, $\operatorname{curl^*}(\boldsymbol{x})_e=x_{t_L} - x_{t_R}$ where $t_L$ and $t_R$ are the left and right triangles to which $e$ is interior (if such triangles exist). The cocurl operation thus records the signed jump of a face signal across each edge. Still assuming $S$ is a triangulated surface, the curl space $\operatorname{im} \operatorname{curl^*}$ is spanned by boundary indicator signals around each triangle with signs $\pm 1$ according to orientation; these are the images of triangle indicator signals $\boldsymbol{1}_t$ under $d^{*}$ (Figure~\ref{fig:curl_frame}). 

The above considerations give explicit physical description to the sorts of `high frequency' information which Hodge diffusion attenuates. We can give quantitative meaning to the magnitude of a given frequency by examining the behavior of Equation~\ref{eqn:dirichlet-energy-as-distance} in this setting. Recall that a (simple) non-harmonic eigenvector (non-DC Fourier mode) $\boldsymbol{u}$ with eigenvalue (frequency) $\lambda>0$ is contained in exactly one of the curl or gradient spaces. In the former case $\boldsymbol{u}=\boldsymbol{u}_{c}$, Equation~\ref{eqn:eigenvalue-as-dirichlet-energy-of-eigenvector} gives \begin{equation}
    \lambda_i = \|d^{^* 0}\boldsymbol{u}_c\| = \|\operatorname{div}\boldsymbol{u}_c\| = \sum_{v \in V}    \text{Flux}(v)^2 ,
\end{equation}
where $\text{Flux}(v)=\sum_{ e \text{ exiting }v} x_{e} - \sum_{e \text{ entering }v} x_{e} $ for $v \in V$. Thus, we see quantitatively that the presence of larger sources and/or sinks (large magnitudes of flux across nodes) corresponds to higher frequency content in an edge flow. As it kills such gradient eigenvectors, Hodge diffusion homogenizes fluxes across nodes, decreasing the magnitudes of sources and sinks. In the latter case $\boldsymbol{u}=\boldsymbol{u}_g$, Equation~\ref{eqn:eigenvalue-as-dirichlet-energy-of-eigenvector} gives \begin{equation}
    \lambda_i = \|d^{1}\boldsymbol{u}_g\| = \|\operatorname{curl}\boldsymbol{u}_c\| = \sum_{t \in T}    \text{Circulation}(t)^2 ,
\end{equation}
where $\text{Circulation}(t)=x_{[v_{0},v_{1}]}+x_{[v_{1},v_{2}]}-x_{[v_{0},v_{2}]}$ for a triangle $t=[v_0, v_1, v_2] \in T$. Thus, we see quantitatively that the presence of larger amounts of net circulation around triangles in $S$ corresponds to higher frequency content in an edge flow. As it kills such curl eigenvectors, Hodge diffusion smoothens the flow toward parsimony; toward getting from one place to another sans circuitous detours.  


\begin{figure}[H]
  \centering
  \begin{subfigure}[t]{0.45\textwidth}
    \centering
    \includegraphics[width=\linewidth]{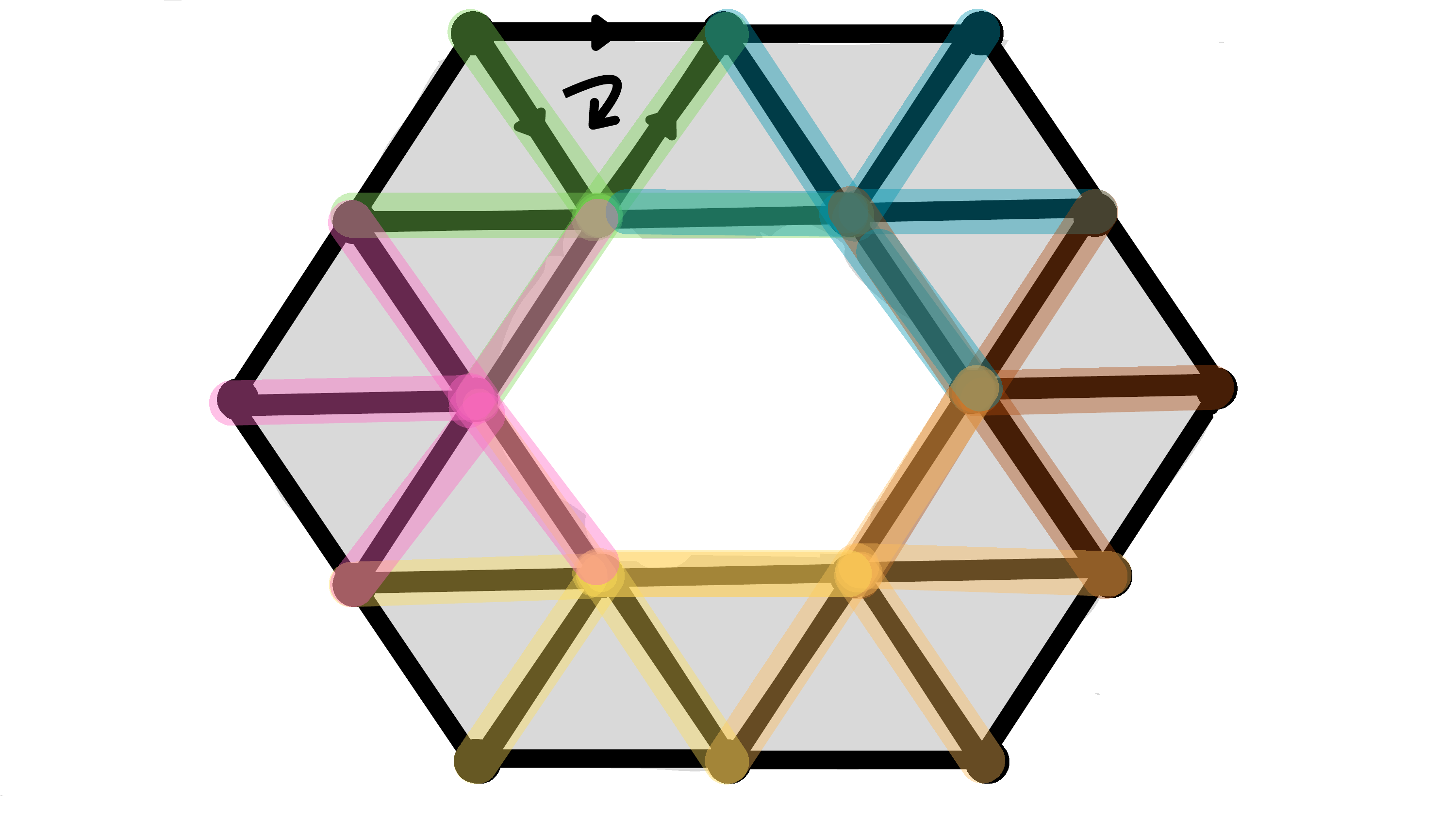}
    \caption{A subset of spanning elements of the gradient space $\operatorname{im}\operatorname{grad}$ on a simplicial complex. Each color corresponds to the image of a basis element $\boldsymbol{1}_v \in C^0(S)$ under $\operatorname{grad}=d^0$ up to signs determined by orientation. Every element of $\operatorname{im}\operatorname{grad}$ then arises as a linear combination of these `smeared node' $1$-cochains.}
    \label{fig:grad_frame}
  \end{subfigure}\hfill
  \begin{subfigure}[t]{0.45\textwidth}
    \centering
    \includegraphics[width=\linewidth]{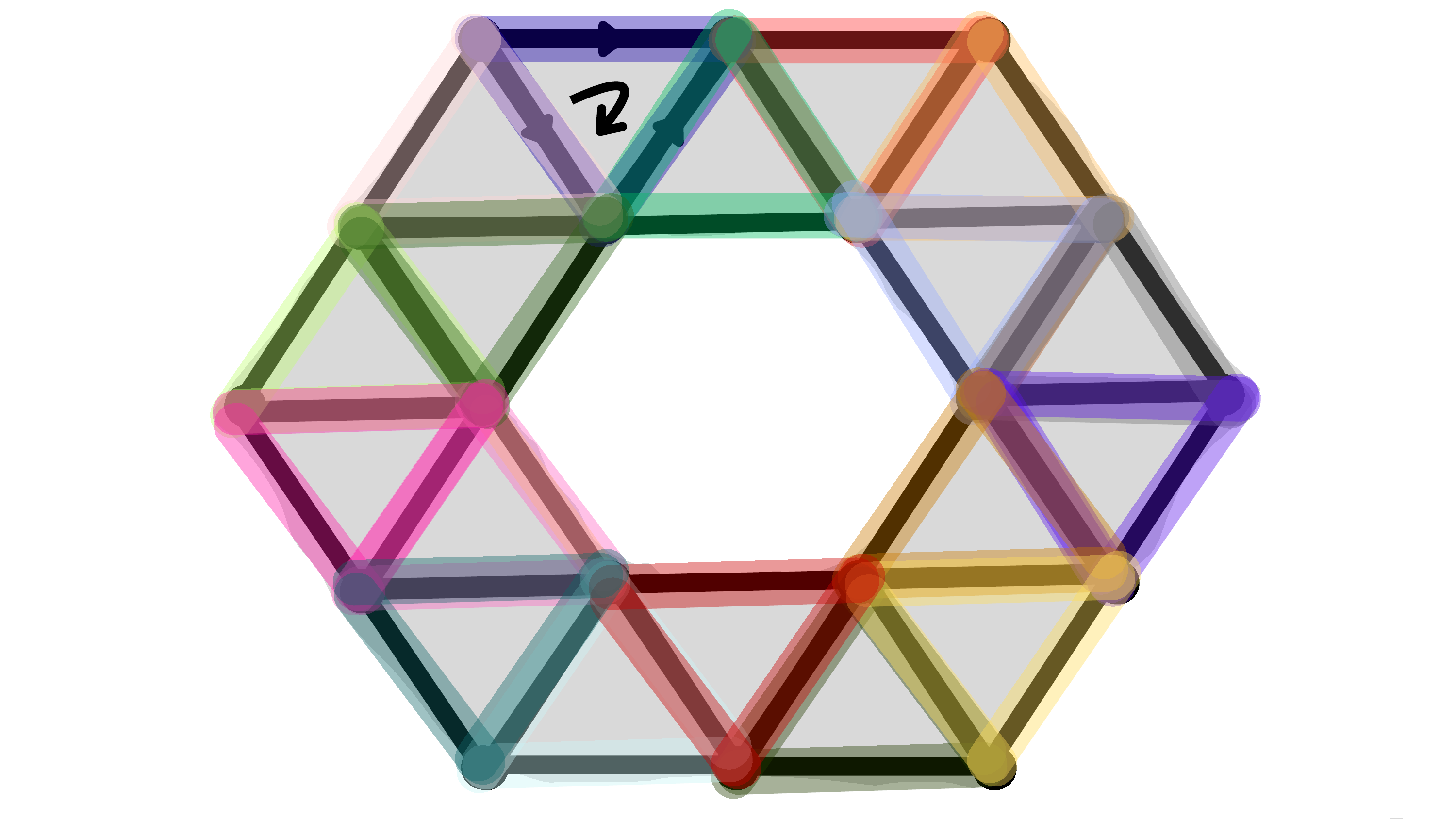}
    \caption{Spanning elements of the curl space $\operatorname{im}\operatorname{curl}^*$ on a simplicial complex. Each color corresponds to the image of a basis element $\boldsymbol{1}_t \in C^2(S)$ under $\operatorname{curl}^*=d^{^* 1}$ up to signs determined by orientation. Every element of $\operatorname{im}\operatorname{curl}^*$ then arises as a linear combination of these `curly' $1$-cochains.}
    \label{fig:curl_frame}
  \end{subfigure}
  \caption{Some spanning elements for the image spaces of the gradient and cocurl operators on a simplicial complex.}
  \label{fig:combined_frames}
\end{figure}

To understand the harmonic space, we recall from the proof of Theorem~\ref{thm:discrete-Hodge-Theorem} that it equals $\ker d^* \cap \ker d$, which in our present notation $d^1=\operatorname{curl}$, $d^{^* 1}=\operatorname{div}$ is \begin{equation}
\ker \Delta^k = \{ \text{edge flows with zero divergence and zero curl} \} =\{\text{irrotational + incompressible edge flows}\}.
\end{equation}
This is the destination to which Hodge $1$-diffusion converges: a smooth flow (in the sense of zero Dirichlet energy) achieved by diffusing away sources and sinks (by killing gradient eigenvectors) and diffusing away net circulations (by killing curl eigenvectors). This lends a physical interpretation to the Hodge bias, at least in certain cases. We observe that, when $k>0$, the `maximally smooth signals' to which Hodge diffusion converges \textit{no longer need to be locally constant} (see Figure~\ref{fig:harmonic-basis-example} for an example of a non-constant harmonic eigenvector). As such, care should be taken when referring to the overpowering asymptotic behavior of Hodge diffusion as `higher-order oversmoothing': while maximal smoothing occurs in a precise, Dirichlet sense, this is distinct from the `equalization of representations' notion of oversmoothing as discussed in the GNN literature. Only when $k=0$ are the two concepts guaranteed to coincide.

\begin{figure}[H]
  \centering
  \begin{subfigure}[t]{0.45\textwidth}
    \centering
    \includegraphics[width=\linewidth]{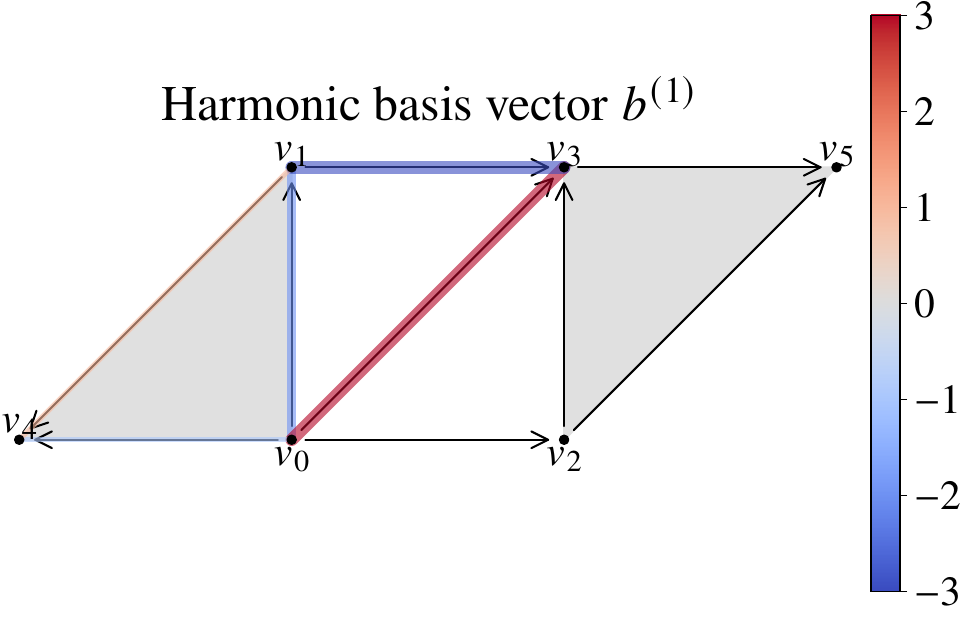}
  \end{subfigure}\hfill
  \begin{subfigure}[t]{0.45\textwidth}
    \centering
    \includegraphics[width=\linewidth]{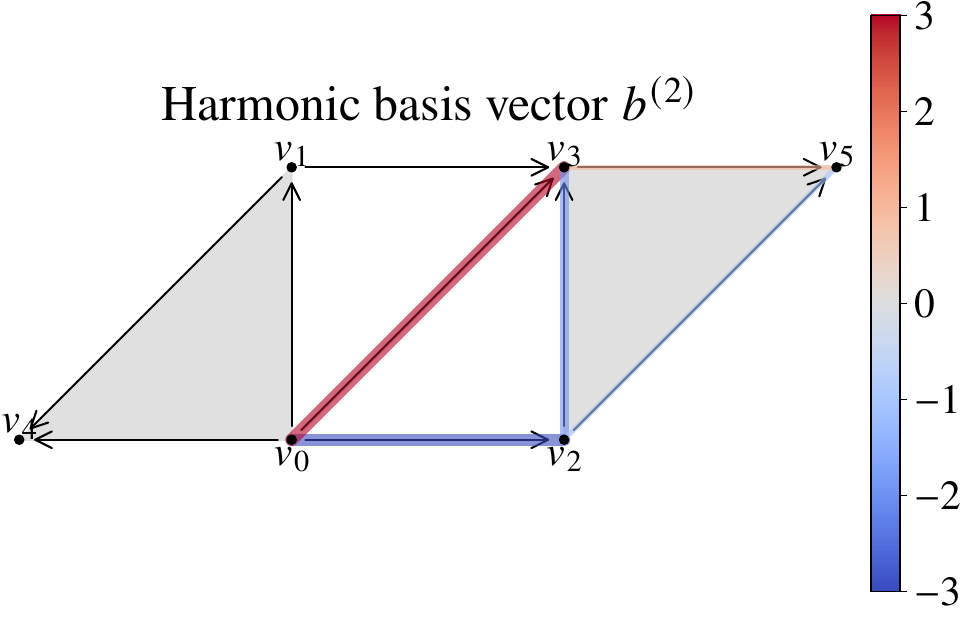}
  \end{subfigure}
  \caption{A harmonic basis for a small simplicial complex.}
  \label{fig:harmonic-basis-example}
\end{figure}

















\paragraph{The Topological Perspective} Per Theorem~\ref{thm:discrete-Hodge-Theorem}, the assignment $\alpha \mapsto [\alpha]$ defines an isomorphism between $\ker \Delta^k$ and the $k$th simplicial cohomology $H^k(S)$ of $S$. As a consequence, Hodge diffusion may be viewed topologically: in the limit, an initial $k$-cochain projects to a representation reflecting its contribution to each $k$-dimensional hole in $S$. In such a sense, the Hodge bias may viewed as the assumption that $k$-dimensional data features on $S$ ($k$-cochains) correlate with $k$-dimensional topological features of $S$ (holes). 

Let us make this claim precise. Write $b_k=\dim H_k(S)=\dim H^k(S)$ for the $k$th Betti number of $S$. Pick $b_k$ cycles $z_1, \dots, z_{\beta_k} \in \ker d_{k}$ descending to a homology basis $[z_1], \dots, [z_k] \in H_k(S)$. These are regarded as encoding the $k$-dimensional holes of $S$. Dualize to a cohomology basis $[\alpha_1], \dots, [\alpha_k] \in H^k(S)$ with each $\alpha_i$ the unique harmonic representative of its class $[\alpha_i]$, $\alpha_i(z_j)=\delta_{ij}$. This yields a basis $\alpha_1, \dots, \alpha_{\beta_k}$ of the harmonic space $\ker \Delta^k$, wherein each $\alpha_i$ corresponds precisely to a hole in $S$ and conversely. (Recovering $\alpha_i$ exactly from $z_i$ can be accomplished e.g. via least squares~\cite{grande_point-level_2025}, but we don't need this for the present theoretical discussion.) For any initial $k$-cochain $\varphi$, its projection onto the harmonic space is then given by $\sum_{i=1}^{\beta_k}$ $\varphi(z_i) \alpha_i$, where the scalar $\varphi(z_i)$ is the (signed) `amount of $\varphi$' contributed to the $i$th hole. The $\sigma$th entry of $\alpha_i$ captures the contribution of the simplex $\sigma$ to the $i$th hole. When $k=0$, the basis elements $\alpha_i$ are of earnest per-connected-component orthogonal (i.e., per $0$-dimensional hole) indicator vectors; when $k>0$ they measure weighted contributions and are (as previously discussed) no longer binary/constant, nor orthogonal. Phrased in terms of the physical perspective provided above, this topological view asserts that curl-free flows without sources or sinks `can only exist around holes'. This fits with intuition from vector calculus and differential geometry.

\subsection{Higher-Order Sheaves in Context}
\label{sec:higher-order-sheaves-in-context}
The perspectives outlined above indicate that the Hodge bias models a specific notion of economicity: that of low-pass flows, without circulation nor flux, correlated with holes in the underlying space. This specializes to the well-studied homophily bias when $k=0$, but diverges from `higher-order homophily' as studied in network science~\cite{sarker2024higher} when $k>0$. Instead, perhaps the most intuitive setting in which the higher-order Hodge bias is leveraged is that of trajectory representation, where it models the assumption that a walk in space rarely backtracks, loops, or fails to exit points it enters~\cite{roddenberry_principled_2021, schaub_random_2020}. Explicitly, a trajectory from node $i_0$ to node $i_m$ on (say) a two-dimensional simplicial complex is viewed as a $1$-(co)chain $[i_0, i_1] + \cdots +[i_{m-1}, i_m]$. However, just as heterophilic data readily provides a setting where the Hodge bias becomes harmful when $k=0$, it is not difficult to imagine instances  where the it is harmful when $k>0$. For example, we know that Hodge diffusion measures the extent to which a trajectory contributes to each hole in the complex, but what of trajectories that do not contribute to any hole at all, i.e., trajectories that live only in the gradient space and/or curl space (Figure~\ref{fig:example-trajs})?  More dramatically, what if $S$ contains no holes at all, so that $\ker \Delta^k=\{0\}$?

\begin{figure}[htbp]
  \centering
  \begin{subfigure}[b]{0.45\textwidth}
    \centering
    \includegraphics[width=\linewidth]{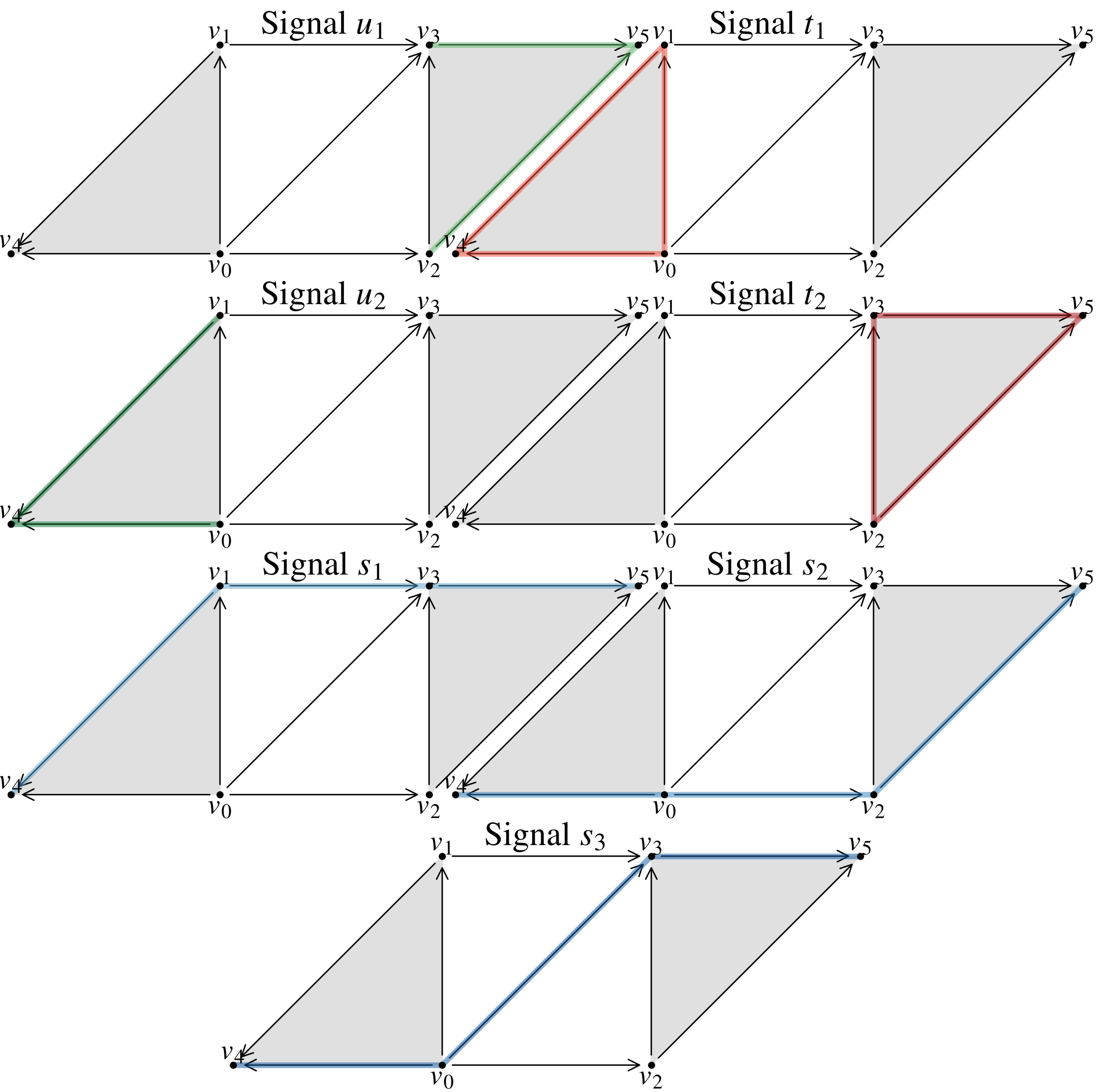}
    \caption{}
    \label{fig:example-trajs-1}
  \end{subfigure}
  \hfill
  \begin{subfigure}[b]{0.45\textwidth}
    \centering
    \includegraphics[width=\linewidth]{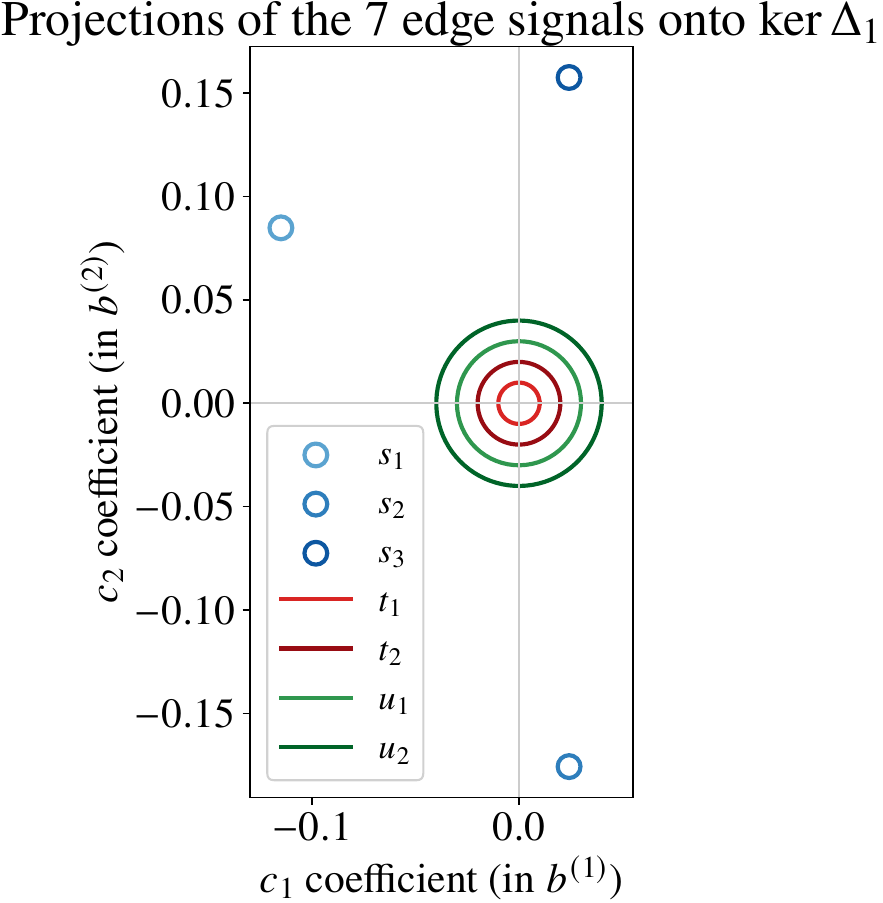}
    \caption{}
    \label{fig:example-trajs-2}
  \end{subfigure}

  \caption{Seven edge flows on the simplicial complex depicted in Figure~\ref{fig:harmonic-basis-example} (\ref{fig:example-trajs-1}), together with their projections onto the harmonic space (\ref{fig:example-trajs-2}). While performing diffusion with the Hodge Laplacian creates a suitable embedding for the \textcolor{blue}{blue} paths, the \textcolor{purple}{red} resp. \textcolor{darkgreen}{green} paths live solely in the curl resp. gradient spaces and therefore are each killed by the diffusion. }
  \label{fig:example-trajs}
\end{figure}
Perhaps a first approach to addressing such concerns would be to perform diffusion with only one of the up- or down- Laplacian. Since $\ker \Delta^k_\text{down}=\ker \Delta^{k} \oplus \operatorname{im }d^{^{*}k}$, performing diffusion with the down-Laplacian damps down-eigenvectors (smoothing out sources and sinks) but not up-eigenvectors (permitting circulation). Similarly,  $\ker \Delta^k_\text{up}=\ker \Delta^{k} \oplus \operatorname{im }d^{k-1}$ and so performing diffusion with the up-Laplacian kills circulations while permitting gradient flows. These constitute perhaps the simplest examples of diffusion with a nonconstant higher-order sheaf. Indeed, define sheaves $F_{\text{up}}$, $F_{\text{down}}$ \ as $F_{\text{up}}(\sigma)=F_{\text{down}}(\sigma)=\mathbb{R}$ for all simplices $\sigma$ and $F_{\text{up}}(v \leq e)=[0]$, $F_{\text{up}}(e \leq t)=[1]$, $F_{\text{down}}(v \leq e)=[1]$, $F_{\text{down}}(e \leq t)=[0]$ for all node-edge incidences $v \leq e$ and edge-triangle incidences $e \leq t$. (Node-triangle restriction maps are defined by composition.) Then, since $d_{F_{\text{up}}}^{0}=0$ and $d^{1}_{F_{\text{up}}}=d^{1}$,    
$$\Delta_{F_{\text{up}}}= d_{F_\text{up}}^{^{*}1} d^{1}_{F_\text{up}}+ d_{F_\text{up}}^{0} d^{^{*}0}_{F_{\text{up}}} =\Delta_{\text{up}}$$ 
and similarly $\Delta_{F_{\text{down}}}=\Delta_{\text{down}}$. Of course, one would expect in practice that the extent to which the Hodge bias is valid varies heterogeneously across the underlying complex. For instance, the underlying complex might approximate an ocean, as investigated in \cite{schaub_random_2020} (Figure~\ref{fig:schaub-ocean-drifters}), in which case e.g. the `curliness' or `harmonic-ness' of, say, a buoy's trajectory would ostensibly depend on the region of the ocean in which it is located. For example, if it is in a gyre, then circulation is an important signal property to preserve and so the Hodge bias is not appropriate; $F_{\text{up}}$ would be better. This can be modeled using a sheaf $F$ obtained by setting stalks as $F(\sigma)=\mathbb{R}$ for all simplices with restriction maps given by $F(v \leq e)=[1]$ for all node-edge incidences and $$F(e \leq t)=\begin{cases}
0 & t \text{ belongs to a gyre} \\
1 & \text{otherwise.}
\end{cases}$$
(Node-triangle restrictions are defined by composition.) If we let $S_+ \subset S$ be the set of triangles outside of gyres and $R:C^2(S) \to C^2(S)$ be the diagonal projector ($R^2=R=R^\top$) that keeps only coordinates in $S_+$, then the modified coboundary is $d_F^1=Rd^1$ and so the sheaf $1$-Laplacian is \begin{equation}
    \Delta_F^1 = d^0 d^{^* 0}+d^{1^*}R d^1.
\end{equation}
The $1$st-order sheaf diffusion thus converges to the orthogonal projection of an initial cochain $x \in C^1$ onto \begin{equation}
    \ker \Delta^1_F = \ker d_F^1 \cap \ker d^{^* 0}_F = \{x \in C^1(S): d^{^*0}x =0 \text{ and } Rd^1 x=0\}.
\end{equation} 
We see that the limit permits only edge flows that have no sources/sinks and no circulation on triangles in $S_+$, while circulations inside gyre regions $S-S_+$ are preserved. 

\begin{figure}
    \centering
    
\end{figure}

\begin{figure}[htbp]
  \centering
  \begin{subfigure}[t]{0.45\textwidth}
    \centering
    \includegraphics[width=\linewidth]{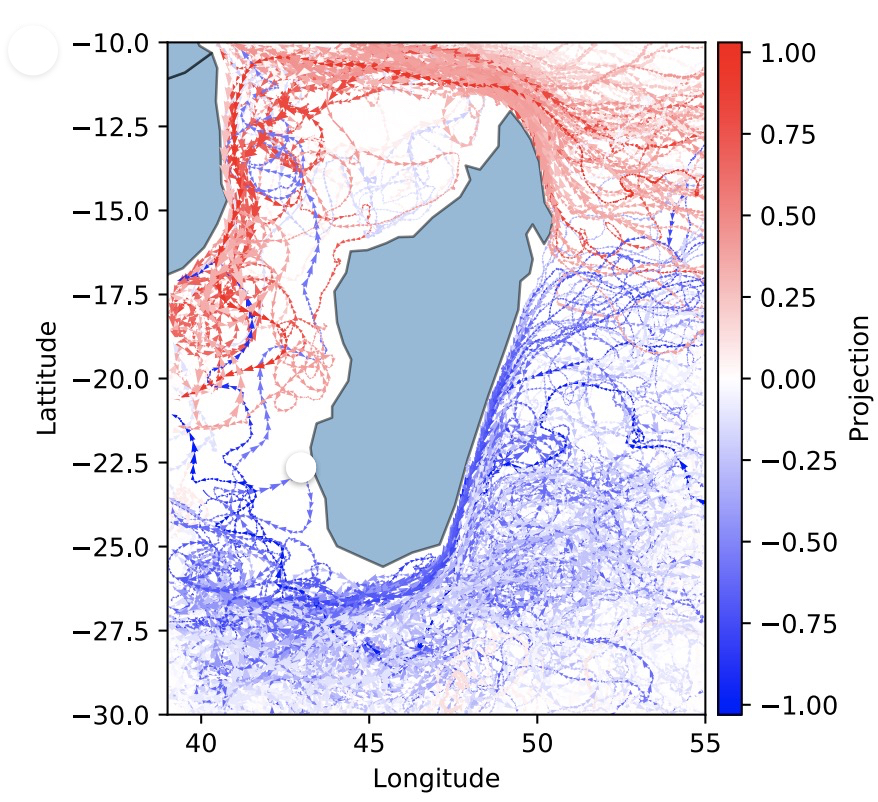}
    \caption{ }
    \label{fig:schaub-ocean-drifters}
  \end{subfigure}\hfill%
  \begin{subfigure}[t]{0.45\textwidth}
    \centering
    \raisebox{.7cm}{
    \includegraphics[width=\linewidth]{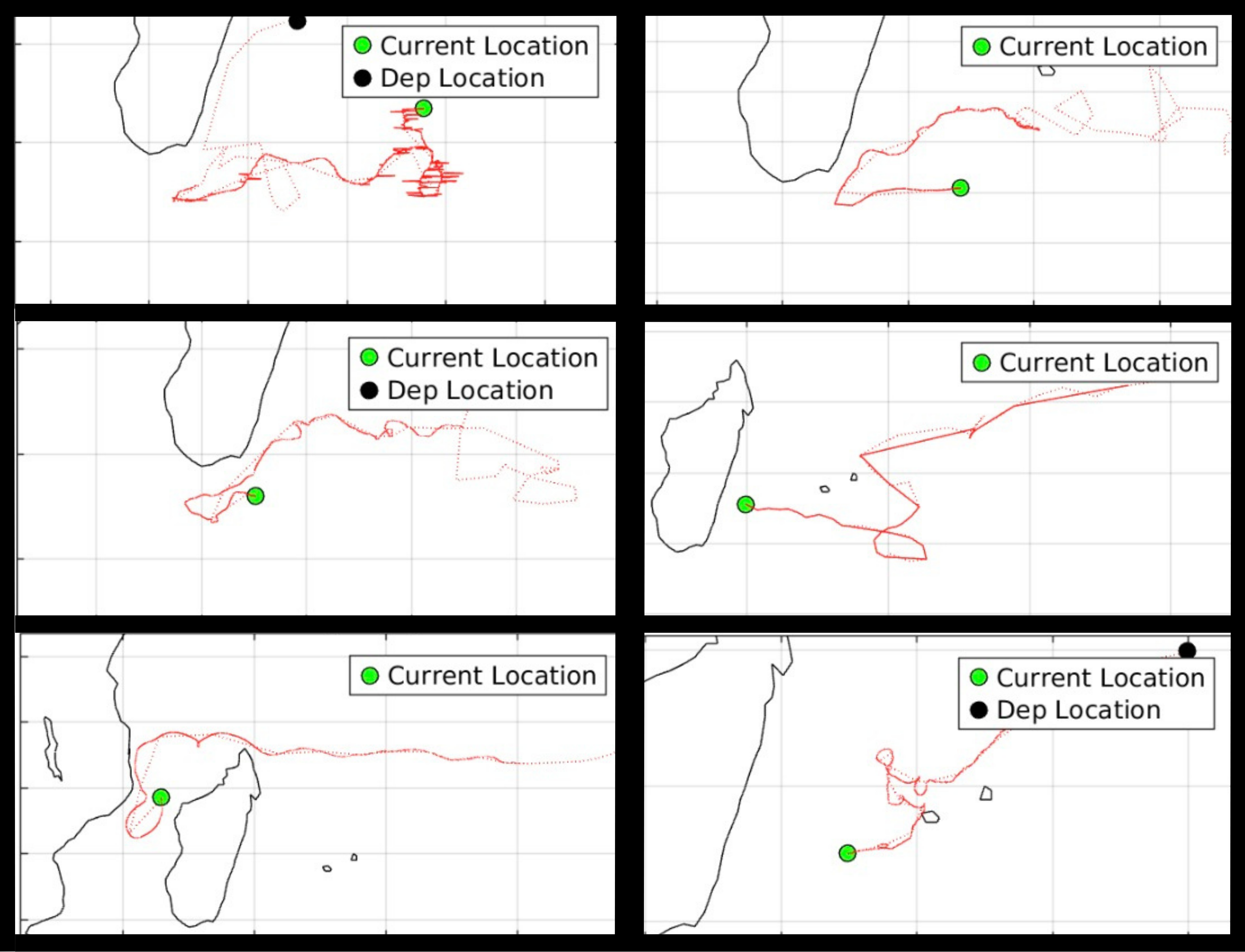}}
    \caption{}
    \label{fig:ocean-buoys-source}
  \end{subfigure}
  \caption{(\ref{fig:schaub-ocean-drifters}): Projection of buoy trajectory data around Madagascar onto the first Hodge-harmonic component of a simplicial complex obtained by discretizing the ocean surface (holes correspond to landmasses). Taken from Schaub et al.~\cite{schaub_random_2020}. (\ref{fig:ocean-buoys-source}): Away from Madagascar, trajectories are empirically less likely to obey the Hodge bias.}
  \label{fig:buoys}
\end{figure}

A straightforward extension of this reasoning shows that weights on a simplicial complex define a sheaf. The sheaf formalism further comes into play when some stalks are multidimensional and restriction maps provide coupling between the dimensions. To continue our running ocean trajectories example, one might imagine incorporating some temporal prior: maybe a certain region has been curl-y of late due to a storm, but is usually quite harmonic-y. Taking edge-stalks to be $\mathbb{R}^{2}$ and triangle-stalks to be $\mathbb{R}$, a restriction map in this region might look like $$F(e \leq t)=\begin{bmatrix}
\overbrace{ 0.1  }^{ \text{last checked} }& \overbrace{ 0.8 }^{ \text{on average} }
\end{bmatrix}$$
where the two values may be interpreted as `to what extent should curliness be treated as noise and hence smoothened away?'. One can lift trajectories to $\mathbb{R}^{2}$ based on how much prioritization the temporal prior should initially be given, e.g. $$\begin{bmatrix}
\overbrace{ 0.1  }^{ \text{last checked} }& \overbrace{ 0.8 }^{ \text{on average} } 
\end{bmatrix} \begin{bmatrix}
1 \\ 1
\end{bmatrix}=\text{last checked + on average}.$$
for equal weight, $$\begin{bmatrix}
\overbrace{ 0.1  }^{ \text{last checked} }& \overbrace{ 0.8 }^{ \text{on average} } 
\end{bmatrix} \begin{bmatrix}
\alpha \\ \beta
\end{bmatrix}=\alpha\text{(last checked) + $\beta$ (on average)}$$
for some other choice. We will revisit these examples in Section~\ref{sec:learning-with-sheaves}.

\section{Cohomology and Diffusion for Sheaves on Posets}
\label{sec:sheaf-cohomology-diffusion}
With the previous sections' machinery in hand, we return to the setting of sheaves supported on posets, and in particular to studying the diffusion processes to which they give rise. Per the identifications in the dictionary (\ref{eqn:diff-dictionary}), we know that studying sheaf diffusion should be closely tied to studying sheaf cohomology. This is the subject to which we now turn. As has been a recurring theme in this document, our discrete setting simplifies the general discussion considerably. Throughout this section, we assume $\mathsf{D}$ is an abelian category with enough injectives (for the relevant definitions, see \cite{weibel1994introduction}). This is true for all categories discussed so far; in particular, every object in $\mathbb{R}\mathsf{Vect}$ is injective.

\subsection{Sheaf Cohomology}

If $0 \to \mathcal{F}' \to \mathcal{F} \to \mathcal{F}''\to 0$ is a short exact sequence of sheaves,  then left-exactness of the global sections functor $\Gamma(X, -)$ implies 
\begin{equation}
0 \to \Gamma(X, \mathcal{F}') \to \Gamma(X, \mathcal{F}) \to \Gamma(X, \mathcal{F}'')
\end{equation}
is exact. But we don't know anything about surjectivity on the right. In a very broad sense, the goal of sheaf cohomology is to extend this sequence to better understand $\Gamma(X, \mathcal{F}'')$. 
The language of right derived functors~\cite{hartshorne2013algebraic} is the `right way to derive sheaf cohomology'. As space constraints prevent us from developing said language in this document, we jump to the punchline.  

\begin{definition}
\label{def:sheaf-cohomology}
For $i \geq 0$, there exist unique covariant functors $H^{i}(X, -):\mathsf{Shv}_{\mathsf{D}}(X) \to \mathsf{D}$ satisfying the following properties. 

\begin{enumerate}
\item $H^{0}(X, \mathcal{F})=\Gamma(X, \mathcal{F})$ 
 
\item (\textbf{Zig-Zag}) Whenever $0 \to \mathcal{F}' \to \mathcal{F} \to \mathcal{F}'' \to 0$ is exact, there are \textbf{connecting maps} $\delta:H^{i}(X, \mathcal{F}'') \to H^{i+1}(X, \mathcal{F}')$ fitting into a long exact sequence  

\begin{center}
    \begin{tikzcd}
0 \arrow[r] & H^0(X, \mathcal{F}') \arrow[r] & H^0(X, \mathcal{F}) \arrow[r] \arrow[d, phantom, ""{coordinate, name=A}] & H^0(X, \mathcal{F}'') \arrow[dll, Curved=A, "\delta"] \\
& H^1(X, \mathcal{F}') \arrow[r] & H^1(X, \mathcal{F}) \arrow[r] \arrow[d, phantom, ""{coordinate, name=B}] & H^1(X, \mathcal{F}'') \arrow[dll, Curved=B, "\delta"] \\
& H^2(X, \mathcal{F}') \arrow[r] & \cdots & \ 
\end{tikzcd}
\end{center}

\item (\textbf{Naturality}) Given a morphism of short exact sequences

\begin{center}
    \begin{tikzcd}
0 \arrow[r] & \mathcal{F}' \arrow[r] \arrow[d] & \mathcal{F} \arrow[r] \arrow[d] & \mathcal{F}'' \arrow[r] \arrow[d] & 0 \\
0 \arrow[r] & \mathcal{G}' \arrow[r]           & \mathcal{G} \arrow[r]           & \mathcal{G}'' \arrow[r]           & 0
\end{tikzcd}
\end{center}

the following diagram commutes for all $i$: 

\begin{center}
    \begin{tikzcd}
{H^i(X, \mathcal{F}'')} \arrow[d] \arrow[r, "\delta"] & {H^{i+1}(X, \mathcal{F}')} \arrow[d] \\
{H^i(X, \mathcal{G}'')} \arrow[r, "\delta"]           & {H^{i+1}(X, \mathcal{G}')}          
\end{tikzcd}
\end{center}

\item When $\mathcal{F}$ is flasque (i.e., all restriction maps are surjective), we have $H^{i}(X, \mathcal{F})=0$ for all $i>0$.
\end{enumerate}

The object $H^{i}(X,\mathcal{F})$ is called the \textbf{$i$th cohomology of the sheaf $\mathcal{F}$}.

\end{definition}

As is not uncommon for (co)homology theories, a definition suitable e.g. for checking properties is often unsuitable for performing computations. This is the case for Definition~\ref{def:sheaf-cohomology}. For computations, a common strategy is to find a cochain complex whose cohomology agrees with $H^*(X; F)$ as defined above. 

\begin{definition}[Ayzenberg et al.~\cite{ayzenberg2025sheaf}] For $X$ a topological space, say a functor $\mathscr{F}:\mathsf{Shv}_{\mathbb{R}\mathsf{Vect}}(X) \to \mathsf{Cochain}(\mathbb{R}\mathsf{Vect})$ \textbf{honestly computes sheaf cohomology} if $H^{*}\big( \mathscr{F}(F) \big) \cong H^{*}(X, F)$ for all sheaves $F$ on $X$.
\end{definition}

Luckily, for a sheaf $F$ supported on a poset $S$ there is a canonical functor honestly computing $H^*(S, F)$.

\begin{definition}
    A collection $\mathcal{K}$ of nonempty finite subsets of a set $M$ is called an \textbf{(abstract) simplicial complex} if it is stable with respect to inclusion.\footnote{ That is, if $I \in \mathcal{K}$ and $J \subset I$, then $J \in \mathcal{K}$. } Elements $F$ of $\mathcal{K}$ are called \textbf{faces}, or \textbf{simplices}, of \textbf{dimension $|F|-1$}. The \textbf{vertex set} $V=V(\mathcal{K})$ of $\mathcal{K}$ is the union of all its faces. It follows from the definition that every face $F$ is a union of vertices and for every vertex $v \in V$, $\{ v \} \in \mathcal{K}$.\footnote{It is common to enforce $M=V$.} We will always assume $V$ is well-ordered. 
    \end{definition}

\noindent A simplicial complex $\mathcal{K}$ is partially ordered by inclusion; the resulting poset is denoted by $\operatorname{Cells}(\mathcal{K})$ or just (again) by $\mathcal{K}$. One recovers the standard geometrical realization of $\mathcal{K}$ as the Euclidean subspace 
\begin{equation}
|\mathcal{K}|= \bigcup_{F \in \mathcal{K}} \Delta_{F} \subset \mathbb{R}^{|V|},
\end{equation}
where $\Delta_{F}$ is the standard simplex $\Delta_{F}=\operatorname{ConvHull}(e_{i}: i \in F)$. 


\begin{definition}
    Let $S$ be a poset. The \textbf{order complex} of $S$ is the (abstract) simplicial complex $\operatorname{ord}(S)$ whose faces/simplices are chains in $S$.
\begin{figure}
    \centering
    \includegraphics[width=0.5\linewidth]{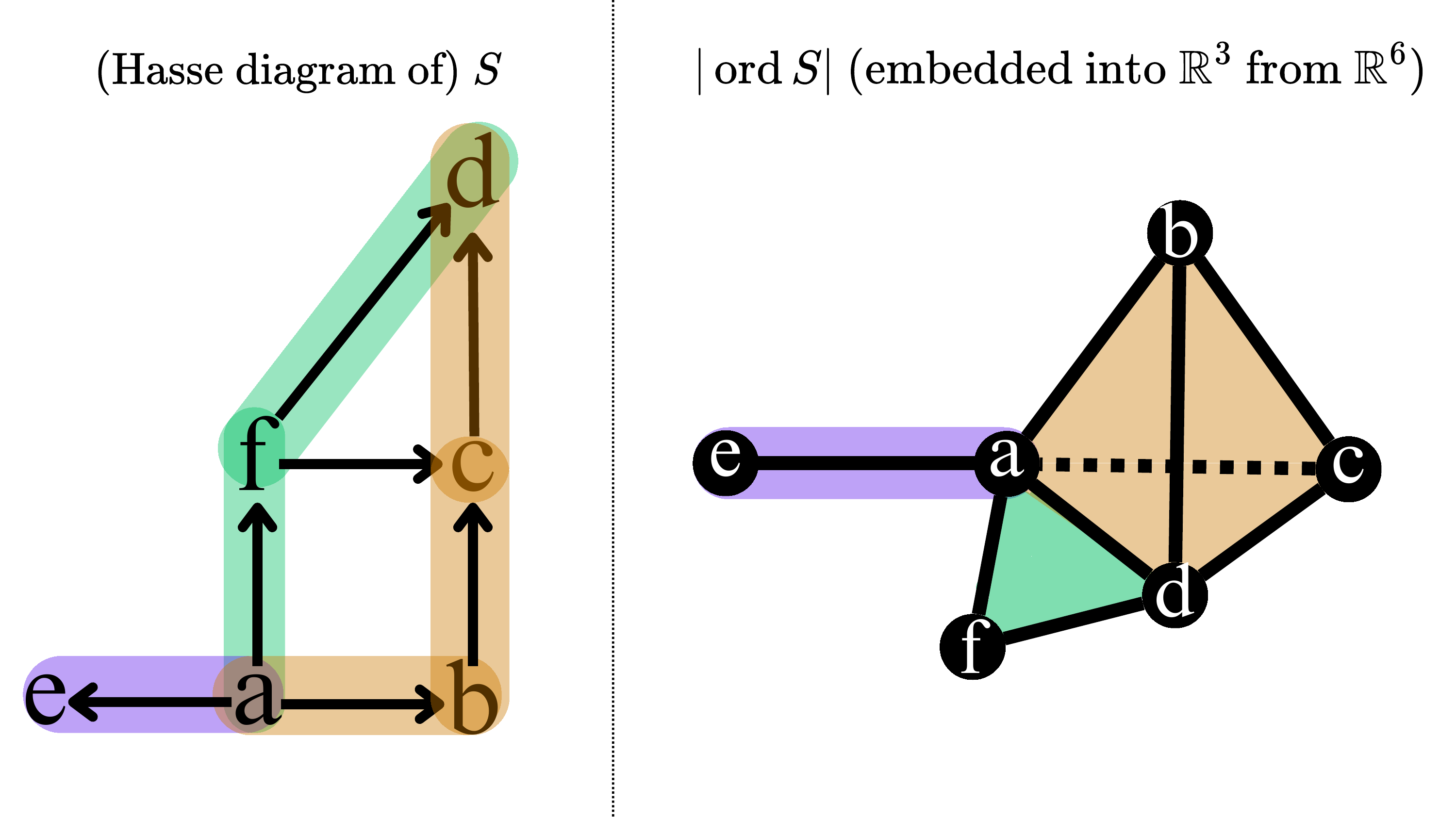}
    \caption{A poset $S$ and (an embedding into $\mathbb{R}^3$ of) the geometric realization of $\operatorname{ord}S$.}
    \label{fig:enter-label}
\end{figure}
\end{definition}

\begin{definition}
\label{def:roos-complex-cohomology}
    If $\sigma$ is a $(j+1)$-dimension \textit{simplex} of $\operatorname{ord}(S)$, i.e., a \textit{chain} $s_{0}<s_{1}<\dots<s_{j+1}$, and $\tau$ is a $j$-dimensional simplex, define the  \textbf{incidence number} 
\begin{equation}
[\sigma:\tau]:= 
\begin{cases}
(-1)^{k} & \tau = \sigma\setminus \{ s_{k} \} \text{ for some } k \ ;\\
0 & \text{otherwise.}
\end{cases}
\end{equation}
The \textbf{Roos complex} of $F$ is the \textit{cochain complex} of \textit{vector spaces} defined via
\begin{equation}
C^{j}_{\operatorname{Roos}}(S; F)=\bigoplus_{\tau \in \operatorname{ord}(S), \  \operatorname{dim }\tau=j}F(\tau_{\max}),
\end{equation}
where $\tau_{\operatorname{max}}=s_{j}$ is the maximal element of the \textit{chain} $\tau=s_{0} \leq \dots \leq s_{j}$. The codifferential  
\begin{align}
d&:C^{j}_{\operatorname{Roos}}(S; F) \to C^{j+1}_{\operatorname{Roos}}(S; F) 
\end{align}
acts on a cochain $\boldsymbol x=(x_{\tau})\in C^{j}_{\operatorname{Roos}}(S ;F)$ via 
\begin{equation}
(d^{j}\boldsymbol  x)_{\sigma}=\sum_{\tau \subsetneq \sigma, \operatorname{dim } \tau=j } [\sigma : \tau] F(\tau_{\operatorname{max}} \leq \sigma _{\operatorname{max}}) x_{\tau}
\end{equation}
for $\sigma \in \operatorname{ord}(S), \operatorname{dim }\sigma=j+1$.\footnote{We verify that $d^2=0$. Let $\boldsymbol x$ be a $j$-cochain. Then the $\beta$th entry of $d^{j+1} d^{j} \boldsymbol x$ is

\begin{align}
\textcolor{LimeGreen}{(}d^{j+1} \textcolor{Thistle}{d^{j} \boldsymbol x}\textcolor{LimeGreen}{)_{\beta}}
&= \textcolor{LimeGreen}{(}  d^{j+1} \big(  \textcolor{Thistle}{(}\textcolor{Thistle}{\sum_{\tau < \sigma, \operatorname{dim } \tau=j} [\sigma :\tau] F(\tau_{\max} \leq \sigma_{\max}) x _{\tau} )_{\{ \sigma  : \operatorname{dim } \sigma=j+1\}}} \big)  \textcolor{LimeGreen}{)_{\beta}}  \\
&=  \textcolor{Thistle}{\sum_{\tau < \sigma, \operatorname{dim } \tau=j} [\sigma: \tau]}  \textcolor{LimeGreen}{(}  d^{j+1}  \big( \textcolor{Thistle}{(F  (\tau_{\max } \leq \sigma_{\max})x_{\tau} \textcolor{Thistle}{)_{\{ \sigma:\operatorname{dim } \sigma=j+1 \}}}}  \big)   \textcolor{LimeGreen}{)_{\beta}} \\
&= \textcolor{Thistle}{\sum_{\tau < \sigma, \operatorname{dim } \tau=j} [\sigma: \tau]}   \textcolor{LimeGreen}{(} d^{j+1}  \big( \textcolor{Thistle}{(F  (\tau_{\max } \leq \sigma_{\max})x_{\tau} \textcolor{Thistle}{)_{\{ \sigma:\operatorname{dim } \sigma=j+1 \}}}}  \big)   \textcolor{LimeGreen}{)_{\beta}} \\
&= \textcolor{Thistle}{\sum_{\tau < \sigma, \operatorname{dim } \tau=j} [\sigma: \tau]}    \sum_{\sigma < \beta, \operatorname{dim } \sigma=j+1}[\beta:\sigma]F(\tau_{\max} \leq \beta_{\max})x_{\tau} \\
&= \textcolor{Thistle}{\sum_{\tau < \sigma, \operatorname{dim } \tau=j}}  \sum_{\sigma < \beta, \operatorname{dim } \sigma=j+1}\textcolor{Thistle}{ [\sigma: \tau]} [\beta : \sigma]   F(\tau_{\max} \leq \beta_{\max})x_{\tau}
\end{align}

\noindent where we have used $F(\sigma_{\max} \leq \beta_{\max}) F(\tau_{\max} \leq \sigma_{\max})=F(\tau_{\max} \leq \beta_{\max})$. The product $[\sigma:\tau][\beta:\sigma]$ is necessarily zero unless $\tau<\sigma<\beta$. The final summation is therefore

\begin{equation}
\sum_{\tau < \beta, \operatorname{dim } \tau=j}  \big(  \sum_{\tau < \sigma < \beta, \operatorname{dim } \sigma=j+1} [\sigma: \tau] [\beta : \sigma] \big) F(\tau_{\max} \leq \beta_{\max} ) x_{\tau}.
\end{equation}

\noindent It is then a general fact~\cite{ayzenberg2025sheaf} that the inner sum is always zero.
}
The \textbf{Roos cohomology} of the sheaf $F$ on $S$ is the cohomology $H^{*}_{\operatorname{Roos}}(S; F):= H^{*}_{\operatorname{Roos}}\big( C^{\bullet}_{\operatorname{Roos}}(S;F) \big)$ of the Roos complex.

\end{definition}


Observe that if $\boldsymbol x \in C^{0}_{\operatorname{Roos}}(S; F)$, then
\begin{equation}
\label{eqn:Roos-kernel-is-global-sections}
d^{0} \boldsymbol  x=0 \iff (d^{0} \boldsymbol  x)_{a \leq b}=0 \text{ for all } a \leq b \iff x_{b}-F(a \leq b) x_{a}=0 \text{ for all } a \leq b.
\end{equation}
In other words, $\operatorname{ker }d^{0}=H^{0}_{\operatorname{Roos}}( S; F)$ is exactly the space $\Gamma(S; F)$ of global sections. In fact, the two cohomologies agree in all degrees:



\begin{theorem}
\label{thm:Roos-complex-computes-sheaf-cohomology}
    The Roos cohomology $H^{*}_{\operatorname{Roos}}(S; F)$ of a sheaf $F$ on a poset $S$ is isomorphic to the cohomology $H^*(X_S; F)$. That is, $H^*_{\operatorname{Roos}}$ honestly computes cohomology.
\end{theorem}

Just as we did not have the tools to justify Definition~\ref{def:sheaf-cohomology}, we will not have the tools to prove Theorem~\ref{thm:Roos-complex-computes-sheaf-cohomology} in its full generality. Such a proof may be found e.g. in~\cite{ayzenberg2025sheaf}. Do note that we \textit{have} proven Theorem~\ref{thm:Roos-complex-computes-sheaf-cohomology} in grading zero, which will in practice be all we need from Section~\ref{sec:linear-separation-power-sheaf-diffusion} onward.

\begin{example}
\label{ex:Roos-complex-graph-example}
    Let $S=\{ u, v , e : u \leq e, v \leq e \}$, and let $F$ be a diagram/sheaf on $S$. Viewing $S$ as a graph $S=(\{ u, v \}, \{ e \})$, depict $F$ as follows:

    \begin{figure}[H]
    \centering
    \includegraphics[width=0.5\linewidth]{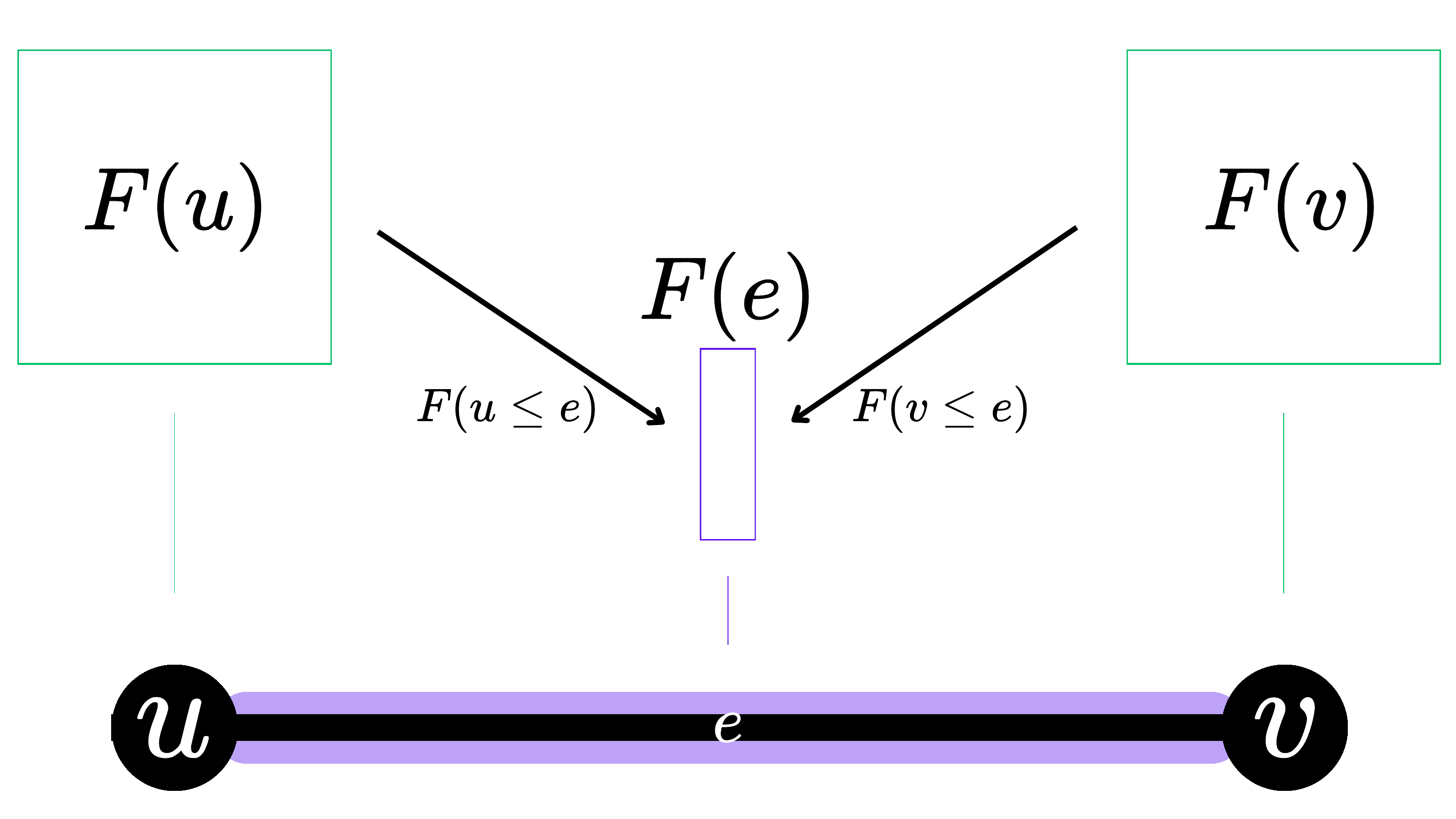}
    \caption{Depiction of the diagram defined in Example~\ref{ex:Roos-complex-graph-example}.}
    \label{fig:simple-diagram-roos-example}
\end{figure}

The Roos cochain spaces are

\begin{equation}C^{j}_{\operatorname{Roos}}(S; F)=
    \begin{cases}
    F(u) \oplus F(e) \oplus F(v)  & j=0\\
     F(e) \oplus F(e) & j =1\\
     0 & \text{otherwise.}
\end{cases}
\end{equation}

\begin{figure}[H]
    \centering
    \includegraphics[width=0.5\linewidth]{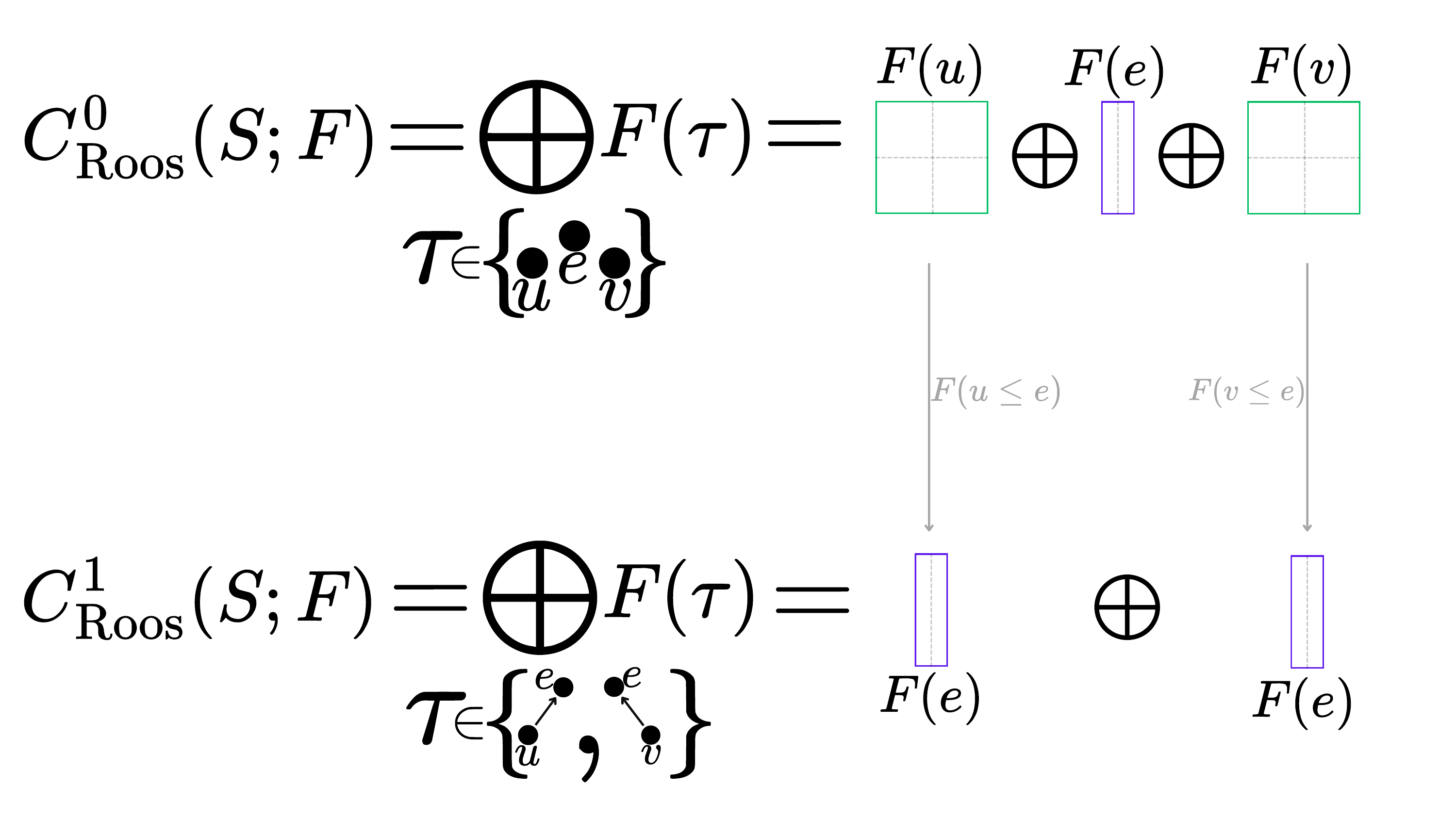}
    \label{fig:enter-label}
\caption{Roos cochain spaces corresponding to the diagram depicted in Figure~\ref{fig:simple-diagram-roos-example}.}
\end{figure}

The incidence numbers are

\begin{align}
[ u < e : u ]& = -1 \\
[v < e : v] &= -1 \\
[u < e : e] &= 1 \\
[v < e : e] &= 1.
\end{align}

If $\boldsymbol x=(x_{u}, x_{e}, x_{v}) \in C^{0}_{\operatorname{Roos}}(S ;F)$, then $(d^{0} \boldsymbol  x)_{u \leq e} = [u < e : u] F(u \leq e)x_{u} + [u< e: e]F(e \leq e)x_{e}=-F(u \leq e)x_{u} + x_{e}$, and similarly $(d^{0}\boldsymbol x)_{v \leq e}=-F(v \leq e)+x_{e}$. Hence
\begin{equation}
d^{0} \boldsymbol  x= (x_{e} - F_{u \leq e}x_{u}, x_{e}- F_{v \leq e} x_{v}).
\end{equation}

\begin{figure}[H]
    \centering
    \includegraphics[width=\linewidth]{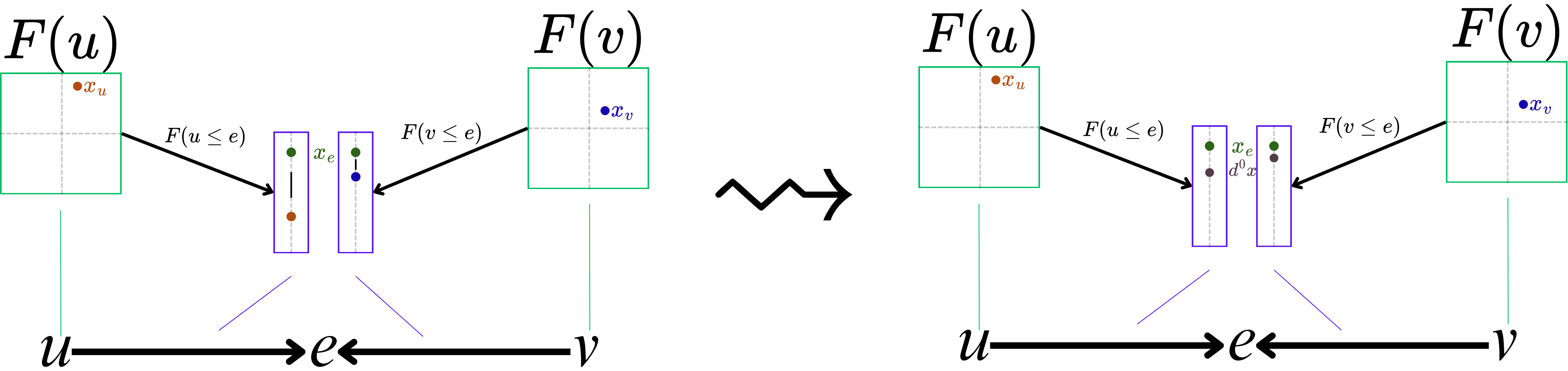}
    \label{fig:enter-label}
    \caption{The Roos codifferential corresponding to the diagram depicted in Figure~\ref{fig:simple-diagram-roos-example}.}
\end{figure}

\end{example}

Theorem~\ref{thm:Roos-complex-computes-sheaf-cohomology} shows that Roos complexes provide a universal formalism for computing the cohomology of sheaves supported on posets. In Example~\ref{ex:Roos-complex-graph-example}, however, there was redundancy: the edge $e$ gave rise to a summand grading zero despite `being one-dimensional', and each relation $u \leq e$, $v \leq e$ gave rise to an individual summand in grading one, even though the two relations regard the same edge $e$. Loosely speaking, the Roos functor `does not know that $e$ is an edge'. Should the posets in question in fact correspond to \textit{cell complexes} as in Example~\ref{ex:Roos-complex-graph-example}, a computationally superior (but still equivalent) cohomology framework emerges: that of \textit{cellular sheaf cohomology} \cite{curry2014sheaves, ayzenberg2025sheaf}. 

\begin{definition}
    \textbf{Definition.} (graded poset)\\
A \textbf{grading}, or \textbf{rank function}, on a (finite) poset $S$ is a function $\operatorname{rk}:S \to \mathbb{Z}$ satisfying the following:
\begin{enumerate}
\item (Strict monotonicity) If $s_{1}<s_{2}$ then $\operatorname{rk }s_{1} < \operatorname{rk }s_{2}$;
\item (Respects covering relation) If $s_{1}<s_{2}$ is a covering relation, then $\operatorname{rk }s_{2}=\operatorname{rk }s_{1}+1$;
\item If $s$ is minimal, then $\operatorname{rk }s=0$.
\end{enumerate}
We call the pair $(S, \operatorname{rk})$ a \textbf{graded poset}.
These properties turn out to determine a unique grading on $S$, should a grading exist at all. Specifically, if $S$ admits a grading $\operatorname{rk}:S \to \mathbb{Z}$, $s \in S$, and
\begin{align}
s_{1}<s_{2}<\dots<s_{k}<s
\end{align}
is a maximal chain descending from $s$, then the rank of $s$ is necessarily the length of this chain: $\operatorname{rk }s=k$.\footnote{Indeed, suppose $S$ admits a grading $\operatorname{rk}: S \to \mathbb{Z}$. Let $s \in S$. If $s$ is minimal, then $\operatorname{rk }s=0$. Otherwise, $s$ has a descending maximal chain $s_{1}<\dots<s_{k}<s$ for some $k \geq 1$. Since this chain is maximal, each comparison in it is a covering; the result follows inductively.} (It follows that if such a definition is not well-defined irrespective of the choice of maximal chain, then no grading exists.)
\end{definition}

In this document, we assume all posets encountered are graded unless stated otherwise.

\begin{example}
    The poset $\operatorname{Cells}(\mathcal{K})$ of nonempty simplices of an (abstract) simplicial complex $\mathcal{K}$ is naturally graded by dimension: $\operatorname{rk }F=\operatorname{dim }F$.
\end{example}

\begin{definition}
    A graded poset $(\mathcal{X}, \operatorname{rk})$ is called a \textbf{cell poset} if, for any $\sigma \in \mathcal{X}$, the boundary $|\partial{C}(\sigma)|$\footnote{Recall that the \textbf{boundary} of $s \in S$ is the lower set $\partial C(s)=S_{<s}=\{ t \in S: t < s \}$.
    } is homeomorphic to the sphere $S^{\operatorname{rk }\sigma - 1}$. In this case, a rank-$k$ element $\sigma \in \mathcal{X}$ is called a \textbf{$k$-dimensional cell}.
\end{definition}

\begin{definition}
    A \textbf{cellular sheaf} $F$ is a diagram on a cell poset $\mathcal{X}$. 
\end{definition}


\begin{definition}
\label{def:cellular-sheaf-cohomology}
Let $\mathcal{X}$ be a \textit{cell poset}, and $F$ a \textit{diagram} (cellular \textit{sheaf}) supported on $\mathcal{X}$. The \textbf{cellular cochain complex corresponding to $F$} is the \textit{cochain complex} of \textit{vector spaces} defined via
\begin{equation}
C^{j}_{\operatorname{cell}}(\mathcal{X}; F)=\bigoplus_{\tau \in \mathcal{X}, \operatorname{rk }\tau=j}D(\tau).
\end{equation}
The codifferential
\begin{equation}
d^{j}_{\operatorname{cell}}:C^{j}_{\operatorname{cell}}(\mathcal{X}; F) \to C^{j}_{\operatorname{cell}}(\mathcal{X}; F)
\end{equation}
acts on a cochain $\boldsymbol x \in C^{j}_{\operatorname{cell}}(\mathcal{X};F)$ via
\begin{equation}
(d^{j} \boldsymbol   x)_{\sigma}=\sum_{\tau < \sigma, \operatorname{rk }\tau=j, \operatorname{rk }\sigma=j+1}[\sigma: \tau]  F(\tau < \sigma),
\end{equation}
where the incidence numbers $[\sigma:\tau]$ are defined as in Part III Algebraic Topology.\footnote{The verification $d_{\operatorname{cell}}^2=0$ is also as in Part III Algebraic Topology.} (In practice, the orientation signs in the codifferential definition are `the obvious ones'.)

The \textbf{cellular sheaf cohomology} of $F$ is the \textit{cohomology} $H^{*}_{\operatorname{cell}}(\mathcal{X}; D):=H^{*}_{\operatorname{cell}}\big( C^{\bullet}(\mathcal{X};D) \big)$.
\end{definition}

Analogous to (\ref{eqn:Roos-kernel-is-global-sections}), observe that if $\boldsymbol x=(x_{v})_{v : \operatorname{rk }v=0} \in C^{0}_{\operatorname{cell}}(S;F)$ lives in $H^{0}_{\operatorname{cell}}(S;F)=\operatorname{ker }d^{0}_{\operatorname{cell}}$, then $$(d^{0}_{\operatorname{cell}}\boldsymbol  x)_{e}=F_{u \leq e} x _{u}-F_{v \leq e} x _{v}=0$$
for all nodes $v$ and edges $e$ in the cell complex, i.e., $$F_{u \leq e}   x_{u}= F_{v \leq e}  x_{v}$$
for all nodes $v$ and edges $e$ in the cell complex. Consideration of compositionality constraints and induction show that $\boldsymbol x$ determines a \textit{unique} global section $\tilde{\boldsymbol x} \in \Gamma(S; F) \subset \bigoplus_{n \geq 0} C^{n}(S ;F)$, yielding a natural isomorphism \begin{equation}
\label{eqn:zeroth-cellular-cohomology-global-sections}
    H^{0}_{\operatorname{cell}}(S; F) \cong \Gamma(S; F).
\end{equation} 
This means that cellular sheaf cohomology and sheaf cohomology (Definition~\ref{def:sheaf-cohomology}) agree in grading zero. As we saw before with Roos cohomology, they in fact agree in all degrees, as one would hope. \begin{theorem}
    The cellular cohomology $H^*_{\operatorname{cell}}(S; F)$ of a sheaf $F$ on a cell poset $S$ is isomorphic to the cohomology $H^*(X_S;F)$. That is, $H^*_{\operatorname{cell}}$ honestly computes cohomology.
\end{theorem}

\begin{example}
\label{ex:cell-complex-graph-example}
    The {poset} $S$ in Example~\ref{ex:Roos-complex-graph-example} (a graph) is naturally {graded}, with $\operatorname{rk }u=\operatorname{rk }v=0$ and $\operatorname{rk }e=1$. The {boundary} of $e$ consists of a pair of points, hence is {homeomorphic} to he $0$-sphere $\mathbb{S}^{0}$. $S=\mathcal{X}$ is therefore a {cell poset}; the {diagram} $F$ on $\mathcal{X}$ a {cellular sheaf}. One has
\begin{equation}
C^{j}_{\operatorname{cell}}(\mathcal{X};F)=
\begin{cases}
F(u) \oplus F(v) & j=0 \\
F(e) & j=1 \\
0 & j>1.
\end{cases}
\end{equation}
\begin{figure}[H]
    \centering
    \includegraphics[width=0.5\linewidth]{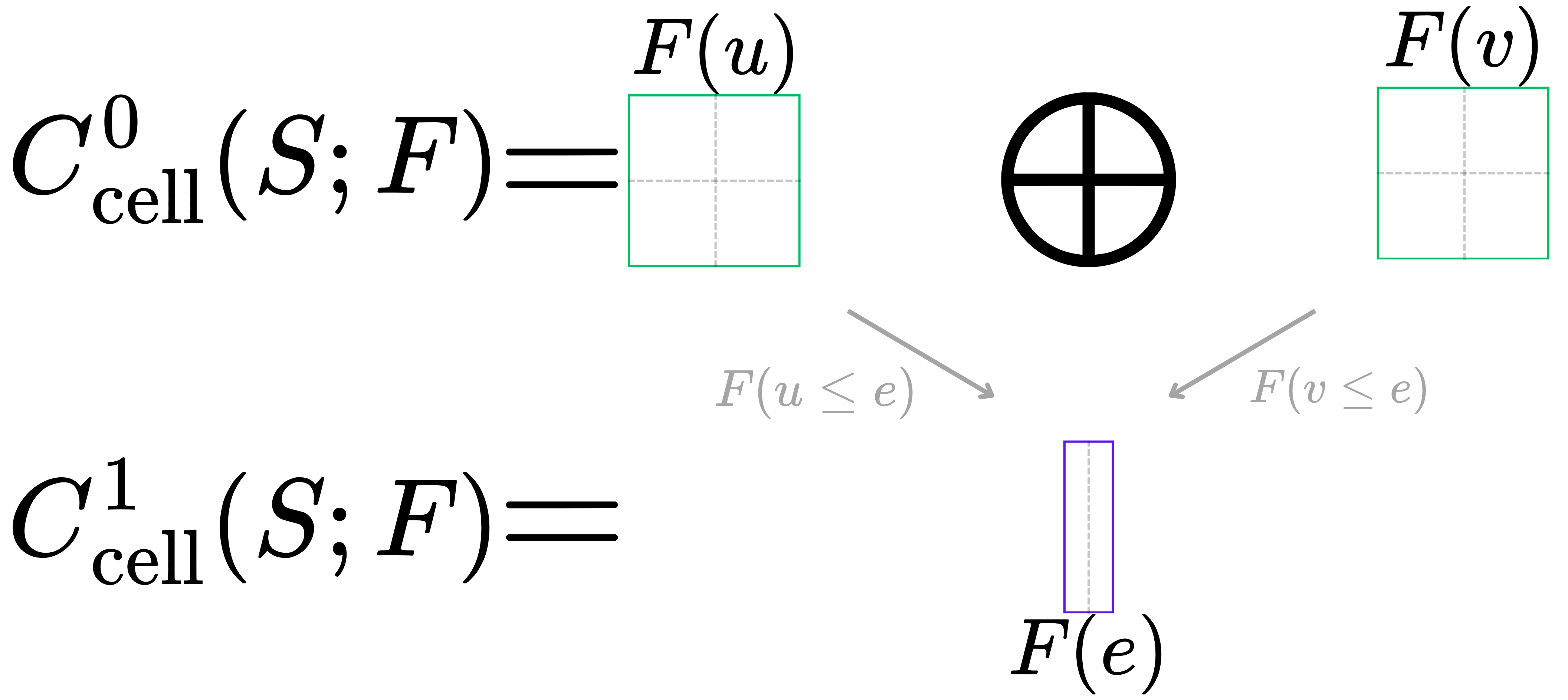}
    \label{fig:enter-label}
    \caption{The cellular cochain spaces corresponding to the diagram depicted in Figure~\ref{ex:Roos-complex-graph-example}.}
\end{figure}

Orienting $e$ to go from $u$ to $v$, the codifferential is
\begin{align}
d^{0}_{\operatorname{cell}}\boldsymbol  x&=d^{0}_{\operatorname{cell}}(x_{u} , x_{v}) \\
&= F(v \leq e) x_{v} - F(u \leq e)x_{u}.
\end{align}

\begin{figure}[H]
    \centering
    \includegraphics[width=\linewidth]{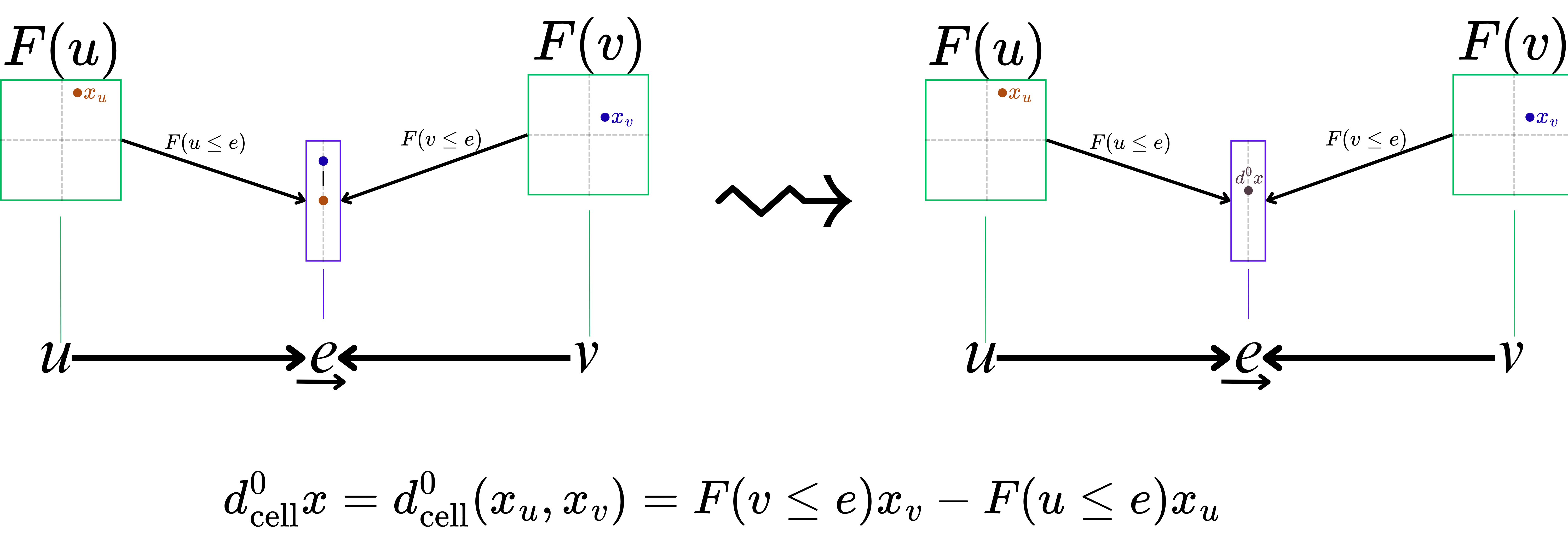}
    \label{fig:enter-label}
    \caption{Cellular codifferential for the diagram depicted in Figure~\ref{ex:Roos-complex-graph-example}.}
\end{figure}

\end{example}

The cellular computation in Example~\ref{ex:cell-complex-graph-example} is manifestly simpler, and more intuitive, than its Roos analogue in Example~\ref{ex:Roos-complex-graph-example}. Hereon, we will use cellular cohomology by default when working with cell posets. Roos cohomology will still be useful when we want to be able to accommodate non-cellular spaces such as hypergraphs.

We conclude this section by capturing into a definition a property commonly satisfied by cochain complexes honestly computing sheaf cohomology (certainly all which have been encountered thus far). 

\begin{definition}[Ayzenberg et al.\cite{ayzenberg2025sheaf}]
    A functor $\mathscr{F}:\mathsf{Shv}_{\mathbb{R}\mathsf{Vect}}(S) \to \mathsf{Cochain}(\mathbb{R}\mathsf{Vect})$ is said to be \textbf{concrete} if, for all $k$, the vector space $\mathscr{F}(F)^k$ is a direct sum of stalks of $F$. 
\end{definition}

\begin{example}
    The functors underlying Definitions~\ref{def:roos-complex-cohomology} and~\ref{def:cellular-sheaf-cohomology} are concrete.
\end{example}

\subsection{Sheaf Diffusion}

We now have a definition of sheaf cohomology along with a couple of cochain complexes — Roos and cellular — which honestly compute it over general posets and cell posets respectively. We would like to study the Laplacians and heat diffusion associated to these complexes. To do so, we need inner products. 

\begin{definition}
    A \textbf{Euclidean sheaf} on a poset $S$ is a sheaf of real vector spaces $F$ on $S$ along with a choice of inner product $\langle -,-\rangle_s$ on each stalk $F(s)$, $s \in S$. 
\end{definition}

Moving forward, we will implicitly assume that our sheaves are Euclidean. If $F$ is a Euclidean sheaf on a poset $S$, and  $\mathscr{F}:\mathsf{Shv}_{\mathbb{R}\mathsf{Vect}}(S) \to \mathsf{Cochain}(\mathbb{R}\mathsf{Vect})$ is a concrete functor honestly computing cohomology, then $(C^\bullet, d^\bullet):=\mathscr{F}(F)$ carries a canonical inner product in each grading $k$.\footnote{Indeed, each $C^k$ is already a direct sum of stalks since $\mathscr{F}$ is concrete; one just takes this direct sum to be orthogonal.}

\begin{definition}
    We call the $k$th Laplacian of $(C^{\bullet}, d^{\bullet})$ the \textbf{$k$th sheaf Laplacian} of $F$ and denote it by $\Delta^{k}_{F}$. We call $\Delta^{0}_{F}$ the \textbf{vanilla sheaf Laplacian}; sometimes it is denoted $L_{F}$. 
\end{definition}

Note that the sheaf Laplacian definition depends both on the choice of functor $\mathscr{F}$ and Euclidean structure $\langle -,- \rangle_{s \in S}$ on $F$. Of course, the kernels are all canonically isomorphic by Theorem~\ref{thm:discrete-Hodge-Theorem}, since they're all isomorphic to the same cohomology.

\begin{definition}
\label{def:sheaf-diff-layer}
Let $S$ be a poset and $\mathscr{F}:\mathsf{Shv}_{\mathbb{R}\mathsf{Vect}}(S) \to \mathsf{Cochain}(\mathbb{R}\mathsf{Vect})$ be a choice of concrete functor honestly computing cohomology. Let $F$ be a sheaf on $S$, and write $(C^{\bullet}, d^{\bullet}):=\mathscr{F}(F)$. For $\eta > 0$ a hyperparameter, the generator of the discrete heat diffusion
\begin{align}
\label{eqn:discrete-heat-diffusion-general}
\operatorname{sd}_{F}:C^{k}(S; D) \to C^{k}(S; F), \ \ \operatorname{sd}_{F}(x)=x-2\eta \Delta^{k}x
\end{align}
is called the ($k$th-order) \textbf{sheaf diffusion}, or \textbf{sheaf diffusion layer}, determined by $F$. 

\end{definition}

When $(S, \leq)$ represents a graph, the operator $\operatorname{sd}_{D}$ may be seen as a standard graph message passing update (a linear one).\footnote{The argument that Kipf-Welling GCNs (strictly speaking, Kipf-Weilling GCNs sans nonlinearity) \cite{kipf2016semi} are indeed MPNNs (Example~\ref{ex:kipf-welling-GCN}) essentially carries over unmodified to demonstrate this.}



\begin{example}[Sheaf diffusion on binary relations, e.g. hypergraphs]
\label{ex:sheaf-diffusion-on-hypergraphs}
Let $R \subset A \times B$ be a binary relation (e.g. a hypergraph), considered as a poset $S$ on $A \coprod B$ where $a<b$ iff $(a,b)\text{ (notation: }ab) \in R$.

A diagram (sheaf) $F$ on $S$ valued in $\mathbb{R}\mathsf{Vect}$ amounts merely to the assignment of a real vector space $F(a)$ to each $a \in A$, a real vector space $F(b)$ to each $b \in B$, and a linear map $F_{ab}:F(a) \to F(b)$ to each $(a,b) \in R$. (Since $\operatorname{dim }S=1$, $F$ is vacuously functorial, so these maps are not required to satisfy any compositionality requirements.)

In this case, the Roos complex $C^{\bullet}_{\operatorname{Roos}}(S;F)$ is as follows. The nontrivial cochain spaces are
\begin{align}
C^{0}_{\operatorname{Roos}}(S; F)&= \bigoplus_{a \in A}F(a) \oplus \bigoplus_{b \in B}F(b), \\
C^{1}_{\operatorname{Roos}}(S; F)&=\bigoplus_{ab \in R}F(b).
\end{align}
The codifferential $d_{\operatorname{Roos}}:C^{0}_{\operatorname{Roos}}(S; F) \to C^{1}_{\operatorname{Roos}}(S; F)$ is given by ($\boldsymbol x \in C^{0}_{\operatorname{Roos}}(S; F)$)
\begin{equation}
(d_{\operatorname{Roos}}\boldsymbol x)_{ab}=x_{b}-F(a < b)x_{a}.
\end{equation}
The Roos Laplacian $\Delta^{0}_F$ acts as
\begin{align}
\Delta^{0}_{F}&: \bigoplus_{a \in A} F(a) \oplus \bigoplus_{b \in B} F(b )  \to \bigoplus_{a \in A} F(a) \oplus \bigoplus_{b \in B} F(b ), \\
(\Delta^{0}_{F} \boldsymbol  x)_{a} &=\sum_{a:a < b} F(a < b)^{\top} \big( F(a< b) x_{a} - x_{b} \big), \\
(\Delta_{F}^{0})_{b} &=  \sum_{b:a < b} \big(x_{b}-F(a < b)x_{a} \big).
\end{align}
The space of global sections $\Gamma(S; F)$ equals $\operatorname{ker }d_{\operatorname{Roos}}$, and thus manifestly consists of the cochains $\boldsymbol x=(\boldsymbol x_{a}, \boldsymbol x_{b})\in C^{0}_{\operatorname{Roos}}(S;F)$ satisfying $F(a < b)x_{a}=x_{b}$ for all incidences $a<b$. By Theorem~\ref{thm:discrete-Hodge-Theorem}, $\Gamma(S;F)$ coincides with $\operatorname{ker }\Delta^{0}_{F}$. The sheaf diffusion minimizes the Dirichlet energy form $Q_{0}(\boldsymbol x)=\langle \boldsymbol x, \Delta^{0}_{F} \boldsymbol x \rangle$ given by
\begin{equation}
\label{eqn:roos-hypergraph-dirichlet-energy}
Q_{0}(\boldsymbol  x)=\langle d^{0}\boldsymbol x, d^{0} \boldsymbol x \rangle=\sum_{ab \in R} \langle (d^{0}\boldsymbol x)_{ab} , (d^{0} \boldsymbol x)_{ab} \rangle=\sum_{ab \in R} \|x_{b} - F(a < b)x_{a}  \|^{2}.
\end{equation}

    \begin{remark}[Relation to Duta et al.~\cite{duta2023sheaf}]

With a view toward directly generalizing the cellular sheaf Laplacian $\boldsymbol L_F: \bigoplus_{v \in V}F(v) \to \bigoplus_{v \in V}F(v)$ on a graph to the hypergraph setting, Duta et al. proposed a standalone Laplacian for a sheaf $F$ supported on the hypergraph $S$ as
\begin{align}
\boldsymbol  L_F&: \bigoplus_{a \in A}F(a) \to \bigoplus_{a \in A}F(a) \\
\boldsymbol  L_{F} (\boldsymbol  x)_{a} & = \sum_{b: a < b  } \frac{1}{|\{a : a < b\}|} F(a<b)^{\top} \big( \sum_{a' < b, a' \neq a} F(a < b) x_{a} - F(a' < b)x_{a'} \big).
\end{align}

Since this operator is not explicitly derived as the Laplacian of a Hilbertian cochcain complex, it does not strictly follow from the considerations in Section~\ref{sec:combinatorial-hodge-heat-diff} that performing linear diffusion with it monotonically minimizes the form
\begin{equation}
\label{eqn:duta-hypergraph-dirichlet-energy}
\langle \boldsymbol  x, \boldsymbol  L \boldsymbol  x \rangle =\sum_{b} \frac{1}{|\{ a : a < b \}|} \sum_{a,a'<b} \|F(a' < b)x_{a'} - F(a < b)x_{a}\|^{2}.
\end{equation}
Proposition 1 in Duta et al.~\cite{duta2023sheaf} shows, however, that this \textit{is} true when $\boldsymbol L$ is suitably normalized. Ayzenberg et al. show~\cite{ayzenberg2025sheaf} that, under near-trivial assumptions, this form is equivalent to the Roos Dirichlet energy~\ref{eqn:roos-hypergraph-dirichlet-energy}. A sufficient condition to be assured that $\boldsymbol L$ monotonically minimizes~\ref{eqn:duta-hypergraph-dirichlet-energy} without normalization (up to the aforementioned near-trivial assumptions), then, would be to prove that $\operatorname{ker } \boldsymbol L \cong \Gamma(S; F)$. This is true, though it also needs to be `proven manually':

\begin{theorem}
\label{thm:duta-laplacian-kernel-and-global-sections}
    The kernel of the sheaf hypergraph Laplacian proposed in Duta et al. is naturally isomorphic to the space $\Gamma(S;F)$ of global sections of $F$.
\end{theorem}

\begin{proof}
    
  If $\boldsymbol L_{F}\boldsymbol x=0$, then
\begin{equation}
0=\langle \boldsymbol  x, \boldsymbol  L\boldsymbol  x \rangle = \sum_{b} \frac{1}{|\{ a : a < b \}|} \sum_{a,a'<b} \|F(a' < b)x_{a'} - F(a < b)x_{a}\|^{2}.
\end{equation}
Each term on the right sum is nonnegative, thus zero, hence
\[
\underbrace{ F(a' < b)x_{a'}=F(a < b)x_{a} }_{ :=y_{b} }
\]
for all incidences $a<b$ in the hypergraph. Denote by $y_{b} \in F(b)$ this canonical shared value. We claim the map
\begin{align}
\operatorname{ker } \boldsymbol  L_{F} & \to \Gamma(S; F) \subset \bigoplus_{a \in A} F(a) \oplus \bigoplus_{b \in B} F(b) \\
\boldsymbol  x  & \mapsto \big( (x_{a})_{a \in A}, (y_{b})_{b \in B} \big)
\end{align}
is an isomorphism. Indeed, it is: the inverse is given by the restriction of the coordinate projection $\bigoplus_{a \in A} F(a) \oplus \bigoplus_{b \in B}F(b) \to \bigoplus_{a \in A}F(a)$ to $\Gamma(S; F)$.
\end{proof}

    \end{remark}
\end{example}

 From now on, we will mostly be concerned with the case where $S=G$ is a graph (or perhaps a hypergraph), as this is where the majority of literature lives. A (cellular) sheaf on a graph is called a \textbf{network sheaf}. A popular intuition for these sheaves, to which we will return in the coming examples, may be found in opinion dynamics~\cite{hansen2021opinion}. Viewing $G$ as a social network, a point in the stalk over person $v$ represents a linear combination of `basis opinions' held by that person, while an edge $e$ between people $u,v$ represents a `discourse space' where opinions may be expressed and compared. A given restriction map $F_{v \leq e}$ determines exactly \textit{how} $v$ chooses to express their opinion to $u$. For instance, perhaps $F_{v \leq e}=1$ when $v$ is being honest to $u$, $F_{v \leq e}=5$ when $v$ is exaggerating to $u$, and $F_{v \leq e}=-1$ when $v$ is lying to $u$. A global section of $G$ represents a `compatible distribution of \textit{expressed} opinions'. Under this model, the content of Section~\ref{sec:combinatorial-hodge-heat-diff} implies that social dynamics gravitate toward public consensus, though private opinions may be maintained. 

Some familiar faces arise among the (cellular) Laplacians of network sheaves.

\begin{example}[`Honest sheaves']
\label{ex:honest-sheaves}
The constant sheaf $\underline{\mathbb{R}}$ on $G$ has as cellular Laplacian the (unweighted) graph Laplacian. More generally, one can define a one-dimensional sheaf $F$ where the restriction maps entering an edge-stalk are still symmetric ($F_{v \leq e}=F_{u \leq e}$) but given as positive scalars $a_e=\sqrt{w_{uv}}=\sqrt{w_{vu}}$. This is exactly the sheaf explored in Motivation~\ref{sec:intuition-oversmoothing-heterophily}; its Laplacian is the (weighted) graph Laplacian. From an opinion dynamics perspective, such sheaves correspond to `honest' interactions concerning single opinion. The global sections, pictured in Figure~\ref{fig:motivation-graph-global-sections-oversmoothing-reprise}, reflect this: node representations become identical in the linear diffusion limit, leading to oversmoothing — private opinions homogenize alongside the public consensus. We may denote this class of one-dimensional sheaves by $\mathscr{F}^1_{\operatorname{sym}}$.
\begin{figure}[H]
\begin{center}
\begin{tikzpicture}[baseline]
    \path[use as bounding box] (-5.5,-1.5) rectangle (5.5,3.5);

    \node [circle, draw=tancolor, fill=tancolor!20, minimum size=0.2cm] (u) at (-3,0) {};
    \node [circle, draw=tancolor, fill=tancolor!20, minimum size=0.2cm] (v) at (3,0) {};

    \draw [black, thick] (u) -- (v);
    \draw [edgegreen, opacity=0.3, line width=3pt] (u) -- (v);

    \draw [opacity=0.2] (-3,0) -- (-5,1);
    \draw [opacity=0.2] (-3,0) -- (-5,0);
    \draw [opacity=0.2] (-3,0) -- (-5,-1);

    \draw [opacity=0.2] (3,0) -- (5,1);
    \draw [opacity=0.2] (3,0) -- (5,0);
    \draw [opacity=0.2] (3,0) -- (5,-1);

    \node [tancolor] at (-3,-0.5) {$u$};
    \node [tancolor] at (3,-0.5) {$v$};
    \node [edgegreen] at (0,-0.3) {$e$};

    \node [draw=tancolor, minimum width=0.2cm, minimum height=1.2cm] (squareu) at (-3,2.5) {};
    \draw [tancolor, opacity=0.2] ($(u)+(0,0.2)$) -- ($(squareu.south)+(0,-0.1)$);
    \node [left=0.1cm of squareu] {\textcolor{tancolor}{$\mathbb{R}$}};
    \draw [dotted] ($(squareu.north)+(0,-0.1)$) -- ($(squareu.south)+(0,0.1)$);
    \node at (-3,2.5) {$-$}; 

    \node [draw=edgegreen, minimum width=0.2cm, minimum height=1.2cm] (squaree) at (0,1.5) {};
    
    \draw [edgegreen, opacity=0.2] ($(0,0)+(0,0.2)$) -- ($(squaree.south)+(0,-0.1)$);
    \node [above=0.1cm of squaree] {\textcolor{edgegreen}{$\mathbb{R}$}};
    \draw [dotted] ($(squaree.north)+(0,-0.1)$) -- ($(squaree.south)+(0,0.1)$);
    \node at (0,1.5) {$-$};

    \node [draw=tancolor, minimum width=0.2cm, minimum height=1.2cm] (squarev) at (3,2.5) {};
    \draw [tancolor, opacity=0.2] ($(v)+(0,0.2)$) -- ($(squarev.south)+(0,-0.1)$);
    \node [right=0.1cm of squarev] {\textcolor{tancolor}{$\mathbb{R}$}};
    \draw [dotted] ($(squarev.north)+(0,-0.1)$) -- ($(squarev.south)+(0,0.1)$);
    \node at (3,2.5) {$-$};

    \draw [->]
        ($(squareu.east)+(0.1,0)$) --
        node[below=0.2cm, midway] {$\begin{bmatrix}
            a_{\textcolor{edgegreen}{e}}
        \end{bmatrix}$}
        ($(squaree.west)+(-0.1,0)$);

    \draw [->]
        ($(squarev.west)+(-0.1,0)$) --
        node[below=0.2cm, midway] {$\begin{bmatrix}
            a_{\textcolor{edgegreen}{e}}
        \end{bmatrix}$}
        ($(squaree.east)+(0.1,0)$);

    \node at (-3,2.8) {\textcolor{blue}{$\bullet$}}; 
    \node at (3,2.8) {\textcolor{pink}{$\bullet$}}; 
    \node at (-0.05,1.6) {\textcolor{blue}{$\bullet$}}; 
    \node at (0.05,1.6) {\textcolor{pink}{$\bullet$}}; 

\end{tikzpicture}
\caption{A global section of an `honest sheaf', i.e., a graph Laplacian (Example~\ref{ex:honest-sheaves}).}
\label{fig:motivation-graph-global-sections-oversmoothing-reprise}
\end{center}
\end{figure}
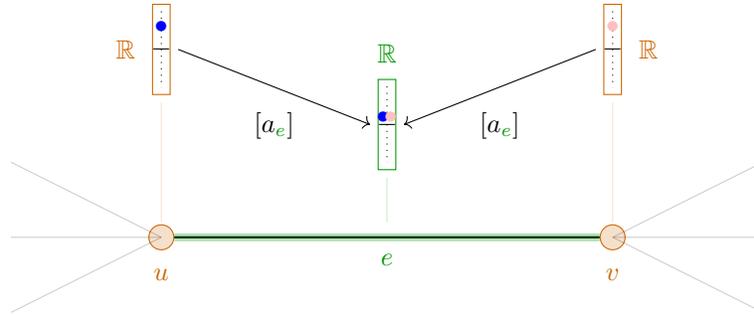

\end{example}

\begin{example}[`Harmonic tension sheaves']
\label{ex:harmonic-tension-sheaves}
    In Motivation~\ref{sec:intuition-oversmoothing-heterophily}, we conjectured that one way to mitigate the oversmoothing depicted in Figure~\ref{fig:motivation-graph-global-sections-oversmoothing-reprise} would be to have antisymmetric restriction maps (the proof that this indeed works will be in Section~\ref{sec:linear-separation-power-sheaf-diffusion}). See Figure~\ref{fig:lying-sheaf-reprise}. From an opinion dynamics perspective, such a `lying' or `polarizing' sheaf captures two-way interactions in which one party is honest and the other is deceptive. We remark that this class of sheaf Laplacians appear to subsume the collection of \textit{signed ratio Laplacians} discussed e.g. in \cite{kunegis2010spectral}.

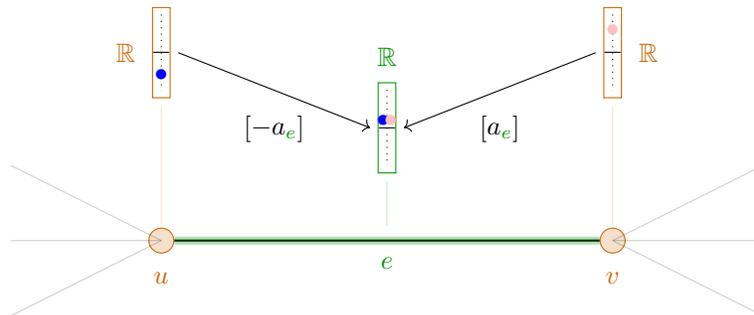
\begin{figure}[H]

\begin{center}
    
\begin{tikzpicture}[baseline]
    \path[use as bounding box] (-5.5,-1.5) rectangle (5.5,3.5);

    \node [circle, draw=tancolor, fill=tancolor!20, minimum size=0.2cm] (u) at (-3,0) {};
    \node [circle, draw=tancolor, fill=tancolor!20, minimum size=0.2cm] (v) at (3,0) {};

    \draw [black, thick] (u) -- (v);
    \draw [edgegreen, opacity=0.3, line width=3pt] (u) -- (v);

    \draw [opacity=0.2] (-3,0) -- (-5,1);
    \draw [opacity=0.2] (-3,0) -- (-5,0);
    \draw [opacity=0.2] (-3,0) -- (-5,-1);

    \draw [opacity=0.2] (3,0) -- (5,1);
    \draw [opacity=0.2] (3,0) -- (5,0);
    \draw [opacity=0.2] (3,0) -- (5,-1);

    \node [tancolor] at (-3,-0.5) {$u$};
    \node [tancolor] at (3,-0.5) {$v$};
    \node [edgegreen] at (0,-0.3) {$e$};

    \node [draw=tancolor, minimum width=0.2cm, minimum height=1.2cm] (squareu) at (-3,2.5) {};
    \draw [tancolor, opacity=0.2] ($(u)+(0,0.2)$) -- ($(squareu.south)+(0,-0.1)$);
    \node [left=0.1cm of squareu] {\textcolor{tancolor}{$\mathbb{R}$}};
    \draw [dotted] ($(squareu.north)+(0,-0.1)$) -- ($(squareu.south)+(0,0.1)$);
    \node at (-3,2.5) {$-$};

    \node [draw=edgegreen, minimum width=0.2cm, minimum height=1.2cm] (squaree) at (0,1.5) {};
    \draw [edgegreen, opacity=0.2] ($(0,0)+(0,0.2)$) -- ($(squaree.south)+(0,-0.1)$);
    \node [above=0.1cm of squaree] {\textcolor{edgegreen}{$\mathbb{R}$}};
    \draw [dotted] ($(squaree.north)+(0,-0.1)$) -- ($(squaree.south)+(0,0.1)$);
    \node at (0,1.5) {$-$};

    \node [draw=tancolor, minimum width=0.2cm, minimum height=1.2cm] (squarev) at (3,2.5) {};
    \draw [tancolor, opacity=0.2] ($(v)+(0,0.2)$) -- ($(squarev.south)+(0,-0.1)$);
    \node [right=0.1cm of squarev] {\textcolor{tancolor}{$\mathbb{R}$}};
    \draw [dotted] ($(squarev.north)+(0,-0.1)$) -- ($(squarev.south)+(0,0.1)$);
    \node at (3,2.5) {$-$}; 
    \draw [->]
        ($(squareu.east)+(0.1,0)$) --
        node[below=0.2cm, midway] {$\begin{bmatrix}
            -a_{\textcolor{edgegreen}{e}}
        \end{bmatrix}$}
        ($(squaree.west)+(-0.1,0)$);

    \draw [->]
        ($(squarev.west)+(-0.1,0)$) --
        node[below=0.2cm, midway] {$\begin{bmatrix}
            a_{\textcolor{edgegreen}{e}}
        \end{bmatrix}$}
        ($(squaree.east)+(0.1,0)$);

    \node at (3,2.8) {\textcolor{pink}{$\bullet$}}; 
    \node at (-0.05,1.6) {\textcolor{blue}{$\bullet$}}; 
    \node at (0.05,1.6) {\textcolor{pink}{$\bullet$}};
    \node at (-3,2.2) {\textcolor{blue}{$\bullet$}};  

\end{tikzpicture}

\caption{A global section of a `lying sheaf' (Example~\ref{ex:harmonic-tension-sheaves}). }
\label{fig:lying-sheaf-reprise}

\end{center}

\end{figure}

\end{example}

\begin{example}[Connection sheaves]
A sheaf $F$ on a poset $S$ is called a \textbf{connection sheaf} if its restriction maps are all isomorphisms. If $G$ is a subgroup of $\operatorname{Aut}(V)$ for some object $V$ of $\mathsf{D}$, a \textbf{$G$-connection sheaf} is a sheaf $F$ on $S$ whose restriction maps all belong to $G$. Connection sheaves are also called \textbf{discrete vector bundles} and \textbf{local systems}, and could likely form the content of an essay in their own right. Warning: while constraints of the medium forbid a faithful justification of these terminologies here, the term `discrete vector bundle' does \textit{not} arise because connection sheaves are equivalent (under Theorem~\ref{thm:sheaves-diagrams-category-equivalence}) to (sheaves of sections of) vector bundles on $X_S$. That is not generally true. 

The most important type of connection sheaf is that where $G=\operatorname{O}(d)$ (the sheaves in Example~\ref{ex:harmonic-tension-sheaves} give an example). The Laplacians of such sheaves correspond to the \textbf{connection Laplacians} studied e.g. in~\cite{singer2012vector}.

\begin{figure}[H]

\begin{center}

\begin{tikzpicture}[baseline]
    
    \path[use as bounding box] (-5.5,-1.5) rectangle (5.5,3.5);
    
    \node [circle, draw=tancolor, fill=tancolor!20, minimum size=0.2cm] (u) at (-3,0) {};
    \node [circle, draw=tancolor, fill=tancolor!20, minimum size=0.2cm] (v) at (3,0) {};
    
    \draw [black, thick] (u) -- (v);
    \draw [edgegreen, opacity=0.3, line width=3pt] (u) -- (v);
    
    \draw [opacity=0.2] (-3,0) -- (-5,1);
    \draw [opacity=0.2] (-3,0) -- (-5,0);
    \draw [opacity=0.2] (-3,0) -- (-5,-1);
    
    \draw [opacity=0.2] (3,0) -- (5,1);
    \draw [opacity=0.2] (3,0) -- (5,0);
    \draw [opacity=0.2] (3,0) -- (5,-1);
    
    \node [tancolor] at (-3,-0.5) {$u$};
    \node [tancolor] at (3,-0.5) {$v$};
    \node [edgegreen] at (0,-0.3) {$e$};
    
    \node [draw=tancolor, minimum width=0.2cm, minimum height=1.2cm] (squareu) at (-3,2.5) {};
    \draw [tancolor, opacity=0.2] ($(u)+(0,0.2)$) -- ($(squareu.south)+(0,-0.1)$);
    \draw [dotted] ($(squareu.north)+(0,-0.1)$) -- ($(squareu.south)+(0,0.1)$);
    \node at (-3,2.5) {$-$}; 
    
    \node [draw=edgegreen, minimum width=0.2cm, minimum height=1.2cm] (squaree) at (0,1.5) {};
    \draw [edgegreen, opacity=0.2] ($(0,0)+(0,0.2)$) -- ($(squaree.south)+(0,-0.1)$);
    \draw [dotted] ($(squaree.north)+(0,-0.1)$) -- ($(squaree.south)+(0,0.1)$);
    \node at (0,1.5) {$-$};
    
    \node [draw=tancolor, minimum width=0.2cm, minimum height=1.2cm] (squarev) at (3,2.5) {};
    \draw [tancolor, opacity=0.2] ($(v)+(0,0.2)$) -- ($(squarev.south)+(0,-0.1)$);
    \draw [dotted] ($(squarev.north)+(0,-0.1)$) -- ($(squarev.south)+(0,0.1)$);
    \node at (3,2.5) {$-$};

    \draw [->] 
        ($(squareu.east)+(0.1,0)$) -- 
        node[below=0.1cm, midway] {${F}({\textcolor{tancolor}{u} \leq \textcolor{edgegreen}{e}})$} 
        ($(squaree.west)+(-0.1,0)$);

    \draw [<-]
        ($(squareu.east)+(0.1,.15)$) -- 
        node[above=0.1cm, midway] {${F}({\textcolor{tancolor}{u} \leq \textcolor{edgegreen}{e}})^{\top}$} 
        ($(squaree.west)+(-0.1,0.15)$);

    \draw [->] 
        ($(squarev.west)+(-0.1,0)$) -- 
        node[below=0.1cm, midway] {${F}({\textcolor{tancolor}{v} \leq \textcolor{edgegreen}{e}})$}
        ($(squaree.east)+(0.1,0)$);

    \draw [<-] 
        ($(squarev.west)+(-0.1,.15)$) -- 
        node[above=0.1cm, midway] {${F}({\textcolor{tancolor}{v} \leq \textcolor{edgegreen}{e}})^{\top}$}
        ($(squaree.east)+(0.1,.15)$);

\end{tikzpicture}

\caption{An $\operatorname{O}(1)$ bundle on a graph. The corresponding sheaf Laplacian aligns with the \textit{connection Laplacian} studied e.g. in~\cite{singer2012vector}.}

\end{center}

\end{figure}
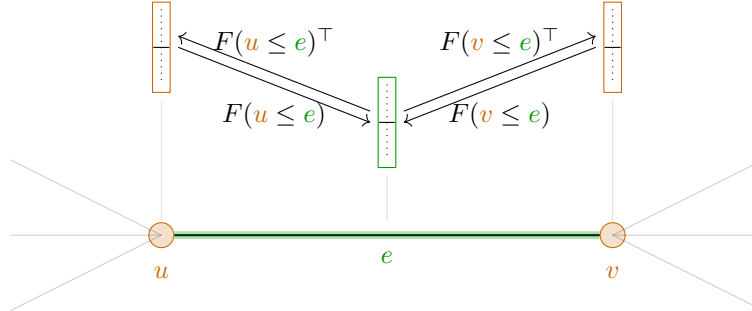

\end{example}


\section{How Powerful is Sheaf Diffusion?}
\label{sec:linear-separation-power-sheaf-diffusion}

We are now in a position to investigate the 
expressivity of sheaf diffusion. Our main novel results will be: \begin{enumerate}
    \item Diffusion with discrete vector bundles can asymptotically solve any rank-$0$ classification task on a poset (such as a cell complex or hypergraph). This is a nontrivial extension of the proof in~\cite{bodnar_neural_2023} for graph special case. 
    \item The class of discrete vector bundles \textit{is not} expressive enough to solve higher-rank classification tasks.
    \item Nevertheless, any rank-$k$ classification task can be solved by performing diffusion with the correct sheaf. Our proof rests on a lemma which roughly states that one can always construct a sheaf with a specified gradient space, and which may be of independent interest.
\end{enumerate}
Beyond these novel results, we will also review a hierarchy of results from Bodnar et al.\cite{bodnar_neural_2023} (Table~\ref{table:linear-sheaf-diffusion-catalog}) concerning the expressive power of \textit{graph} sheaf diffusion, leveraging the machinery developed in this essay thus far to provide novel proofs for many.


\subsection{Sheaf Diffusion is Expressive for Order $k=0$}


We prove that any rank-$0$ classification task on a poset $S$ can be solved by performing $0$th order heat diffusion with a suitable sheaf. We begin by investigating the case where $S=G$ is a graph. Table~\ref{table:linear-sheaf-diffusion-catalog} catalogs the hierarchy of linear separation results established. It says, among other things, that normalized graph Laplacian diffusion has two shortcomings: it can fail in non-homophilic settings, and it can fail when there are more than two node classes. Performing diffusion with the appropriate sheaf overcomes both of these limitations.




\begin{table}[h]
  \centering
  \setlength{\tabcolsep}{6pt}
  \renewcommand{\arraystretch}{1.35}

  \begin{tabular}{@{} l c c c c c @{}}
    \toprule
    & \makecell{Unnormalized\\graph\\Laplacian}
    & \makecell{Normalized\\graph\\Laplacian\\($\mathscr{F}^1_{
\operatorname{sym}
    }$)}
    & \makecell{Same–sign\\asymmetric\\restriction maps\\($\mathscr{F}^1_+$)}
    & \makecell{Lying\\$1$-sheaf}
    & \makecell{Lying\\$\ell$-sheaf} \\ \midrule
    \makecell[l]{$G$ two components\\$=$ classes}
      & \cmark & \cmark & \cmark & \cmark & \cmark \\ \midrule

    \makecell[l]{$G$ homophilic,\\connected, two classes}
      & \makecell{\xmark\\[-.3em] {\scriptsize Motivation~\ref{sec:intuition-oversmoothing-heterophily}}}%
      & \makecell{\cmark\\[-.3em] {\scriptsize Theorem~\ref{thm:linear-sheaf-diffusion-result-1}}}%
      & \cmark & \cmark & \cmark \\ \midrule

    \makecell[l]{$G$ connected, two-class\\(homo., hetero., anything in-between)}
      & \xmark
      & \makecell{\xmark\\[-.3em] {\scriptsize Theorem~\ref{thm:linear-sheaf-diffusion-result-1}}}
      & \makecell{\xmark\\[-.3em] {\scriptsize Remark \ref{rmk:asymmetric-positive-doesnt-work}}}%
      & \makecell{\cmark\\[-.3em] {\scriptsize Theorem~\ref{thm:2-way-linear-separability}}}%
      & \cmark \\ \midrule

    \makecell[l]{$G$ connected, $\ell$-class}
      & \xmark
      & \xmark
      & \xmark
      & \makecell{\xmark\\[-.3em] {\scriptsize Theorem~\ref{thm:1d-sheaf-cant-separate-3-classes}}}%
      & \makecell{\cmark\\[-.3em] {\scriptsize Theorem~\ref{thm:k-way-linear-separability}}}\\
    \bottomrule
  \end{tabular}

  \caption{A hierarchy of expressivity results for sheaf diffusion on graphs. }
\label{table:linear-sheaf-diffusion-catalog}
\end{table}

Let us begin by revisiting the Motivation~\ref{sec:intuition-oversmoothing-heterophily}: suppose $G=(V,E)$ is a connected graph with two classes $A,B$, as we had there. Recall that, in this situation, placing a (positive) weighting on the edges of $G$ is equivalent to placing a sheaf $F$ on $G$ whose stalks are all $\mathbb{R}$ and whose restriction maps are `symmetric': $F_{v \leq e}=\alpha_{e}=F_{u \leq e}$ for all $(u,v) \in E$, where $\alpha_{e}>0$; one determines the other. Recall (Example~\ref{ex:honest-sheaves}) that we have denoted by $\mathscr{F}^{1}_{\operatorname{sym}}$ this collection of sheaves. In this picture, the constant sheaf $\underline{\mathbb{R}} \in \mathscr{F}^{1}_{\operatorname{sym}}$ corresponds to the case where $G$ is unweighted. The collections of Laplacians are the same.\\

In Motivation~\ref{sec:intuition-oversmoothing-heterophily}, we showed that $\mathscr{F}^{1}_{\operatorname{sym}}$ has no separation power over $G$: since $G$ is connected, there must be one or more edges between nodes in different classes, and in the diffusion limit with the \textit{unnormalized} Laplacian these edges will homogenize the node representations between the two classes. In practice, of course, it is common to normalize graph Laplacians. Therefore, it is sensible to wonder: does oversmoothing remain inevitable even when performing linear diffusion with a normalized Laplacian? \\

A necessary condition, at least when $A$ and $B$ are of the same size, is that at least one node of $G$ has nonzero homophily. A sufficient condition is that \textit{each} node\footnote{Theorem~\ref{thm:linear-sheaf-diffusion-result-1} actually only assumes that each node in one of the classes (say, $A$) has nonzero homophily.} has nonzero homophily. This is the content of the following theorem.

\begin{theorem}
\label{thm:linear-sheaf-diffusion-result-1}
    Let $G=(V,E)$ be a connected graph with two classes $A,B \subset V$. If each node in $A$ has a neighbor also in $A$, then $\mathscr{F}^{1}_{\operatorname{sym}}$ can linearly separate the classes of $G$ in the normalized diffusion limit for almost all initial conditions. On the other hand, if no node in $G$ has a neighbor in its class (i.e., $G$ is bipartite) and $|A|=|B|$, then $\mathscr{F}^{1}_{\operatorname{sym}}$ cannot linearly separate the classes of $G$ in the normalized diffusion limit for any choice of initial condition. 
\end{theorem}

\begin{proof}
    Suppose $G$ is such that for each $v \in A$, there exists $u \in A$ with $(u,v) \in E$. Define a sheaf $F$ whose stalks are all $\mathbb{R}$ and for an edge $(u,v) \in E$,  $$F_{v \leq e}=F_{u \leq e}=\begin{cases}
\sqrt{ N } &  u,v \in A\\
1 & \operatorname{otherwise,}
\end{cases}$$
where $N>0$ is a constant to be chosen later. Let $h_{v}$ denote the number of neighbors of node $v$ in the same class as $v$. By the homophily assumption on $A$, $h_{v} \geq 1$, and $\operatorname{ker }\Delta_{F}$ is spanned by $\boldsymbol y =(y_{v})_{v \in V}$ given by $$y_{v}=\begin{cases}
 \sqrt{ \operatorname{deg }v+ h_{v}(N-1) } & v \in A \\
\sqrt{ \operatorname{deg }v } & v \in B.
\end{cases}$$
Put $\widetilde{\boldsymbol y}=\boldsymbol y/\|\boldsymbol y\|$. In the diffusion limit, the node representations converge to  $\langle \boldsymbol  x(0), \widetilde{\boldsymbol  y}\rangle \widetilde{\boldsymbol  y} .$ We can assume $\boldsymbol x(0) \notin (\operatorname{ker}\boldsymbol L_{F})^{\perp}$, a set of measure zero in $\mathbb{R}^{| V|}$, and WLOG that $\langle \boldsymbol x(0), \widetilde{\boldsymbol y} \rangle > 0$. Then, for $N$ chosen large enough, $y_{v} >y_{u}$ whenever $v \in A$ and $u \in B$. \\

Now assume $G=(A,B,E)$ is bipartite with $|A|=|B|$. Let $\boldsymbol x(0)$ be any initial condition. $G$ is connected; $\operatorname{ker }\Delta_{F}$ is spanned by $\boldsymbol  y=(y_{v})_{v \in V}$ given as  $$y_{v}=\sqrt{ \text{deg }v}=\sqrt{ \sum_{v \leq e} \|F_{v \leq e}\|^{2} }$$
where the first equality is viewing $F \in \mathscr{F}^{1}_{\text{sym}}$ as choice of weighting on $G$ and the latter equality is viewing $F$ as a sheaf. Up to scaling, the diffusion procedure converges to $\langle \boldsymbol x(0), \boldsymbol y \rangle \boldsymbol y$. Linear separability therefore occurs if and only if $\langle \boldsymbol x(0), \boldsymbol y \rangle \neq 0$ and (WLOG) $y_{v}< y_{u}$ whenever $v \in A$, $u \in B$, i.e., if and only if $\text{deg }v < \text{deg }u$ whenever $v \in A$ and $u \in B$. Since $|A|=|B|$, this condition implies  $$\sum_{e \in E} \|F_{v \leq e}\|^{2} < \sum_{e \in E} \|F_{u \leq e}\|^{2},$$which is impossible for $F \in \mathscr{F}^{1}_{\text{sym}}$. 
\end{proof}

\begin{remark}
\label{rmk:asymmetric-positive-doesnt-work}    
Theorem~\ref{thm:linear-sheaf-diffusion-result-1}, together with Motivation~\ref{sec:intuition-oversmoothing-heterophily}, show that — regardless of normalization — obtaining full separation power guarantees on a two-class connected graph requires one to `look beyond' $\mathscr{F}^{1}_{\operatorname{sym}}$ and the realm of graph Laplacians. Perhaps the most parsimonious step outside of $\mathscr{F}^{1}_{\operatorname{sym}}$ one could take is to permit relations which are asymmetric but still positive. This fails in general. For instance, the harmonic space of a sheaf $F$ on $G=(\{ v, u \}, \underbrace{ (v, u) }_{:= e })$ with $F(v)=F(u)=F(e)=\mathbb{R}$, $F_{v \leq e}=\alpha$, and $F_{u \leq e}=\beta$ ($\alpha, \beta > 0$) is spanned by $(\alpha \beta, \alpha \beta)$. Motivation~\ref{sec:intuition-oversmoothing-heterophily} closed with a conjecture that one should nevertheless not have to look too far outside $\mathscr{F}^{1}_{\operatorname{sym}}$, however: merely negating one of the restriction maps entering each edge (forming a `lying'/`harmonic tension sheaf' as in Example~\ref{ex:harmonic-tension-sheaves}) should be enough to gain full separation power. This conjecture is indeed true, and in fact holds for more general posets than graphs. For this discussion, we shall suggestively call rank-$0$ poset elements \textit{nodes} $v \in V$ and rank-$1$ poset elements \textit{edges} $e \in E$. 
\end{remark}

\begin{theorem}
\label{thm:2-way-linear-separability}
    Let $G$ be a connected poset with two rank-$0$ classes $A,B \subset V$. Consider a sheaf $F$ on $G$ satisfying $F(v)=F(e)=\mathbb{R}$ for all $v \in V, e \in E$ and $$F_{v \leq e} := \begin{cases}
 - 1 & v \in A \\
1 & v \in B 
\end{cases}$$for all incidences $v \leq e$. (The remaining stalks and restriction maps may be arbitrarily chosen, e.g. zero, as long as functoriality is preserved.) Then the diffusion~\ref{eqn:discrete-heat-diffusion-general} can linearly separate the classes of $G$ for almost all initial conditions. Thus, $\mathscr{F}^{1}$ has linear separation power over the collection $\mathcal{G}$ of posets with two rank-$0$ classes (such as all two-class graphs and hypergraphs). 
\end{theorem}

This result generalizes Proposition 12 in Bodnar et al.~\cite{bodnar_neural_2023}. 



\begin{proof}
    For notational convenience, set
\begin{align}
\operatorname{cl }v:=\begin{cases}
-1 & \text{node } v \text{ belongs to  class } A \\
1 & \text{node } v \text{ belongs to class } B,
\end{cases}
\end{align}
so that
\begin{align}
F_{v \leq e}(y_{v})=\operatorname{cl } v.
\end{align}

Per the discussions in Section~\ref{sec:combinatorial-hodge-heat-diff}, $\boldsymbol x(\infty)$ is given by the orthogonal projection of the initial condition $\boldsymbol x(0)$ onto the harmonic space $\operatorname{ker }\boldsymbol L_{F}$. We will show this space is one-dimensional, and in particular is spanned by the vector $\operatorname{sign}(V):=(\operatorname{cl }v)_{v \in V}$. To do so, it suffices to examine the space $\Gamma(G, F)$ of global sections, then transfer to $\operatorname{ker } \boldsymbol L_{F}$ via the canonical isomorphism (\ref{eqn:Hodge-thm-iso}).\footnote{\textit{Which} `canonical isomorphism', of course, depends on how the vanilla Laplacian $\boldsymbol L_{F}$ is constructed:
\begin{enumerate}
\item If $\boldsymbol L_{F}=\boldsymbol L_{F}^{\operatorname{Roos}}=(d^{0}_{\operatorname{Roos}})^{*}d^{0}_{\operatorname{Roos}}$ is the Laplacian of the Roos complex of $G$, then 
\begin{align}
\Gamma(G, F) \underset{(\text{\ref{eqn:Roos-kernel-is-global-sections}})}{=}
 H^{0}_{\operatorname{Roos}}(G ; F) \xrightarrow[\text{Theorem~\ref{thm:discrete-Hodge-Theorem}}]{\cong}\operatorname{ker }\boldsymbol  L_{F}
\end{align}
\item If $\boldsymbol L_{F}=\boldsymbol L_{F}^{\operatorname{cell}}=(d_{\operatorname{cell}}^{0})^{*} d_{\operatorname{cell}}^{0}$ is the cellular sheaf Laplacian of $G$, then 
\begin{align}
\Gamma(G, F) \xrightarrow[\text{(\ref{eqn:zeroth-cellular-cohomology-global-sections}})]{\cong} H^{0}_{\operatorname{cell}}(G ; F) \xrightarrow[\text{Theorem~\ref{thm:discrete-Hodge-Theorem}}]{\cong}\operatorname{ker }\boldsymbol  L_{F}.
\end{align}
\item If $\boldsymbol L_{F}=\boldsymbol L_{F}^{\operatorname{Duta}}$ is the sheaf hypergraph Laplacian of Duta et al.~\cite{duta2023sheaf}, then 
\begin{align}
\Gamma(G, F) \xrightarrow[\text{Theorem~\ref{thm:duta-laplacian-kernel-and-global-sections}}]{\cong}\operatorname{ker } \boldsymbol  L_{F}.
\end{align}
\end{enumerate}

\noindent We refer to elements of $H^{0}_{\operatorname{cell}}$ and of $\operatorname{ker }\boldsymbol L_{F}$ freely as `global sections' in light of these identifications, even though, unlike for $(1)$, these are strictly speaking elements of $\mathbb{R}^{|V|}$ rather than $\mathbb{R}^{|V|} \oplus \mathbb{R}^{|E|}$. 

}


The coherence condition determining a global section $\boldsymbol y=(y_{v})_{v \in V} \in \mathbb{R}^{|V|}$ is
\begin{align}
\label{eqn:proof-of-temp}
F_{v \leq e}(y_{v})=F_{u \leq e}(y_{u})
\end{align}
for all $(u,v) \in E$. Recalling that $F_{v \leq e}(y_{v})=\operatorname{cl } v$, the condition~\ref{eqn:proof-of-temp} becomes $(\operatorname{cl } v )y_{v} = (\operatorname{cl }u)y_{u}$, i.e., $\boldsymbol y$ is a global section iff
\begin{align}
y_{v}=(\operatorname{cl } v)(\operatorname{cl }u)y_{u} \tag{$*$}
\end{align}
for all edges $(u,v) \in E$.

\emph{Subclaim. Any global section $\boldsymbol y=(y_{v})_{ v \in V}$ necessarily has the form $\boldsymbol 1_{A}-\boldsymbol 1_{B}=\operatorname{sign}(V)$ up to a real scaling, where $\boldsymbol 1_{A}$ and $\boldsymbol 1_{B}$ are the indicator vectors of their respective classes.}

\emph{Proof of subclaim.} Let $\boldsymbol y$ be a global section, and let $w_{0}, w$ be any pair of nodes. We want to show that $y_{w_{0}}=y_{w}$ if $\operatorname{cl }w_{0}=\operatorname{cl }w$, and $y_{w_{0}}=-y_{w}$ otherwise. Since $G$ is connected, there exists a path $w_{0} \to w_{1}\to\dots\to w_{t}=w$ from $w_{0}$ to $w$. We proceed by induction on its length $t$. The base case follows from {(${*}$)}. If the result holds for the path $w_{0} \to w_{1}\to \dots \to w_{t-1}$, i.e.,
\begin{align}
y_{w_{0}}=\begin{cases}
y_{w_{t-1}} & \operatorname{cl }w_{0}=\operatorname{cl }w_{t-1} \\
-y_{w_{t-1}} & \operatorname{cl }w_{0} \neq \operatorname{cl }w_{t-1}
\end{cases}
= (\operatorname{cl }{w_{0}})(\operatorname{cl }w_{t-1})y_{w_{t-1}},
\end{align}
then
\begin{align}
y_{w_{t}}&=(\operatorname{cl }w_{t-1}) (\operatorname{cl }w_{t})y_{w_{t-1}} \\
&= (\operatorname{cl }w_{t-1}) (\operatorname{cl }w_{t}) (\operatorname{cl }w_{t-1})(\operatorname{cl }w_{0})y_{w_{0}} \\
&=(\operatorname{cl }w_{t}) (\operatorname{cl }w_{0})y_{w_{0}},
\end{align}
as was required.

Note that $\boldsymbol 1_{A}-\boldsymbol 1_{B}$ is indeed a global section, so that the harmonic space $\Gamma(G, F)=\operatorname{span}(\boldsymbol 1_{A}- \boldsymbol 1_{B})$ is one-dimensional.

\emph{Finishing the main proof.} So now suppose $\boldsymbol x(0) \in \mathbb{R}^{| V|}$ is an initial condition. In the diffusion limit, $\boldsymbol x(t)$ converges to the orthogonal projection of $\boldsymbol x(0)$ onto the harmonic space $\operatorname{ker }\boldsymbol L_{F}$, which we have canonically identified with $\Gamma(G, F)$. Provided that $\boldsymbol x(0) \notin (\operatorname{ker}\boldsymbol L_{F})^{\perp}$, a set of measure zero in $\mathbb{R}^{| V|}$, we know that $\boldsymbol x(\infty)$ is a nonzero scaling of $\operatorname{sign}(V)$. The result follows.

\end{proof}

Thus far, we have only discussed the problem of two-way linear separability, both in motivational considerations and in theoretical results. In a sense, this is without much loss of generality, because a diffusion procedure on an direct sum of $k$ sheaves in general acts like $k$ independent diffusion procedures. 

\begin{theorem}
\label{thm:k-way-linear-separability}
    Let $G$ be a connected poset with $k$ rank-$0$ classes $A_{1},\dots,A_{k}$. Then the (orthogonal) direct sum  $\bigoplus_{i =1}^{k} F^{i}$ of sheaves $F^{i}$, where $F^{i}(v)=F^{i}(e)=\mathbb{R}$ for all $v \in V$, $e \in E$ and\footnote{This is the same 1-dimensional sheaf as in Theorem~\ref{thm:2-way-linear-separability}, with $B=V-A_i$ the complement of $A_i$. } $$F^{i}_{v \leq e}= \begin{cases}
-1 & v \in A_{i} \\
1 & v \notin A_{i}
\end{cases}$$ can linearly separate the classes of $G$ in the diffusion limit for almost any choice of initial condition. Thus, $\mathscr{F}^{k}_{\operatorname{diag}}$ has linear separation power over the collection $\mathcal{G}$ of connected graphs with $k$ classes.
\end{theorem}

\begin{proof}
    This follows from the fact that kernels and direct sums of linear maps commute. In particular, $\operatorname{ker}(\bigoplus_{i=1}^{k} \boldsymbol  L_{F^{i}}) = \bigoplus_{i=1}^{k} \operatorname{ker }\boldsymbol  L_{F^{i}}$: the harmonic space of $\bigoplus_{i=1}^{k} \boldsymbol L_{F^{i}}$ is the orthogonal direct sum of the harmonic spaces of each $\boldsymbol L_{F^{i}}$. 

With this in mind, Let $\boldsymbol x=(\boldsymbol x^{i})_{i=1}^{k}, \boldsymbol x^{i} \in \bigoplus_{v \in V}F^{i}(v)$ be a nondegenerate initial condition — i.e., one such that $\boldsymbol x^{i} \notin (\operatorname{ker}\boldsymbol L_{F^{i}})^{\perp}$. Note that the set of degenerate initial conditions forms a subdimensional, hence measure zero, linear subspace of $\mathbb{R}^{|V|k}$. 

The orthogonal projection of $\boldsymbol x$ onto the harmonic space is $(\pi^{i} \boldsymbol x^{i})_{i=1}^{k}$, where $\pi^{i}:\bigoplus_{v \in V}F^{i}(V) \to \bigoplus_{v \in V}F^{i}(V)$ denotes the orthogonal projection onto $\operatorname{ker }\boldsymbol L_{F^{i}}$. By Theorem~\ref{thm:2-way-linear-separability}, $\pi^{i}\boldsymbol x^{i} \in \operatorname{span}(\boldsymbol 1_{A_{i}}- \boldsymbol 1_{A_{i}^{c}}) \in \mathbb{R}^{|V|}$, say, $\pi^i \boldsymbol{x}^i=c^i (\boldsymbol 1_{A_{i}}- \boldsymbol 1_{A_{i}^{c}})$. Put $\boldsymbol{Y}:=\operatorname{diag}(\operatorname{sign}c_1, \dots, \operatorname{sign}c_k)$. Then for any $v \in A_i$, $(\boldsymbol{Y}\pi \boldsymbol{x})_i=|c_i|>0$ while  $(\boldsymbol{Y}\pi \boldsymbol{x})_j=-|c_j|<0$ for $j \neq i$. Linear separability follows.

\end{proof}


The stalk dimension required for Theorem~\ref{thm:k-way-linear-separability} grows linearly with the number of classes. One wonders: is it possible to achieve $k$-way linear separation using $d<k$ 1-dimensional twisted sheaves? In general, it is too much to ask for such a property to hold given any $d$.

\begin{theorem}
\label{thm:1d-sheaf-cant-separate-3-classes}
    Let $G$ be a connected graph with $k$ classes. If $k > 2$, then no $d=1$-dimensional lying sheaf $F$ can achieve linear separation in the diffusion limit, regardless of initial condition $\boldsymbol{x}(0)$. 
\end{theorem}

\begin{proof}
    If $\Gamma(G;F)$ is zero then the result is immediate. Suppose $\Gamma(G;F)$ is nonzero. Then $\Gamma(G;F)$ is one-dimensional, for reasons mimicking those in the proof of Theorem~\ref{thm:2-way-linear-separability}. Suppose it is generated by unit vector $\boldsymbol{h}$, so that $\boldsymbol{x}(\infty)=\langle  \boldsymbol{x}(0), \boldsymbol{h}\rangle \boldsymbol{h}$. Pick any three nodes $u,v,w$ in $G$, WLOG $h_v \leq h_u \leq h_v$. There then exists a convex combination $h_u=\alpha h_v + (1-\alpha)h_w$ for some $\alpha \in [0,1]$. We may then compute $$x_{u}(\infty)=\langle \boldsymbol  x(0), \boldsymbol  h \rangle h_{u}= \alpha \langle \boldsymbol  x(0), \boldsymbol  h \rangle h_{v} +(1-\alpha) \langle \boldsymbol  x^{k}(0), \boldsymbol  h \rangle h_{w}= \alpha x_{v}(\infty) +(1-\alpha) x_{w}(\infty).  $$ This implies that $x_u(\infty)$ lies in the convex hull of points belonging to the other classes. The result now follows from the general fact that two Euclidean sets with intersecting convex hulls are not linearly separable.
\end{proof}

That said, there do exist special cases where utilizing $\operatorname{O}(d)$-bundles can achieve $k$-way linear separation power for $d<k$. 

\begin{theorem}
    If $G$ is a connected graph with $k 
\leq 4$ resp. $k \leq 8$ classes, then there exists an $\operatorname{O}(2)$- resp. $\operatorname{O}(4)$-bundle capable of linearly separating the classes of $G$ in the diffusion limit. 
\end{theorem}

We won't need this result, and refer the reader to~\cite{bodnar2022neural} for its proof.




\subsection{Why Don't Order-Zero Arguments Extend? }

Our proof of Theorem~\ref{thm:k-way-linear-separability} depended intimately on the relationship between grade-zero sheaf diffusion and global sections. A similarly strong interpretation is unavailable in positive grading $k \geq 1$; the next best tool available is sheaf cohomology. Indeed, we next show that, in the higher-order setting, a direct extension of Theorem~\ref{thm:k-way-linear-separability} is too much to ask. The problem turns out to be that nondegenerate relationships between features are not enough to counteract an `uninteresting' topology. Indeed, there exist posets for which no discrete vector bundle can achieve linear separation power, as the following result shows.  

\begin{theorem}
    The class of discrete vector bundles on posets has rank-$k$ linear separation power if and only if $k=0$. 
\end{theorem}

\begin{proof}
Separation power in the case $k=0$ was established in Theorem~\ref{thm:k-way-linear-separability}; here we prove that it fails when $k>0$. Suppose $X$ is a \textit{contractible} cell complex, and let $F$ be a discrete vector bundle on $X$. Since $X$ is contractible, it is path connected, and this means that every stalk is (non-canonically) isomorphic to some real vector space $V_{0}$. Since $X$ is moreover simply connected, Corollary 20.12 of~\cite{dahlhausenalgebra} implies that $F$ is a constant sheaf, $F\cong \underline{V_0}$. It then follows from the discussion in Example 3H.1 of~\cite{hatcher2002algtop} that the sheaf cohomology $H^{i}(X; F)$ is just the cohomology with coefficients in the real vector space $V_{0}$, that is,  $H^{i}(X; F)=H^{i}(X; V_{0})$. By (for example) universal coefficients, the latter space equals $H^{i}(X; \mathbb{R}) \otimes_{\mathbb{R}}V_{0}$. Since $X$ is contractible, $H^{i}(X; \mathbb{R})=0$ for all $i>0$, and so $H^{i}(X; F)=0$ for all $i>0$. Theorem~\ref{thm:discrete-Hodge-Theorem} now gives $\operatorname{ker }\Delta^{k}_{F}=0$, and so, by Theorem~\ref{thm:ODE-fact}, linear sheaf diffusion with $F$ exponentially sends any $k$-cochain to zero. Linear separability is therefore impossible in the diffusion limit. 
\end{proof}

\subsection{Sheaf Diffusion is Expressive for Order $k>0$}

We now turn to the expressive power of higher-order sheaf diffusion. Upon combining with Theorem~\ref{thm:k-way-linear-separability}, we will have proven the following: 
\begin{theorem}
\label{thm:all-orders-separability}
    Linear sheaf diffusion can asymptotically solve any rank-$k$ classification task on a poset $S$ for almost all initial conditions. 
\end{theorem}

Of course, our positive-grading arguments require that $k \neq 0$, meaning that the following Theorem~\ref{thm:positive-grading-higher-order-separability} has no overlap with Theorem~\ref{thm:k-way-linear-separability}: the two theorems examine two complementary cases which together prove Theorem~\ref{thm:all-orders-separability}. 

\begin{theorem}
\label{thm:positive-grading-higher-order-separability}
    When $k \geq 1$, linear sheaf diffusion can asymptotically solve any rank-$k$ classification task on a poset $S$ for almost all initial conditions. 
\end{theorem}

Our proof of Theorem~\ref{thm:positive-grading-higher-order-separability} utilizes the following lemma which may be of independent interest.

\begin{lemma}
\label{lemma:positive-sep-lemma}
    Let $X$ be a simplicial complex. Assume $k \geq 1$, and let $V:=\bigoplus_{\tau \in X^{k}} V(\tau)$ be a direct sum of vector spaces indexed by the $k$-skeleton of $X$. Then for any $W \subset V$, there exists a sheaf $F$ on $X$ such that $C^{k}(X; F)=V$ and $\operatorname{im }d^{k-1}=W$. Furthermore, for this sheaf we may assume $d^{k}=0$. 
\end{lemma}

\begin{proof}
    Pick a decomposition $W=\bigoplus_{\tau \in T}W_{\tau}$ such that $W_{\tau} \subset V(\tau)$. For each $\tau \in T$, choose one $(k-1)$-face $f(\tau)\leq \tau$, and define $S:=\{ f(\tau): \tau \in T \} \subset X^{k-1}$. We define a sheaf $F$ on $X$ as follows. First, take $F(\tau)=V(\tau)$ for $\tau \in X^{k}$. Then, for $\sigma \in X^{k-1}$, take $F(\sigma)=\bigoplus_{\tau \in T: f(\tau)=\sigma}W_{\tau}$ if $\sigma \in S$ and $F(\sigma)=0$ otherwise. We set all other stalks to zero. For incidences $\sigma \leq \tau$ with $\sigma=f(\tau)$, we take $$F(\sigma \leq \tau)=[\tau:\sigma]\big(\iota_{\tau} |_{W_{\tau}}: W_{\tau}\hookrightarrow V(\tau)\big)$$
where $\iota_{\tau}$ is the inclusion $V(\tau) \hookrightarrow V$. For all other incidences, we set the corresponding restriction map to zero; note that this leads to functoriality being satisfied by construction. The situation is depicted in Figure~\ref{fig:gradient-space-lemma}. With $F$ so constructed, and denoting by $d^{k-1}_{\tau}$ the $\tau$-component of the codifferential $d^{k-1}:C^{k-1}(X;F) \to C^{k}(X; F)$, we have, for $\boldsymbol x \in C^{k-1}(X;F)$,
$$(d^{k-1}\boldsymbol x)_{\tau}=\sum_{\sigma \leq \tau, \text{dim } \sigma=k-1}[\tau:\sigma]F(\sigma \leq \tau)x_{\sigma}=\begin{cases}
\iota_{\tau}\big( x_{f(\tau)} \big) & \tau \in T \\
0 & \tau \not \in T
\end{cases}$$ where we have used that only $\sigma=f(\tau)$ contributes to the sum and $[\tau:\sigma]^{2}=1$. Hence $\operatorname{im }d^{k-1}_{\tau}=\iota_{\tau}(W_{\tau})=W_{\tau}$, from which it follows  that $$\operatorname{im }d^{k-1}=\bigoplus _{\tau}\operatorname{im }d^{k-1}_{\tau}=\bigoplus_{\tau \in T}W_{\tau}=W$$
as required. 
\end{proof}

\begin{figure}
    \centering
    \includegraphics[width=\linewidth]{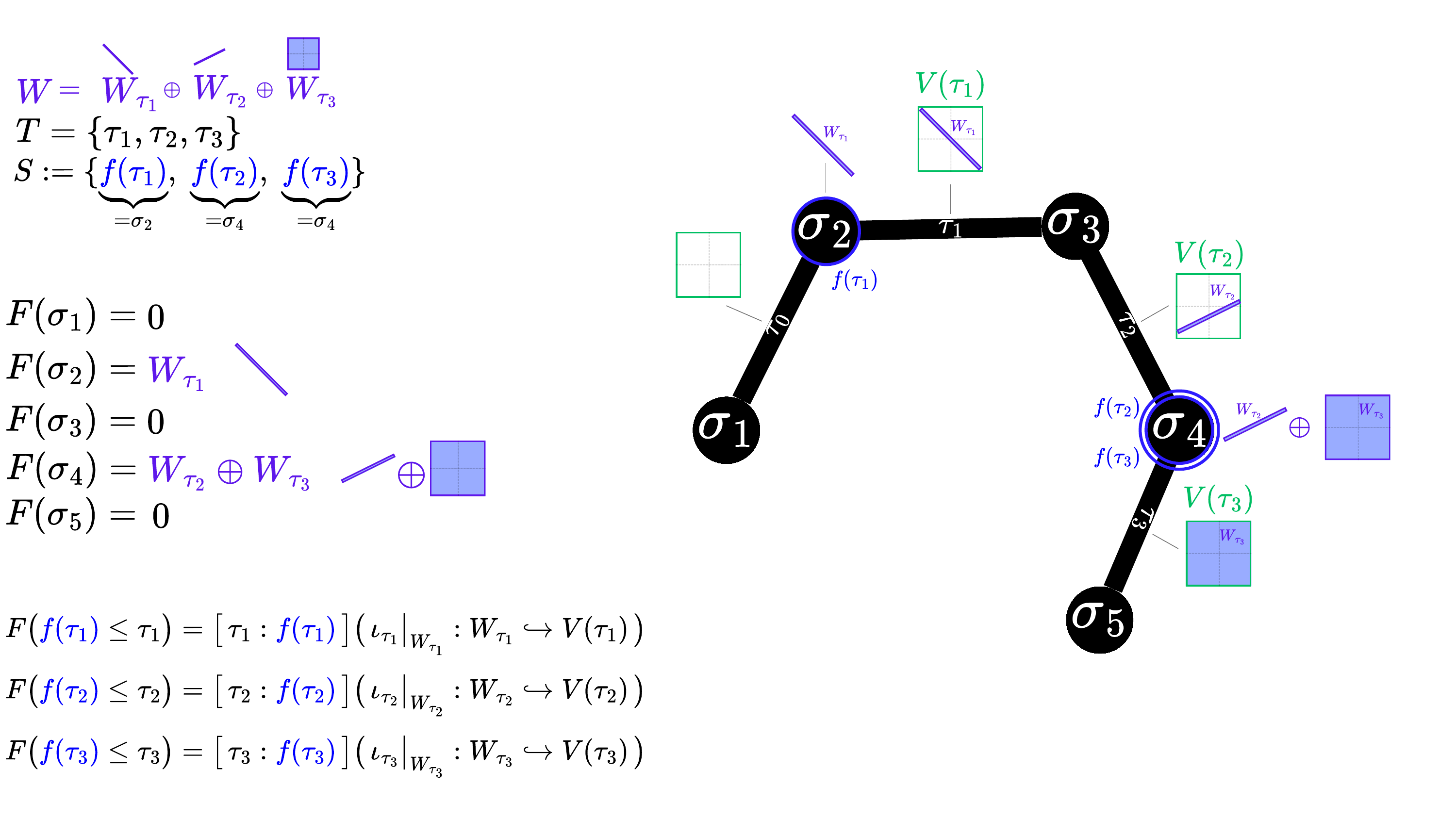}
    \caption{Visualizing the situation of the proof of Lemma~\ref{lemma:positive-sep-lemma}}
    \label{fig:gradient-space-lemma}
\end{figure}

\begin{proof}[Proof of Theorem~\ref{thm:positive-grading-higher-order-separability}.]
    First suppose that the $k$-simplices of $X$ are partitioned into two classes $A$ and $B$. Put $V=\bigoplus_{\tau}\mathbb{R}(\tau)$ an orthogonal direct sum, and consider the element $\boldsymbol s:=\boldsymbol 1_{A}-\boldsymbol 1_{B}$ of $V$, where $\boldsymbol 1_{A}$, $\boldsymbol 1_{B}$ respectively represent the indicator vectors for the classes $A$ and $B$. Let $W=\operatorname{span}^{\perp}\boldsymbol s$. Using Lemma~\ref{lemma:positive-sep-lemma}, obtain a sheaf $F$ on $X$ such that $C^{k}(X; F)=V$, $\operatorname{im }d^{k-1}=W$, and $d^{k}=0$. The $k$th Laplacian of $F$ then has kernel \begin{align}
\operatorname{ker }\Delta^{k}_{F}&=\overbrace{ \operatorname{ker } (d^{k-1})^{*} }^{ = \operatorname{im }^{\perp}d^{k-1} } \cap \overbrace{ \operatorname{ker } d^{k} }^{= C^{k}(X; F) } \\
&= \operatorname{im }^{\perp} d^{k-1} \\
&= W^{\perp}\\
&= \operatorname{span} \boldsymbol  s.
\end{align}
It follows from Theorem~\ref{thm:ODE-fact} that, in the diffusion limit, any initial $k$-cochain not in the orthogonal complement of $\operatorname{span}\boldsymbol s$ (a subspace of measure zero) exponentially converges to a nonzero scalar multiple of $\boldsymbol s=\boldsymbol 1_{A}-\boldsymbol 1_{B}$ and thus witnesses linear separability of the classes $A$ and $B$. 

To prove general multiway linear separability, we would like to establish a `diffusion direct sum' argument analogous to Theorem~\ref{thm:k-way-linear-separability}. This may be accomplished as follows. Suppose that the $k$-simplices of $X$ are partitioned into classes $A^{1},\dots,A^{\ell}$. To disambiguate, we will use superscripts for class indexing (e.g. $A^{i}$) and subscripts for simplex indexing (e.g. $\boldsymbol x_{\tau}$). For $i \in [\ell]$, let $V^{i}=\bigoplus_{\tau}\mathbb{R}(\tau)$ and consider the element $\boldsymbol s^{i}=\boldsymbol 1_{A^{i}}-\boldsymbol 1_{B^{i}}$ of $V^{i}$, where $B^{i}:=A^{c}$ is the complement of $A$. Now apply the two-way classification argument above to obtain a sheaf $F^{i}$ whose diffusion converges exponentially to the projection of the initial condition onto $\boldsymbol s_{i}$, yielding one-vs-many separation power for class $i$. We claim that the direct sum $F:=\bigoplus _{i=1}^{\ell} F^{i}$ can linearly separate the classes $A^{1},\dots,A^{\ell}$ in the diffusion limit for almost any choice of initial condition $\boldsymbol x=(\boldsymbol x^{i})_{i=1}^{\ell} \in C^{k}(X; F)$. Indeed, with $\pi^{i}:\bigoplus_{\tau: \text{dim } \tau=k}F^{i}(\tau) \to\pi^{i}:\bigoplus_{\tau: \text{dim } \tau=k}F^{i}(\tau)$ denoting the orthogonal projection onto $\operatorname{ker }\Delta_{F^{i}}^{k}$, we know $\pi^{i}\boldsymbol x^{i} \in \operatorname{span}(\boldsymbol 1_{A_{i}}- \boldsymbol 1_{B_{i}})$, say, $\pi^{i}\boldsymbol x^{i}=c^{i}\boldsymbol  s ^{i}$ for some $c \in \mathbb{R}$. Then $\boldsymbol x=(\boldsymbol x_{\tau})=(x_{\tau}^{i})$ converges to $$\big(\boldsymbol  x_{\tau}(\infty) \big)=\big( c^{1} s ^{1}_{\tau}, \dots, c^{\ell} s ^{\ell}_{\tau} \big)= \begin{cases}
 \big( \pm c^{1}, \dots, +c^{i}, \dots, \pm c^{\ell} \big) & \tau \in A^{i} \\
\big( \pm c^{1}, \dots, - c^{i}, \dots, \pm c^{\ell} \big) & \tau \not \in A^{i}.
\end{cases}$$
Taking $\boldsymbol Y:=\operatorname{diag}\big( \frac{\operatorname{sign}(c^{1})}{|c^{1}|}, \dots, \frac{\operatorname{sign}(c^{\ell})}{|c^{\ell}|} \big)$ then recovers $\boldsymbol Y \boldsymbol x_{\tau}(\infty)=(s_{\tau}^{1}, \dots, s_{\tau}^{\ell})$, provided that each initial $\boldsymbol x^{i} \in C^{k}(X; F^{i})$ is not in the orthogonal complement of $\operatorname{span}(\boldsymbol 1_{A_{i}}- \boldsymbol 1_{B_{i}})$. Linear separability for almost all initial conditions follows immediately. 
\end{proof}

\section{Learning with Sheaves}
\label{sec:learning-with-sheaves}


The previous sections have established that linear sheaf diffusion on a poset can achieve remarkable expressivity gains. Arising immediately from these demonstrations are two questions. The first question returns to Motivation~\ref{sec:intuition-oversmoothing-heterophily}, where it was shown that the Graph Convolutional Architecture (GCN) of Kipf and Welling~\cite{kipf2016semi} may be viewed as an augmentation of Hodge diffusion on a graph. More generally, many approaches to higher-order message passing arise as augmented Hodge diffusion on more general posets. It is thus sensible to wonder: what would a `neural' version of sheaf diffusion look like, and how would it behave? The second question stems from the observation that, in practice, the most expressive sheaf for the observed data is not known \textit{a priori}. Can it be learned? The following definition pertains to both questions.

\begin{definition}[\cite{ayzenberg2025sheaf}]
\label{def:sheaf-type-nn}
    Provisionally speaking, a \textbf{sheaf-type neural network} is an architecture which cascades linear sheaf diffusion layers alongside nonlinear activations, often in tandem with learned weights, auxiliary channels, etc. Commonly the sheaf with respect to which heat diffusion is performed is itself learned. 
\end{definition}

The past five years have witnessed a proliferation of sheaf-type neural networks for (hyper)graphs in the literature. Most intimately related to our dicussions so far are the Sheaf Neural Networks of Hansen and Gebhart~\cite{hansen2020sheaf} and {Neural Sheaf Diffusion (NSD)} architecture of Bodnar et al.~\cite{bodnar_neural_2023}.  
Alongside the development of sheaf-type neural networks has been that of higher-order message passing with the Hodge Laplacian e.g. on a simplicial or cell complex~\cite{ebli_simplicial_2020, roddenberry_principled_2021, hajij2022topological, bunch2020simplicial, yang2022efficient, yang_convolutional_2023, hajij2020cell}  . Of yet, there has not been an explicit integration of the sheaf formalism into these higher-order architectures. Thus, the goal of this section is twofold. To begin, we circle back to Motivation~\ref{sec:intuition-oversmoothing-heterophily} to explore the expressivity of graph NSD with respect to phenomena such as oversmoothing. Following this, we novelly define and investigate a suite of architectures for higher-order message passing with sheaves. 



\subsection{Sheaf Convolutional Networks on Graphs}

We have seen in Section~\ref{sec:linear-separation-power-sheaf-diffusion} that while linear diffusion-based approaches to node classification on a graph using the vanilla graph Laplacian are doomed to oversmoothing, any node classification task on a graph (in fact, on a hypergraph) can be asymptotically solved in principle by performing linear diffusion more generally with a suitable sheaf. Recall now the following question, originally posed in Section~\ref{sec:intuition-oversmoothing-heterophily}: 

\begin{quote}
\centering
\emph{Does oversmoothing remain inevitable even when nonlinearity, weights, and normalization are (re)introduced?}
\end{quote}

Section~\ref{sec:linear-separation-power-sheaf-diffusion} proved that normalization helps \textit{a bit} (especially in homophilic settings), but is generally not enough to mitigate oversmoothing. We would like to explore the extent to which adding nonlinearity and weights to (normalized) graph Laplacian-based diffusion — that is, employing the GCN architecture of Kipf and Welling (\cite{kipf2016semi}, Example~\ref{ex:kipf-welling-GCN}) — help. Moreover, Section~\ref{sec:linear-separation-power-sheaf-diffusion} showed that linear sheaf diffusion can be very expressive beyond just mitigating oversmoothing.   The following sheaf-type neural network architecture (Definition~\ref{def:sheaf-type-nn}), based on \cite{hansen2020sheaf, bodnar_neural_2023}, pertains to both points. 



\begin{definition}
\label{def:NSD}
    For $d$ a hyperparameter and $G$ an $n$-node graph, NSD learns the restriction maps for a sheaf $F$ on $G$, where each stalk of $F$ is a copy of $\mathbb{R}^{d}$. A layer of neural sheaf diffusion with $f_{\text{in}}$ input channels and $f_{\text{out}}$ output channels is given by applying a nonlinearity $\sigma$ to the tensor product of linear maps
\begin{align}
\label{eqn:order-0-nsd}
(\operatorname{sd}_{F} \circ \boldsymbol W_{1}^{\oplus n}) \otimes \boldsymbol W_{2}: C^{0}(F; G) \otimes \mathbb{R}^{f_{\operatorname{in}}} \to C^{0}(F; G) \otimes \mathbb{R}^{f_{\operatorname{out}}},
\end{align}
where $\boldsymbol W_{1}:\mathbb{R}^{d} \to \mathbb{R}^{d}$ consists of learnable weights applied independently at each node-stalk and $\boldsymbol W_{2}:\mathbb{R}^{f_{\operatorname{in}}} \to \mathbb{R}^{f_{\operatorname{out}}}$ consists of learnable weights for channel mixing.

\end{definition}
The details of `learning a sheaf' will be discussed in Section~\ref{sec:higher-order-sheaf-type-NNs}. Note that the GCN architecture of Kipf and Welling~\cite{kipf2016semi} arises as the special case of Definition~\ref{def:NSD} wherein $d=1$, $\boldsymbol{W}_1$ is the identity, and $F$ is the constant sheaf $\underline{\mathbb{R}}$. In particular, a corollary of the following result answers the question posed in Motivation~\ref{sec:intuition-oversmoothing-heterophily}: to what extent is oversmoothing inevitable in GCNs?




\begin{theorem}[\cite{bodnar_neural_2023}]
\label{thm:gcn-answer-to-question}
     Denote by $\mathscr{F}^{1}_{+}$ the class of one-dimensional sheaves $F$ on connected graphs $G=(V,E)$ satisfying $F_{v \leq e}F_{u \leq e}>0$ for all $(u,v)\in E$.  Consider one sheaf-convolutional layer 
\begin{align}
\boldsymbol  Y=\sigma \big( (\boldsymbol  I- \boldsymbol  \Delta_{F}  ) (\boldsymbol  I_{n} \otimes \boldsymbol  W_{1}) \boldsymbol  X \boldsymbol  W_{2} \big)
\end{align}
with (Leaky)ReLU $\sigma$ and $\boldsymbol \Delta_{F}$ denoting the \textit{normalized} vanilla sheaf Laplacian, $\boldsymbol \Delta_{F}=\dots$. Then any $F \in \mathscr{F}^{1}_{+}$  satisfies 
\begin{align}
E_{F}(\boldsymbol  Y) \leq \lambda_{*} \|\boldsymbol  W_{1}\|_{2}^{2} \|\boldsymbol  W_{2}^{\top}\|_{2}^{2} E_{F}(\boldsymbol  X),
\end{align}
where $\lambda_{*}=\max_{i  > 0}\{(\lambda_{i}(\Delta_{F})-1)^{2}\}$. Note that $0<\lambda_{*} \leq 1$. 

\end{theorem}

Theorem~\ref{thm:gcn-answer-to-question} implies, in particular, that the GCN architecture sends initial node representations to $\operatorname{ker }\boldsymbol \Delta$ exponentially quickly provided the weights are small enough. When $G$ is connected, this is tantamount to asserting that the GCN architecture exhibits oversmoothing exponentially quickly when the weights are small enough. Moreover, once oversmoothing ($E_{F}(\boldsymbol X)=0$) occurs, it cannot be undone, no matter what the weights are. 

\begin{proof}
    We follow the proof in \cite{bodnar_neural_2023}. 

First note that it suffices to argue the result ignoring $\sigma(\cdot)$, because in general $E_{F}\big( \sigma(Z)\big) \leq E_{F}(Z)$ for $\sigma$ a (Leaky)ReLU.  Indeed, 
\begin{align*}
E_{F}\big( \sigma(\boldsymbol  x) \big) &= \frac{1}{2} \sum_{v,w \leq e} \|F_{v \leq e} d_{v}^{-1/2} \sigma(x_{v})-F_{w \leq e} d_{w}^{-1/2} \sigma(x_{w})\|_{2}^{2} \\
&= \frac{1}{2} \sum_{v, w \leq e} \| |F_{v \leq e}| d_{v}^{-1/2} \sigma(x_{v}) - |F_{w \leq e}|d_{w}^{-1/2}\sigma(x_{w})  \|_{2}^{2} \tag{ $F_{v \leq e} F_{w \leq e}>0$. } \\
&= \frac{1}{2} \sum_{v, w \leq e} \| \sigma\big( \frac{|F_{v \leq e}| x_{v} }{\sqrt{ d_{v} }}\big) - \sigma\big( \frac{|F_{w \leq e}|x_{w}}{\sqrt{ d_{w} }} \big) \|_{2}^{2} \tag{$c \sigma(x)=\sigma(cx), c>0$}\\
&\leq \frac{1}{2} \sum_{v,w \leq e}  \| \frac{|F_{v \leq e}|x_{v}}{\sqrt{ d_{v} }}- \frac{|F_{v \leq e}|x_{w}}{\sqrt{ d_{w} }} \|_{2}^{2} \tag{ ReLU 1-Lipschitz cont. }\\
&=\frac{1}{2} \sum_{v, w \leq e} \|F_{v \leq e} d_{v}^{-1/2}x_{v}-F_{w \leq e}D_{w}^{-1/2}x_{w}\|_{2}^{2} \tag{$F_{v \leq e}F_{w \leq e}>0$}\\
&= E_{F}(\boldsymbol  x).
\end{align*}
Put $\boldsymbol P:= \boldsymbol I-\boldsymbol \Delta_{F}$. Noting that in this case $\boldsymbol W_{1}$ is just a scalar $w_{1}$, we are done if we can show 
\begin{align}
E_{F}(  \underbrace{ w_{1}\boldsymbol  P \boldsymbol  X \boldsymbol  W_{2}  }_{ \boldsymbol  Y }) \leq \lambda_{*}|w_{1}|^{2}\|\boldsymbol  W_{2}\|_{2}^{2}E_{F}(\boldsymbol  X).
\end{align}
Indeed,

\begin{align*}
E_{F}(w_{1}\boldsymbol  P \boldsymbol  X \boldsymbol  W_{2}) &= |w_{1}|^{2}E_{F}(\boldsymbol  P \boldsymbol  X \boldsymbol  W_{2})\\
& \leq |w_{1}|^{2} \lambda_{*} E_{F} (\boldsymbol  X \boldsymbol  W_{2}) \tag{$*$} \\
& = |w_{1}|^{2} \lambda_{*}\operatorname{Tr}(  \boldsymbol  W^{\top}\boldsymbol  X^{\top} \boldsymbol  \Delta_{F} \boldsymbol  X \boldsymbol  W )\\
&= |w_{1}|^{2} \lambda_{*}\operatorname{Tr}(\boldsymbol  X^{\top} \boldsymbol  \Delta_{F} \boldsymbol  X \boldsymbol  W \boldsymbol  W^{\top})\\
& \leq |w_{1}|^{2} \lambda_{*} \operatorname{Tr}(\boldsymbol  X^{\top} \boldsymbol  \Delta_{F} \boldsymbol  X) \|\boldsymbol  W \boldsymbol  W^{\top}\|_{2}\\
&=|w_{1}| ^{2} \lambda_{*}\operatorname{Tr}(\boldsymbol  X^{\top} \boldsymbol  \Delta_{F} \boldsymbol  X) \|\boldsymbol  W^{\top}\| _{2}^{2}\\
&= |w_{1}|^{2} \lambda_{*} \|\boldsymbol  W^{\top}\|_{2}^{2} E_{F}(\boldsymbol  X).
\end{align*}
To see why $(*)$ holds, assume $\boldsymbol x:=\boldsymbol X\boldsymbol W_{2}$ has eigencoordinates $(c_{i})$, whence $\boldsymbol x^{\top} \boldsymbol \Delta_{F} \boldsymbol x=\sum_{i}c_{i}^{2}\lambda_{i}(\boldsymbol \Delta_{F})$  and so 
\begin{align}
E_{F}(\boldsymbol  P \boldsymbol  x)=\boldsymbol  x^{\top} \boldsymbol  P^{\top} \boldsymbol  \Delta_{F} \boldsymbol  P \boldsymbol  x= \sum_i c_{i}^{2} \lambda_{i}(1-\lambda_{i})^{2} \leq \lambda_{*} \sum_{i} c_{i}^{2} \lambda_{i}=\lambda_{*}E_{F}(\boldsymbol  x),
\end{align}
where the inequality is a consequence of the fact that the eigenvalues of $\boldsymbol \Delta_{F}$ are normalized to lie in $[0,2]$. 

\end{proof}


\subsection{Higher-Order Sheaf-Type Neural Networks}
\label{sec:higher-order-sheaf-type-NNs}
We now define a collection of architectures for TDL based on sheaf diffusion. A first extension of the SCN architecture to higher-order settings may be obtained by augmenting general order-$k$ Hodge heat diffusion in precisely the same manner that Equation~\ref{eqn:order-0-nsd} augments order-$0$ Hodge heat diffusion. A layer of this model is given by applying an odd nonlinearity $\phi$ to the tensor product of linear maps \begin{equation}
\label{eqn:naive-hosd-update}
   (\text{sd}_{F}^{(k)} \circ   W_{\text{stalk}}^{\oplus n}) \otimes  W_{\text{channel}}:C^{k}(S;F) \otimes f_{\text{in}} \to C^{k}(S; F) \otimes f_{\text{out}} 
\end{equation}
where $\text{sd}_{F}^{(k)}=\operatorname{id}- 2 \eta \Delta^{k}_{F}=\operatorname{id} - 2 \eta (\Delta^{k}_{F, \text{up}}+ \Delta^{k}_{F, \text{down}})$ (Definition~\ref{def:sheaf-diff-layer}). Here, $W_{\text{stalk}}:\mathbb{R}^{d} \to \mathbb{R}^{d}$ represents a linear transformation of the typical stalk parameterized by learned weights, while $W_{\text{channel}}:\mathbb{R}^{f_{\text{in}}} \to \mathbb{R}^{f_{\text{out}}}$ consists of learned weights for channel mixing. The nonlinearity $\phi$ is chosen to be odd to ensure equivariance with respect to the orientation on cochains used to define the codifferential $d$; the proof of this fact is a straightforward generalization to posets of the argument in~\cite{roddenberry_principled_2021} for simplicial complexes and so we omit it. While simple, the update~\ref{eqn:naive-hosd-update} forces uniform weights across all factors of the Hodge decomposition $C^{k}(S;F)=\operatorname{ker }\Delta^{k} \oplus \operatorname{im }d^{k-1} \oplus \operatorname{im }d^{^{*}k}$. We decouple by assigning separate weights for the up- (curl) and down- (gradient) Laplacians, as well as a 'center' term facilitating the passage of information along the harmonic component. With $\text{sd}_{F, \text{up}}^{(k)}= \frac{1}{2} \operatorname{id}-\Delta^{k}_{F,\text{up}}$ and $\text{sd}_{F, \text{down}}^{(k)}= \frac{1}{2} \operatorname{id}-\Delta^{k}_{F, \text{down}}$, so that $\text{sd}_{F}^{(k)}=\text{sd}_{F, \text{up}}^{(k)}+\text{sd}_{F, \text{down}}^{(k)}$, the resulting layer is given by a nonlinearity $\phi$ applied to

\begin{equation}
\begin{aligned}
&(\text{sd}_{F,\text{up}}^{(k)} \circ   W_{\text{stalk,up}}^{\oplus n}) \otimes  W_{\text{channel,up}} \\
+ &(\text{sd}_{F,\text{down}}^{(k)} \circ   W_{\text{stalk,down}}^{\oplus n}) \otimes  W_{\text{channel,down}} \\
+ &W_{\text{stalk,center}}^{\oplus n} \otimes W_{\text{channel,center}}:
C^{k}(S;F) \otimes f_{\text{in}} \to C^{k}(S;F) \otimes f_{\text{out}}
\end{aligned}
\label{eq:honsd-operator-form}
\end{equation}

where $W_{\text{stalk}, \bullet}^{(k)} : \mathbb{R}^{d} \to \mathbb{R}^{d}$ and $W_{\text{channel}, \bullet}: \mathbb{R}^{f_{\text{in}}} \to \mathbb{R}^{f_{\text{out}}}$. In practice, we identify $C^{k}(S;F) \cong \mathbb{R}^{d |X^{k}|}$ by stacking stalks, giving in matrix form the update $(\boldsymbol X \in \mathbb{R}^{d |X^{k} | \times f_{\text{in}}})$ 

\begin{equation}
\begin{aligned}
\boldsymbol X_{\text{new}}=\phi \big(
&\text{sd}^{(k)}_{F,\text{up}}\;\big(I_{|X^{k}|}\otimes \boldsymbol{W}_{\text{stalk,up}}\big)\;\boldsymbol{X}\;\boldsymbol{W}_{\text{channel,up}} \\
+&\text{sd}^{(k)}_{F,\text{down}}\;\big(I_{|X^{k}|}\otimes \boldsymbol{W}_{\text{stalk,down}}\big)\;\boldsymbol{X}\;\boldsymbol{W}_{\text{channel,down}} \\
+&\big(I_{|X^{k}|}\otimes \boldsymbol{W}_{\text{stalk,center}}\big)\;\boldsymbol{X}\;\boldsymbol{W}_{\text{channel,center}}
\big)
\end{aligned}
\label{eq:honsd-matrix-form}
\end{equation}

Let us examine some special cases. When $S$ is a simplicial complex and $F$ is the constant sheaf, and $\eta=\frac{1}{2}$, Equation~\ref{eq:honsd-matrix-form} recovers the SCoNe architecture of~\cite{roddenberry_principled_2021}. If the channel weight matrices are furthermore equal it recovers a one-hop version of the simplicial neural network architecture of~\cite{ebli_simplicial_2020}. When $k=0$ and $S$ represents a hypergraph (with Laplacian obtained per Example~\ref{ex:sheaf-diffusion-on-hypergraphs} and the accompanying remark), Equation~\ref{eq:honsd-matrix-form} recovers the sheaf hypergraph network architecture of~\cite{duta2023sheaf}, and in turn the sheaf convolutional graph neural network of~\cite{hansen2020sheaf, bodnar2022neural}, up to normalization. 



\paragraph{Learning Sheaves}
In many practical settings, the 'right' sheaf for the data is not known and the restriction maps are therefore learned. Let $d$ be a stalk width hyperparameter and $f$ the feature dimension. In the most naive graph setting, one learns an (asymmetric) function $\Phi:\mathbb{R}^{f} \times \mathbb{R}^{f} \to \mathbb{R}^{d \times d}$ which takes as input pairs $(\boldsymbol x_{u}, \boldsymbol x_{v})$ of node features and outputs a $d \times d$ matrix $\Phi(\boldsymbol x_{u}, \boldsymbol x_{v})$ representing the restriction map $F_{u \leq (u,v)}$. Usually one takes $\Phi$ to be an $\operatorname{MLP}$ applied to concatenations $\boldsymbol x_{u}\| \boldsymbol x_{v}$. It is not uncommon to constrain the restriction maps to be e.g. diagonal or orthogonal to decrease parameter count and achieve stabler training~\cite{bodnar_neural_2023}. Assuming that $(k-1)$- and $k$-simplex features of dimensions $f_{k-1}, f_{k}$ have been procured, the most straightforward generalization of this procedure to simplicial complexes is to learn a pair of permutation-aware functions $\Phi_{k-1}:\underbrace{ \mathbb{R}^{f_{k-1}} \times \mathbb{R}^{f_{k-1} }   \times \dots \times \mathbb{R}^{f_{k-1}} }_{ k \text{ times}}\to \mathbb{R}^{d \times d}$ and  $\Phi_{k}:\underbrace{ \mathbb{R}^{f_{k}} \times \mathbb{R}^{f_{k}} \times\dots \times \mathbb{R}^{f_{k}} }_{ k+1 \text{ times} } \to \mathbb{R}^{d \times d}$. Since one is usually concerned only with interactions spanning at most two adjacent simplex ranks, it suffices to then take the restriction maps for all other covering incidences to be zero, and then define the remaining restriction maps by composition to obtain an earnest sheaf. For general posets, one loses the regularity condition that every element of rank $k$ covers exactly $k+1$ elements of rank $k-1$. In this case, our generalization follows the strategy for hypergraphs outlined in~\cite{duta2023sheaf}: we learn functions $\Phi(\boldsymbol x_{\sigma}, \boldsymbol x_{\tau})$ of features $\boldsymbol x_{\sigma} \in F(\sigma) \cong \mathbb{R}^{d}$, $\boldsymbol x_{\tau} \in F(\tau) \cong \mathbb{R}^{d}$ for $\sigma \leq \tau$, where the target element's features are obtained via a permutation-invariant aggregation of its covered features if not naturally available. As before, usually one takes $\Phi=\operatorname{MLP}(\cdot \| \cdot)$, e.g. $F(\sigma \leq \tau)=\operatorname{MLP}(\boldsymbol x_{ \sigma} \| \boldsymbol x_{\tau})$.

\paragraph{Experiment: Trajectory Prediction} We demonstrate the effectiveness of the proposed architecture on a synthetic trajectory prediction task inspired by the motivating examples of higher-order sheaves in Section~\ref{sec:higher-order-sheaves-in-context}. Given a rank-$2$ simplicial complex $S$, we define an \textit{trajectory} on $S$ to be a sequence of adjacent nodes $i_0,\dots,i_{m-1}$ in $S$. The goal is to forecast what the subsequent node $i_{m}$ will be. Roddenberry et al.~\cite{roddenberry_principled_2021} approach this task by modeling trajectories as simplicial $1$-chains $[i_0, i_1] + \dots + [i_{m-2}, i_{m-1}] \in C_1(S)$,  learning a neural map $f:C^1(S) \to C^1(S)$, and then softmaxing a node signal induced by $f$ (e.g. via postcomposition with $\operatorname{div}$) over the one-hop neighborhood of $i_{m-1}$. We will compare three neural methods for learning $f$. The first takes $F=\underline{\mathbb{R}}$ to be the constant sheaf in Equation~\ref{eq:honsd-matrix-form}, recovering the neural Hodge diffusion architecture of Roddenberry et al.~\cite{roddenberry_principled_2021}. The second takes $F=F_{\text{handcrafted}}$ to be a sheaf tailored to the data generation process for the synthetic dataset we employ. The third \textit{learns} $F$, per the preceding discussion. We also compare to linear sheaf diffusion, in the form of projections onto the kernels $\ker \Delta_{F_{\text{handcrafted}}}$, $\ker \Delta_{\text{up}}=\ker \Delta_{F_{\text{up}}}$, $\ker \Delta=\ker \Delta_{\underline{\mathbb{R}}}$ $\ker \Delta_{\text{down}}=\ker \Delta_{F_{\text{down}}}$. For the linear kernel projection approaches, we obtain node signals by restricting the boundary map to just edges incident to the current terminal node $i_{m-1}$ (since postcomposing an harmonic signal with $\operatorname{div}$ nullifies it). 


Our evaluation takes place on a synthetic dataset consisting of trajectories on a simplicial complex $S$ given by a regularly triangulated plane with a hole in the middle. In one half of the plane, the trajectories are designed to be `approximately harmonic', such that the ordinary Hodge Laplacian bias is appropriate. We will call this the \textit{harmonic region} of the plane. In the other half, the trajectories are designed to be `approximately curly', such that the ordinary Hodge Laplacian bias is inappropriate but something like the Hodge down-Laplacian \textit{is} appropriate (Motivation~\ref{sec:higher-order-sheaves-in-context}, Figure~\ref{fig:example-trajs}). We will call this the \textit{curl region}. Specifically, trajectories in the harmonic region are constructed by, for each step, choosing among neighbors in the harmonic region (excluding immediate backtracking) the neighbor whose displacement from the hole's center achieves the highest dot product with a $\frac{\pi}{2}$-rotation of the current vector. This encourages trajectories to roughly circle around the hole. Trajectories in the curl region are constructed as follows: at each step, with $p_{\text{curl}}=0.8$ and if there is a previous node, choose a common triangle to the two nodes uniformly at random and jump to its other vertex (if that vertex is also in the curl region). Otherwise, step to a random neighbor in the curl region. This creates trajectories for which circulation is an important signal property. We generate $250$ synthetic trajectories in the harmonic region and $250$ in the curl region, each of length $10$. 

This dataset's heterogeneity presents a challenge for architectures based on Hodge diffusion: for half of the trajectories, $\boldsymbol{W}_{\text{stalk, up}}$ should really be set (or learned) to zero to get the `right' inductive bias. For the other half, setting (learning) $\boldsymbol{W}_{\text{stalk, up}}=0$ would be nonsensical, as triangle information is essential. Following Example~\ref{sec:higher-order-sheaves-in-context}, we handcraft a sheaf $F_{\text{handcrafted}}$ to circumvent this challenge: we take the constant sheaf $\underline{\mathbb{R}}$, but set the triangle stalks in the curl region to zero. This produces a sheaf Laplacian emulating $\Delta_{\text{down}}$ in the curl region and $\Delta$ in the harmonic region (see Example~\ref{sec:higher-order-sheaves-in-context} for further details). We train neural sheaf diffusion (Equation~\ref{eq:honsd-matrix-form}) $4$ layers of $32$ hidden dimensions each for $100$ epochs using the Adam optimizer~\cite{kingma2014adam}. Results are shown in Figure~\ref{fig:traj-prediction-results}.




\begin{figure}[ht]
  \centering
  \begin{subfigure}[t]{0.32\textwidth}
    \centering
    \includegraphics[width=\linewidth]{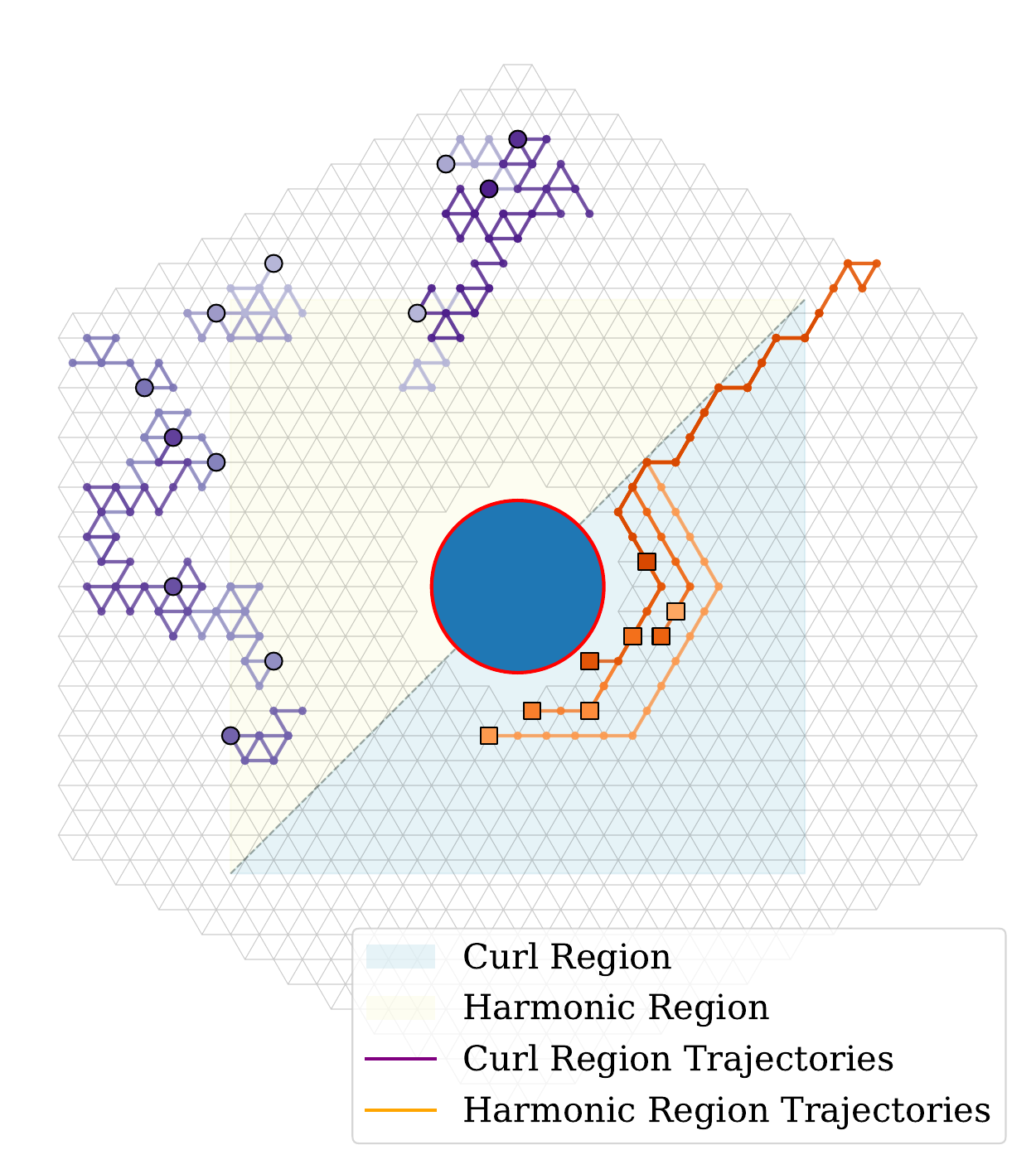}
    \caption{}
    \label{fig:traj-pred-trajectories}
  \end{subfigure}\hfill
  \begin{subfigure}[t]{0.32\textwidth}
    \centering
    \includegraphics[width=\linewidth]{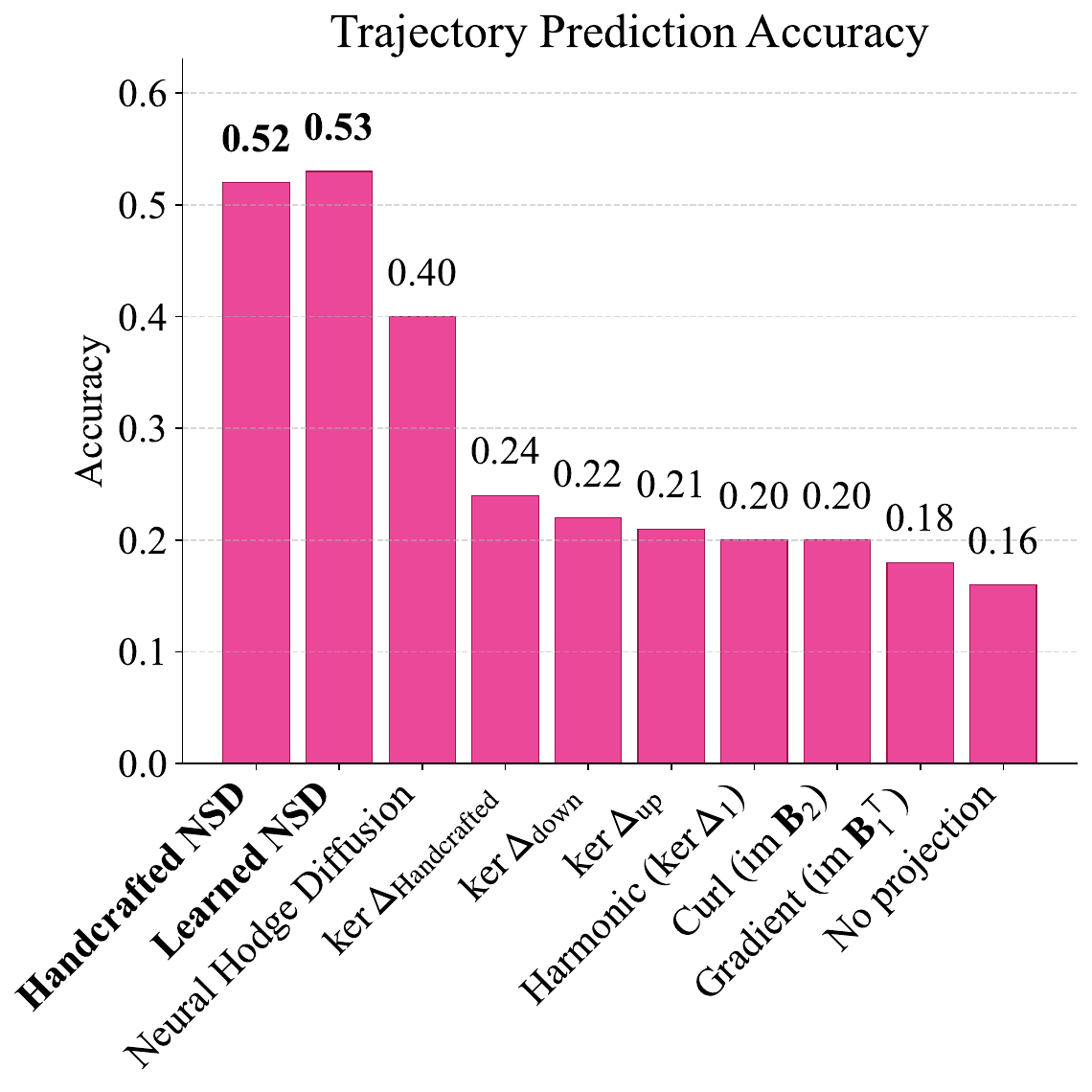}
    \caption{}
    \label{fig:traj-pred-results-subfig}
  \end{subfigure}\hfill
  \begin{subfigure}[t]{0.32\textwidth}
    \centering
    \includegraphics[width=\linewidth]{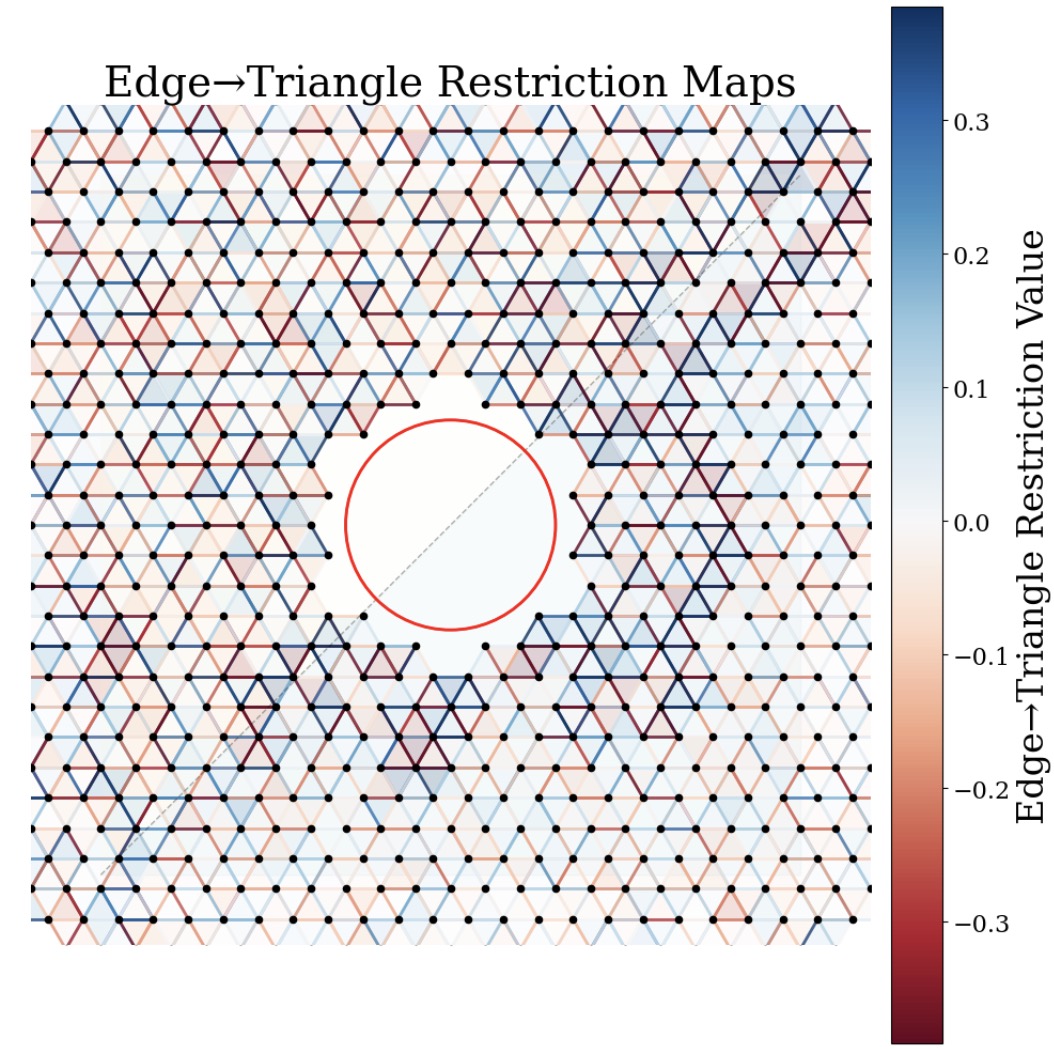}
    \caption{}
    \label{fig:traj-pred-learned-sheaf}
  \end{subfigure}
\caption{(\ref{fig:traj-pred-trajectories}): Heterogeneous trajectories on a triangulation of the punctured plane. (\ref{fig:traj-pred-results-subfig}): Trajectory prediction results. The sheaf-based methods (both handcrafted and learned) outperform the Hodge-based methods and the linear diffusion methods, with the best-performing linear method being diffusion with the handcrafted sheaf Laplacian. (\ref{fig:traj-pred-learned-sheaf}) Visualization of the edge-triangle restriction maps for the (unit stalk width) learned sheaf. Notably, the learned restriction maps of highest magnitude concentrate around the planted harmonic-like trajectories, while the restriction maps around the planted curl-like trajectories have low magnitude. This pattern aligns with the philosophy used to construct the handcrafted sheaf.}
  \label{fig:traj-prediction-results}
\end{figure}



\section{Conclusion}
\label{sec:conclusion}

This essay initiated a theory and practice for the integration of sheaves into Topological Deep Learning (TDL). Following an introduction shepherding sheaf theory from the abstract to the applied, we comprehensively reviewed its extant theoretical and empirical applications in graph representation learning before developing novel extensions to the higher-order setting. These included new techniques for examining and ultimately establishing the expressivity of higher-order sheaf diffusion (alongside proofs that order-zero techniques fail to generalize) as well as the development of novel sheaf-type neural networks for TDL. Along the way, we provided a detailed discussion on the potential and pitfalls of diffusion with the Hodge Laplacian, incidentally shedding light on some misconceptions in recent literature regarding the nature of `higher-order oversmoothing'. Carrying out the exposition at the level of posets allowed for a more ready reconciliation of the `topological' and `combinatorial' sheaf characterizations found throughout pure and applied mathematics while allowing for diverse domains (e.g. cell complexes versus hypergraphs) to be treated with the same formalism.

This work was completed as a portion of the Essay component of Part III of the Mathematical Tripos at the University of Cambridge. The constraints of the Essay medium leave a wealth of future work to be explored. Most notably, comprehensive empirical validation against the state-of-the-art is not part of the Essay rubric. An updated version of this manuscript will include such experiments. Additional future work might consider sheafifying architectures beyond the baselines outlined here. This may include extension to attentional~\cite{goh2022simplicial, battiloro_generalized_2024, giusti2023cell} or Dirac-based~\cite{battiloro_generalized_2024, calmon2023dirac} formulations of higher-order message passing. It may also include going beyond higher-order message passing altogether, e.g. toward continuous formulations~\cite{einizade2025cosmos} or multi-cellular networks~\cite{eitan2024topological}.

Another potential avenue for future work might involve changing the data category $\mathsf{D}=\mathbb{R}\mathsf{Vect}$ altogether. Many categories of recent interest are nonabelian, such as lattices or convex spaces of probability measures, but nevertheless recent work has developed for them a young but fruitful sheaf theory~\cite{ghrist2020cellular, ghrist2025categorical, dacunto_relativity_2025}. In parallel, the field of Algebraic Signal Processing (ASP)~\cite{puschel2008algebraic} poses an analogue to graph signal processing (including convolutional neural networks) for data valued in categories beyond $\mathsf{Vect}$. It seems a fascinating direction of future work to marry the nascent theories of deep learning and sheaves within these categories.

\paragraph{Acknowledgments} The authors would like to acknowledge Cristian Bodnar, Claudio Battiloro, Riccardo Ali, Anton Ayzenberg, and Vincent Grande for their insightful discussions.

\vfill \eject

\bibliographystyle{unsrtnat}
\bibliography{references, zotero}  






\end{document}